\documentclass[11pt]{article} 
\usepackage{url}
\usepackage{smile}
\usepackage{graphicx} 
\usepackage{algorithm}
\usepackage{algorithmic}
\usepackage{todonotes}
\usepackage{epstopdf}
\usepackage{wrapfig}
\usepackage[colorlinks, linkcolor=C2, anchorcolor=C2, citecolor=C2]{hyperref}
\usepackage[margin=1in]{geometry}
\usepackage[normalem]{ulem}
\usepackage[export]{adjustbox}
\usepackage{mathtools, cuted}
\usepackage{natbib}
\usepackage{enumerate}
\usepackage{enumitem}
\usepackage{xcolor}
\definecolor{C1}{RGB}{255, 127, 14}
\definecolor{blue}{rgb}{0, 0, 0}
\definecolor{C2}{rgb}{0, 0, 1}

\usepackage{kpfonts}
\DeclareMathAlphabet{\mathsf}{OT1}{cmss}{m}{n}

\SetMathAlphabet{\mathsf}{bold}{OT1}{cmss}{bx}{n}

\providecommand{\norm}[1]{\|#1\|}

\newtheorem*{theorem*}{Theorem}

\title{\huge \bf Nonparametric Regression on Low-Dimensional Manifolds using Deep ReLU Networks : Function Approximation and Statistical Recovery}

\author{Minshuo Chen, Haoming Jiang, Wenjing Liao, Tuo Zhao\thanks{Alphabetical order. Minshuo Chen, Haoming Jiang, and Tuo Zhao are affiliated with School of Industrial and Systems Engineering at Georgia Tech; Wenjing Liao is affiliated with School of Mathematics at Georgia Tech; Email:$\{$mchen393, hmjiang, tourzhao, wliao60$\}$@gatech.edu.}}

\newcommand{\commentout}[1]{}

\begin{document}

\maketitle

\begin{abstract}
Real world data often exhibit low-dimensional geometric structures, and can be viewed as samples near a low-dimensional manifold. This paper studies nonparametric regression of H\"{o}lder functions on low-dimensional manifolds using deep ReLU networks. Suppose $n$ training data are sampled from a H\"{o}lder function in $\cH^{s,\alpha}$ supported on a $d$-dimensional Riemannian manifold isometrically embedded in $\RR^D$, with sub-gaussian noise.
A deep ReLU network architecture is designed to estimate the underlying function from the training data.
The mean squared error of the empirical estimator is proved to converge in the order of $n^{-\frac{2(s+\alpha)}{2(s+\alpha) + d}}\log^3 n$. This result shows that deep ReLU networks give rise to a fast convergence rate depending on the data intrinsic dimension $d$, which is usually much smaller than the ambient dimension $D$. It therefore demonstrates  the adaptivity of deep ReLU networks to low-dimensional geometric structures of data, and partially explains the power of deep ReLU networks in tackling high-dimensional data with low-dimensional geometric structures.
\end{abstract}


\section{Introduction}\label{sec:intro}

Deep learning has made astonishing breakthroughs in various real-world applications, such as computer vision \citep{krizhevsky2012imagenet, goodfellow2014generative, Long_2015_CVPR}, natural language processing \citep{graves2013speech, bahdanau2014neural, young2018recent}, healthcare \citep{miotto2017deep, jiang2017artificial}, robotics \citep{gu2017deep}, etc. For example, in image classification, the winner of the $2017$ ImageNet challenge retained a top-$5$ error rate of $2.25\%$ \citep{hu2018squeeze}, while the data set consists of about $1.2$ million labeled high resolution images in $1000$ categories. In speech recognition, \citet{amodei2016deep} reported that deep neural networks outperformed humans with a $5.15\%$ word error rate on the LibriSpeech corpus constructed from audio books \citep{panayotov2015librispeech}. Such a data set consists of approximately $1000$ hours of $16$kHz read English speech from $8000$ audio books.

The empirical success of deep learning brings new challenges to the conventional wisdom of machine learning. Data sets in these applications are in high-dimensional spaces. 
In existing literature, a minimax lower bound has been established for the optimal algorithm of learning $C^s$ functions in $\RR^D$ \citep{gyorfi2006distribution, tsybakov2008introduction}.
Denote the underlying function by $f_0$.  
The minimax lower bound suggests a pessimistic sample complexity: To obtain an estimator $\hat{f}$ for each $C^s$ function $f_0$ with an $\epsilon$-error, uniformly for all $C^s$ functions (i.e., $\sup_{f_0 \in C^s}\| \hat{f} - f_0\|_{L_2} \leq \epsilon$ with $\|\cdot\|_{L_2}$ denoting the function $L_2$ norm), the optimal algorithm requires the sample size $n \gtrsim \epsilon^{-\frac{2s + D}{s}}$ in the worst scenario (i.e., when $f_0$ is the most difficult for the algorithm to estimate). We instantiate such a sample complexity bound to the ImageNet data set, which consists of RGB images with a resolution of $224 \times 224$. The theory above suggests that, to achieve an $\epsilon$-error, the number of samples has to scale as $\epsilon^{-224 \times 224 \times 3/s}$, where the smoothness parameter $s$ is significantly smaller than $224 \times 224 \times 3$. Setting $\epsilon = 0.1$ already gives rise to a huge number of samples far beyond practical applications, which well exceeds $1.2$ million labeled images in ImageNet.

To bridge the aforementioned gap between theory and practice, we take the low-dimensional geometric structures in data sets into consideration. This is motivated by the fact that real-world data sets often exhibit low-dimensional structures. Many images consist of projections of a three-dimensional object followed by some transformations, such as rotation, translation, and skeleton. This generating mechanism induces a small number of intrinsic parameters \citep{hinton2006reducing, osher2017low}. Speech data are composed of words and sentences following the grammar, and therefore have small degrees of freedom \citep{djuric2015hate}. More broadly, visual, acoustic, textual, and many other types of data often have low-dimensional geometric structures due to rich local regularities, global symmetries, repetitive patterns, or redundant sampling \citep{tenenbaum2000global, roweis2000nonlinear,coifman2005geometric,allard2012multi}. It is therefore reasonable to assume that data lie on a manifold $\cM$ of dimension $d \ll D$.

\subsection{Summary of main results}
In this paper, we study nonparametric regression problems \citep{wasserman2006all, gyorfi2006distribution, tsybakov2008introduction} using neural networks in exploitation of low-dimensional geometric structures of data. Specifically, we model data as samples from a probability measure supported on a $d$-dimensional Riemannian manifold $\cM$ isometrically embedded in $\RR^D$ where $d \ll D$. The goal is to recover the regression function $f_0$ supported on $\cM$ using the samples $S_n = \{(\xb_i, y_i)\}_{i=1}^n$ with $\xb \in \cM$ and $y \in \RR$. The $\xb_i$'s are i.i.d. sampled from a distribution $\cD_x$ on $\cM$, and the response $y_i$ satisfies $$y_i = f_0(\xb_i) + \xi_i,$$ where $\xi_i$'s are i.i.d. sub-Gaussian noise independent of $\xb_i$'s.

We use multi-layer ReLU (Rectified Linear Unit) neural networks to recover $f_0$. ReLU networks are widely used in computer vision, speech recognition, natural language processing, etc. \citep{nair2010rectified, glorot2011deep, maas2013rectifier}. These networks can ease the notorious vanishing gradient issue during training, which commonly arises with sigmoid or hyperbolic tangent activations \citep{glorot2011deep, Goodfellow-et-al-2016}. Given an input $\xb$, an $L$-layer ReLU neural network computes an output as
\begin{equation}\label{eq:reluf}
f(\xb) = W_L \cdot \textrm{ReLU}(W_{L-1} \cdots \textrm{ReLU}(W_1 \xb + \bbb_1) \cdots + \bbb_{L-1}) + \bbb_L,
\end{equation}
where $W_1, \dots, W_L$ and $\bbb_1, \dots, \bbb_L$ are weight matrices and vectors of proper sizes, respectively, and $\textrm{ReLU}(\cdot)$ denotes the entrywise rectified linear unit (i.e., $\textrm{ReLU}(a) = \max\{0, a\}$). We denote $\cF$ as a class of neural networks with bounded weight parameters and bounded output (we refer to $\cF$ as a ReLU network structure throughout the rest of the paper):
\begin{align}\label{eq:classf}
\cF(R, \kappa, L, p, K) & = \big\{f ~|~ f(\xb) \textrm{ in the form \eqref{eq:reluf} with $L$-layers and width bounded by $p$}, \nonumber \\
& \norm{f}_\infty \leq R,  \norm{W_i}_{\infty, \infty} \leq \kappa, \norm{\bbb_i}_\infty \leq \kappa ~\textrm{for}~ i = 1, \dots, L, \sum_{i=1}^L \norm{W_i}_0 + \norm{\bbb_i}_0 \leq K \big\},
\end{align}
where $\norm{\cdot}_0$ denotes the number of nonzero entries in a vector or a matrix, $\norm{\cdot}_\infty$ denotes $\ell_\infty$ norm of a function or entrywise $\ell_\infty$ norm of a vector. For a matrix $M$, we have $\norm{M}_{\infty, \infty} = \max_{i, j} |M_{ij}|$.

To obtain an estimator $\hat{f} \in \cF(R, \kappa, L, p, K)$ of $f_0$, we minimize the empirical quadratic risk
\begin{align}\label{eq:hat_f}
\hat{f}_n = \argmin_{f \in \cF(R, \kappa, L, p, K)} ~\hat{\cR}_n(f) = \argmin_{f \in \cF(R, \kappa, L, p, K)} ~ \frac{1}{n} \sum_{i=1}^n \left(f(\xb_i) - y_i\right)^2.
\end{align}
The subscript $n$ emphasizes that the estimator is obtained using $n$ pairs of samples. Our theory shows that $\hat{f}_n$ converges to $f_0$ at a fast rate depending on the intrinsic dimension $d$,
under some mild regularity conditions. We assume $f_0 \in \cH^{s+\alpha}(\cM)$ is an $(s+\alpha)$-H\"{o}lder function on $\cM$, where $s > 0$ is an integer and $\alpha \in (0, 1]$. For the network class $\cF(R, \kappa, L, p, K)$, we choose $$L = \tilde{O}\left(\frac{s+\alpha}{2(s+\alpha)+d}\log n\right), ~~ p = \tilde{O}\left(n^{\frac{d}{2(s+\alpha)+d}}\right), ~~ K = \tilde{O}\left(\frac{s+\alpha}{2(s+\alpha)+d} n^{\frac{d}{2(s+\alpha)+d}}\log n\right), ~~ R = \norm{f_0}_\infty,$$ and set $\kappa$ as a constant depending on $s$, $f_0$, and $\cM$. Here we use $\tilde{O}$ to hide factors depending on $s, d$ and logarithmic factors (e.g., $\log D$). Then the empirical minimizer $\hat{f}_n$ of \eqref{eq:hat_f} gives rise to
\begin{align*}
\EE\left[\int_\cM \left(\hat{f}_n(\xb) - f_0(\xb)\right)^2 d \cD_x(\xb) \right] \leq c (R^2 + \sigma^2) \left(n^{-\frac{2(s+\alpha)}{2(s+\alpha) + d}} + \frac{D}{n}\right) \log^3 n,
\end{align*}
where the expectation is taken over the training samples $S_n$, $\sigma^2$ is the variance proxy of sub-Gaussian noise $\xi_i$, and $c$ is a constant depending on $\log D$, $s$, $\kappa$, and $\cM$ (see a formal statement in Theorem \ref{thm:stat}).

Our theory implies that, in order to estimate an $(s+\alpha)$-H\"{o}lder function up to an $\epsilon$-error, the sample complexity is $n \gtrsim \epsilon^{-\frac{2(s+\alpha) + d}{s+\alpha}}$ up to a log factor. This sample complexity depends on the intrinsic dimension $d$, and thus largely improves on existing theories of nonparametric regression using neural networks, where the sample complexity scales as $\tilde{O}(\epsilon^{-\frac{2(s+\alpha) + D}{s+\alpha}})$ \citep{hamers2006nonasymptotic, kohler2005adaptive, kohler2016nonparametric, kohler2011analysis, schmidt2017nonparametric}. Our result partially explains the success of deep ReLU neural networks in tackling high-dimensional data with low-dimensional geometric structures.

An ingredient in our analysis is an efficient universal approximation theory of deep ReLU networks for $(s+\alpha)$-H\"{o}lder functions on $\cM$ (Theorem \ref{thm:bias}). A preliminary version of the approximation theory appeared in \citet{chen2019efficient}. Specifically, we show that, in order to uniformly approximate $(s+\alpha)$-H\"{o}lder functions on a $d$-dimensional manifold with an $\epsilon$-error, the network consists of at most $\tilde{O}(\log 1/ \epsilon + \log D)$ layers and $\tilde{O}(\epsilon^{-d/(s+\alpha)} \log 1/\epsilon + D  \log 1/\epsilon + D \log D)$ neurons and weight parameters (see Theorem \ref{thm:bias}). The network size in our approximation theory weakly depends on the data dimension $D$, which significantly improves on existing universal approximation theories of neural networks \citep{barron1993universal, mhaskar1996neural, lu2017expressive, hanin2017universal, yarotsky2017error}, where the network size scales as $\tilde{O}(\epsilon^{-D/(s+\alpha)})$. Figure \ref{fig:networksize} illustrates a huge gap between the network sizes used in practice \citep{tan2019efficientnet} and the required size predicted by existing theories, e.g., \citet{yarotsky2017error} for the ImageNet data set. Our approximation theory partially bridges this gap by exploiting the data intrinsic geometric structures, and justifies why neural networks of moderate size have achieved a great success in various applications. %
Meanwhile, our network size also matches its lower bound up to logarithmic factors for a given manifold $\cM$ (see Proposition \ref{thm:lowerbound}).

\begin{wrapfigure}{r}{0.54\textwidth}
\vspace{-0.1in}
\centering
\includegraphics[width = 0.52\textwidth]{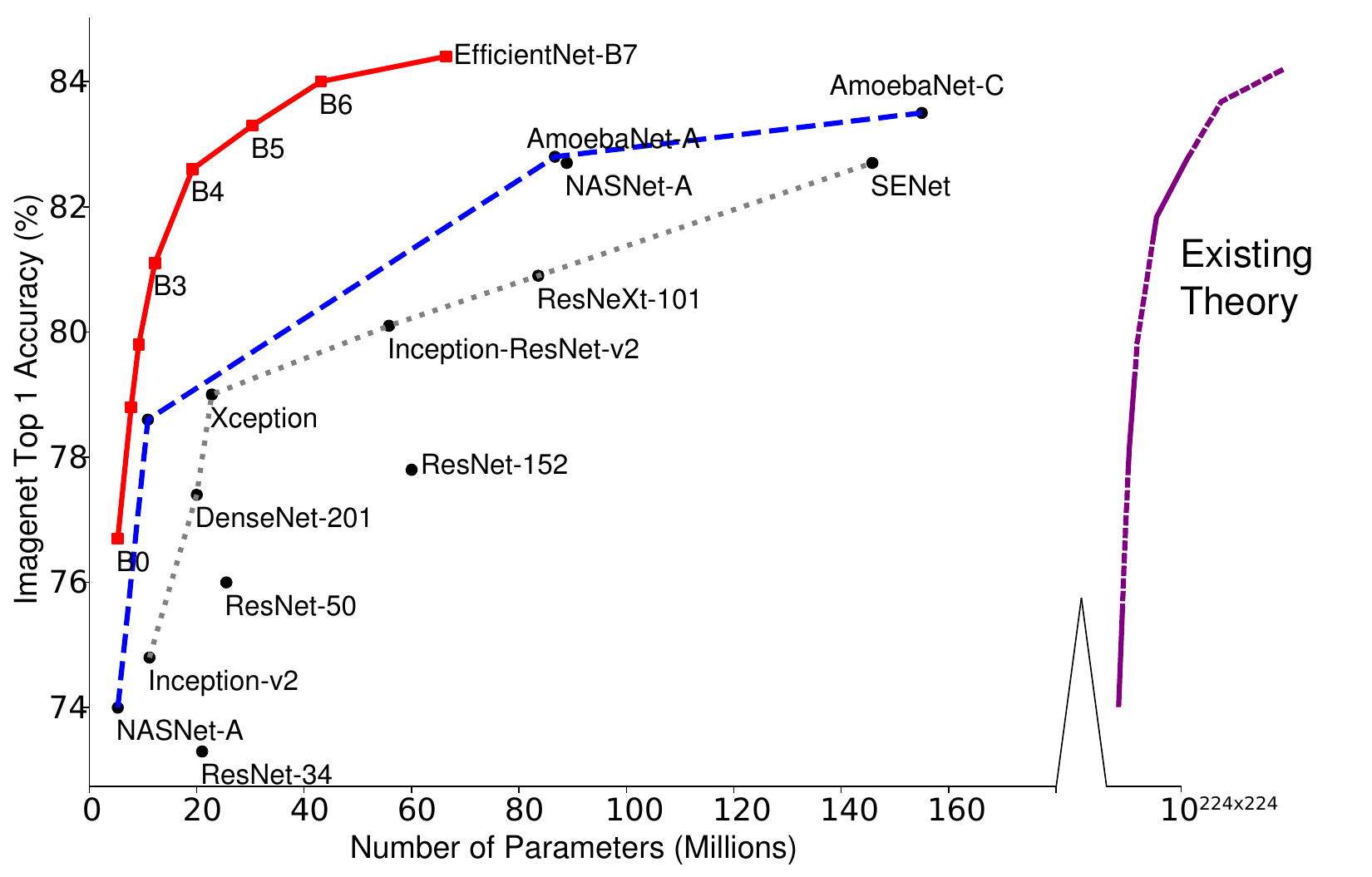}
\vspace{-0.1in}
\caption{Practical network sizes for the ImageNet data set \citep{tan2019efficientnet} versus the required size predicted by existing theories \citep{yarotsky2017error}.}
\label{fig:networksize}
\vspace{-0.2in}
\end{wrapfigure}

\subsection{Related Work}
Nonparametric regression has been widely studied in statistics. A variety of methods has been proposed to estimate the regression function, including kernel methods, wavelets, splines, and local polynomials \citep{wahba1990spline, altman1992introduction, fan1996local, tsybakov2008introduction, gyorfi2006distribution}. Nonetheless, there is limited study on regression using deep ReLU networks until recently. The earliest works focused on neural networks with a single hidden layer and smooth activations (e.g., sigmoidal and sinusoidal functions, \citep{barron1991complexity, mccaffrey1994convergence}). Later results achieved the minimax lower bound for the mean squared error in the order of $O(n^{-\frac{2s}{2s+D}})$ up to a logarithmic factor for $C^s$ functions in $\RR^D$ \citep{hamers2006nonasymptotic, kohler2005adaptive, kohler2016nonparametric, kohler2011analysis}. Theories for deep ReLU networks were developed in \citet{schmidt2017nonparametric}, where the estimate matches the minimax lower bound up to a logarithmic factor for H\"{o}lder functions. Extensions to more general function spaces, such as Besov spaces, can be found in \citet{suzuki2018adaptivity} and results for classification problems can be found in \citet{kim2018fast, ohn2019smooth}.

The rate of convergence in the results above cannot fully explain the success of deep learning due to the curse of the data dimension with a large $D$. Fortunately, many real-world data sets exhibit low-dimensional geometric structures. 
It has been demonstrated that, some classical methods are adaptive to the low-dimensional structures of data sets, and perform as well as if the low-dimensional structures were known. 
Results in this direction include local linear regression \citep{BickelLi, cheng2012local}, multiscale polynomial regression \citep{liao2021multiscale}, $k$-nearest neighbor \citep{NIPS2011_4455}, kernel regression \citep{NIPS2013_5103}, and Bayesian Gaussian process regression \citep{yang2015minimax},  where optimal rates depending on the intrinsic dimension were proved for functions having the second order of continuity \citep{BickelLi}, globally Lipschitz functions \citep{NIPS2011_4455}, and H\"older functions with H\"older index no more than $1$ \citep{NIPS2013_5103}. 

Recently, several independent works \citep{schmidt2019deep, nakada2020adaptive, cloninger2020relu} justified the adaptability of deep neural networks to the low-dimensional data structures. \citet{schmidt2019deep} considered function approximation and regression of H\"{o}lder functions on a low-dimensional manifold, which is similar to the setup in this paper. The proofs in \citet{schmidt2019deep} and this paper both utilize a collection of charts to map each point on $\cM$ into a local coordinate in $\RR^d$, and then approximate functions in $\RR^d$. There are two differences in the detailed proof: (1) In exploitation of a positive reach property of $\cM$, we construct local coordinates on the manifold given by orthogonal projections onto the tangent spaces, while \citet{schmidt2019deep} assumed the existence of smooth local coordinates; (2) A main novelty of our work is to explicitly construct a chart determination sub-network which assigns each data point to its proper chart. In \citet{schmidt2019deep}, the chart determination is realized by the partition of unity. In order to approximate functions in $\cH^{s,\alpha}(\cM)$, \citet{schmidt2019deep} required a uniform upper bound on the derivatives of each coordinate map and each function in the partition of unity, up to order $(s+\alpha) D/d$. Our proof does not rely on such regularity conditions depending on the ambient dimension $D$. To describe the intrinsic dimensionality of data, \citet{nakada2020adaptive} applied the notion of Minkowski dimension, which can be defined for a broader class of sets without smoothness restrictions. The intrinsic dimension of manifolds and the Minkowski dimension are different notions for low-dimensional sets, and one does not naturally imply the other.
\citet{schmidt2019deep} and \citet{nakada2020adaptive} established a $O(n^{-\frac{2(s+\alpha)}{2(s+\alpha)+d}})$ convergence rate of the mean squared error for learning functions in $\cH^{s,\alpha}(\cM)$, where $d$ is the manifold dimension in \citet{schmidt2019deep} and Minkowski dimension in \citet{nakada2020adaptive}, respectively. 
Recently \citet{cloninger2020relu} studied the approximation and regression error of ReLU neural networks for a class of functions in the form of $f(\xb) = g(\pi_{\cM}(\xb))$, where $\xb$ is near the low-dimensional manifold $\cM$, $\pi_{\cM}$ is a projection onto $\cM$, and $g$ is a H\"older function on $\cM$.

A crucial ingredient in the statistical analysis of neural networks is the universal approximation ability of neural networks. Early works in literature justified the existence of two-layer networks with continuous sigmoidal activations (a function $\sigma(x)$ is sigmoidal, if $\sigma(x) \rightarrow 0$ as $x \rightarrow -\infty$, and $\sigma(x) \rightarrow 1$ as $x \rightarrow \infty$) for a universal approximation of continuous functions in a unit hypercube \citep{irie1988capabilities, funahashi1989approximate, cybenko1989approximation, hornik1991approximation, chui1992approximation, leshno1993multilayer}. In these works, the number of neurons was not explicitly given. Later, \citet{barron1993universal, mhaskar1996neural} proved that the number of neurons can grow as $\epsilon^{-D/2}$ where $\epsilon$ is the uniform approximation error. Recently, \citet{lu2017expressive, hanin2017universal} and \citet{daubechies2019nonlinear} extended the universal approximation theory to networks of bounded width with ReLU activations. The depth of such networks grows exponentially with respect to the dimension of data. \citet{yarotsky2017error} showed that ReLU neural networks can uniformly approximate functions in Sobolev spaces, where the network size scales exponentially with respect to the data dimension and matches the lower bound. \citet{zhou2019universality} also developed a universal approximation theory for deep convolutional neural networks \citep{krizhevsky2012imagenet}, where the depth of the network scales exponentially with respect to the data dimension.

The aforementioned results focus on functions on a compact subset (e.g., $[0, 1]^D$) in $\RR^D$. Function approximation on manifolds has been well studied using classical methods, such as local polynomials \citep{BickelLi} and wavelets \citep{coifman2006diffusion}. However, studies using neural networks are limited.  Two noticeable works are \citet{chui2016deep} and \citet{shaham2018provable}. In \citet{chui2016deep}, high order differentiable functions on manifolds are approximated by neural networks with smooth activations, e.g., sigmoid activations and rectified quadratic unit functions ($\max^2\{0, x\}$).
These smooth activations are not commonly used in mainstream applications such as computer vision \citep{krizhevsky2012imagenet, Long_2015_CVPR, hu2018squeeze}. 
In \citet{shaham2018provable}, a $4$-layer network with ReLU activations was proposed to approximate $C^2$ functions on low-dimensional manifolds.
This theory does not cover arbitrarily $C^s$ functions.
We are also aware of a concurrent work of ours, \citet{shen2019deep}, which established an approximation theory of ReLU networks for H\"{o}lder functions in terms of a modulus of continuity. When the target function belongs to the H\"older class $\cH^{0,\alpha}$ supported in a neighborhood of a $d$-dimensional manifold embedded in $\RR^D$, \citet{shen2019deep} constructed a ReLU network which yields an approximation error in the order of $N^{-2 \alpha/{d_\delta}}L^{-2\alpha/{d_\delta}}$ where $N$ and $L$ are the width and depth of the network, and $d<d_{\delta} <D$.
Their proof utilizes a different approach compared to ours: They first construct a piecewise constant function to approximate the target function, and then implement the piecewise constant function using a ReLU network. The higher order smoothness for $\cH^{s,\alpha}$ functions while $s+\alpha>1$ is not exploited due to the use of piecewise constant approximations.

\subsection{Roadmap and Notations}

The rest of the paper is organized as follows: Section \ref{sec:pre} presents a brief introduction to manifolds and functions on manifolds. Section \ref{sec:theory} presents a statistical estimation theory of functions on low-dimensional manifolds using deep ReLU neural networks, and a universal approximation theory; Section \ref{sec:approxproof} sketches the proof of the approximation theory. Section \ref{sec:statproof} sketches the proof of the statistical estimation theory in Section \ref{sec:theory}, and the detailed proofs are deferred to Appendix; Section \ref{sec:discuss} provides a conclusion of the paper.

We use bold-faced letters to denote vectors, and normal font letters with a subscript to denote its coordinate, e.g., $\xb \in \RR^d$ and $x_k$ being the $k$-th coordinate of $\xb$. Given a vector $\sbb= [s_1, \dots, s_d]^\top \in \NN^d$, we define $\sbb! = \prod_{i=1}^d s_i!$ and $|\sbb| = \sum_{i=1}^d s_i$. We define $\xb^\sbb = \prod_{i=1}^d x_i^{s_i}$. Given a function $f : \RR^d \mapsto \RR$, we denote its derivative as $D^\sbb f = \frac{\partial^{|\sbb|} f}{\partial x_1^{s_1} \dots \partial x_d^{s_d}}$, and its $\ell_\infty$ norm as $\norm{f}_\infty = \max_\xb |f(\xb)|$. We use $\circ$ to denote the composition operator.


\section{Preliminaries}\label{sec:pre}

We briefly review manifolds, partition of unity, and function spaces defined on smooth manifolds. Details can be found in \citet{tu2010introduction} and \citet{lee2006riemannian}. 
Let $\cM$ be a $d$-dimensional Riemannian manifold isometrically embedded in $\RR^D$.
\begin{definition}[Chart]
A chart for $\cM$ is a pair $(U, \phi)$ such that $U \subset \cM$ is open and $\phi : U \mapsto \RR^d,$ where $\phi$ is a homeomorphism (i.e., bijective, $\phi$ and $\phi^{-1}$ are both continuous).
\end{definition}
The open set $U$ is called a coordinate neighborhood, and $\phi$ is called a coordinate system on $U$. A chart essentially defines a local coordinate system on $\cM$. Given a suitable coordinate neighborhood $U$ around a point $\cbb$ on the manifold $\cM$, we denote ${\sf P}_{\cbb}$ as the orthogonal projection onto the tangent space at $\cbb$, which gives a particular coordinate system on $U$.
\begin{example}[Projection to Tangent Space]
Let $T_{\cbb}(\cM)$ be the tangent space of $\cM$ at the point $\cbb \in \cM$ (see the formal definition in \citet[Section 8.1]{tu2010introduction}). We denote $\vb_1, \dots, \vb_d$ as an orthonormal basis of $T_{\cbb}(\cM)$. Then the orthogonal projection onto the tangent space $T_{\cbb}(\cM)$ is defined as ${\sf P}_{\cbb}(\xb) = V^\top (\xb - \cbb)$ for $\xb \in U$ with $V = [\vb_1, \dots, \vb_d] \in \RR^{D \times d}$.
\end{example}

We say two charts $(U, \phi)$ and $(V, \psi)$ on $\cM$ are $C^k$ compatible if and only if the transition functions,
$$\phi \circ \psi^{-1} : \psi(U \cap V) \mapsto \phi(U \cap V) \quad \textrm{and} \quad \psi \circ \phi^{-1} : \phi(U \cap V) \mapsto \psi(U \cap V)$$ are both $C^k$. 
\begin{definition}[$C^k$ Atlas]
\label{def:atlas}
A $C^k$ atlas for $\cM$ is a collection of pairwise $C^k$ compatible charts $\{(U_i, \phi_i)\}_{i \in \cA}$ such that $\bigcup_{i \in \cA} U_i = \cM$.
\end{definition}
\begin{definition}[Smooth Manifold]\label{def:smoothmanifold}
A smooth manifold is a manifold together with a $C^\infty$ atlas.
\end{definition}
Classical examples of smooth manifolds are the Euclidean space $\RR^D$, the torus, and the unit sphere. We further define a Riemannian manifold as a pair $(\cM, g)$, where $\cM$ is a smooth manifold and $g$ is a Riemannian metric \citep[Chapter 2]{lee2018introduction}. To better interpret Definition \ref{def:atlas} and \ref{def:smoothmanifold}, we give an example of a $C^\infty$ atlas on the unit sphere in $\RR^3$.
\begin{example}
We denote $\SSS^2$ as the unit sphere in $\RR^3$, i.e., $x^2+y^2+z^2 = 1$. The following atlas of $\SSS^2$ consists of $6$ overlapping charts $(U_1, {\sf P}_1), \dots, (U_6, {\sf P}_6)$ corresponding to hemispheres:
\begin{align*}
& U_1 = \{(x, y, z) ~|~ x > 0\}, ~{\sf P}_1(x, y, z) = (y, z), \quad U_2 = \{(x, y, z) ~|~ x < 0\}, ~{\sf P}_2(x, y, z) = (y, z), \\
& U_3 = \{(x, y, z) ~|~ y > 0\}, ~{\sf P}_3(x, y, z) = (x, z), \quad U_4 = \{(x, y, z) ~|~ y < 0\}, ~{\sf P}_4(x, y, z) = (x, z), \\
& U_5 = \{(x, y, z) ~|~ z > 0\}, ~{\sf P}_5(x, y, z) = (x, y), \quad U_6 = \{(x, y, z) ~|~ z < 0\}, ~{\sf P}_6(x, y, z) = (x, y).
\end{align*}
Here ${\sf P}_i$ is the orthogonal projection onto the tangent space at the pole of each hemisphere. Moreover, all the six charts are $C^\infty$ compatible, and therefore, $(U_1, {\sf P}_1), \dots, (U_6, {\sf P}_6)$ form an atlas of $\SSS^2$.

For a general compact smooth manifold $\cM$, we can construct an atlas using orthogonal projections to tangent spaces as local coordinate systems. Let ${\sf P}_{\cbb}$ be the orthogonal projection to the tangent space $T_{\cbb}(\cM)$ for $\cbb\in \cM$. Let $U_{\cbb}$ be an open coordinate neighborhood containing $\cbb$ such that ${\sf P}_{\cbb}$ is a homeomorphism. 
Since $\cM$ is compact, there exist a finite number of points $\{\cbb_i\}$ such that the
 charts $\{(U_{\cbb_i}, {\sf P}_{\cbb_i})\}$ form an atlas of $\cM$.
\end{example}
The existence of an atlas on $\cM$ allows us to define differentiable functions.
\begin{definition}[$C^s$ Functions on $\cM$]
Let $\cM$ be a $d$-dimensional Riemannian manifold isometrically embedded in $\RR^D$. A function $f: \cM \mapsto \RR$ is $C^s$ if for any chart $(U, \phi)$, the composition $f \circ \phi^{-1}: \phi(U) \mapsto \RR$ is continuously differentiable up to order $s$.
\end{definition}
\begin{remark}
The definition of $C^s$ functions is independent of the choice of the chart $(U, \phi)$. Suppose $(V, \psi)$ is another chart and $V \bigcap U \neq \emptyset$. Then we have $$f \circ \psi^{-1} = (f \circ \phi^{-1}) \circ (\phi \circ \psi^{-1}).$$ Since $\cM$ is a smooth manifold, $(U, \phi)$ and $(V, \psi)$ are $C^\infty$ compatible. Thus, $f \circ \phi^{-1}$ is $C^s$ and $\phi \circ \psi^{-1}$ is $C^\infty$, and their composition is $C^s$.
\end{remark}

We next generalize the definition of $C^s$ functions to H\"{o}lder functions on the smooth manifold $\cM$.
\begin{definition}[H\"{o}lder Functions on $\cM$]\label{def:holder}
Let $\cM$ be a $d$-dimensional compact Riemannian manifold isometrically embedded in $\RR^D$. Let $\{(U_i, {\sf P}_i)\}_{i \in \cA}$ be an atlas of $\cM$ where the ${\sf P}_i$'s are orthogonal projections onto tangent spaces. For a positive integer $s$ and $\alpha \in (0, 1]$, a function $f : \cM \mapsto \RR$ is $(s+\alpha)$-H\"{o}lder continuous if for each chart $(U_i, {\sf P}_i)$ in the atlas, we have
\begin{enumerate}
\item $f \circ {\sf P}_i^{-1} \in C^{s}$ with $|D^{\sbb} (f \circ {\sf P}_i^{-1})| \leq 1$ for any $|\sbb| \leq s, \xb \in U_i$;
\item for any $|\sbb| = s$ and $\xb_1, \xb_2 \in U_i$,
\begin{align}
\left|D^{\sbb} (f \circ {\sf P}_i^{-1}) \big|_{{\sf P}_i(\xb_1)} - D^{\sbb} (f \circ {\sf P}_i^{-1}) \big|_{{\sf P}_i(\xb_2)} \right| \leq \norm{{\sf P}_i(\xb_1) - {\sf P}_i(\xb_2)}_2^\alpha.
\end{align}
\end{enumerate}
Moreover, we denote the collection of $(s+\alpha)$-H\"{o}lder functions on $\cM$ as $\cH^{s, \alpha}(\cM)$.
\end{definition}

Definition \ref{def:holder} requires that all $s$-th order derivatives of $f \circ {\sf P}_i^{-1}$ are H\"{o}lder continuous. 
We recover the standard H\"{o}lder class on a Euclidean space if ${\sf P}_i$ is the identity mapping.
We next introduce the partition of unity, which plays a crucial role in our construction of neural networks.
\begin{definition}[Partition of Unity, Definition 13.4 in \citet{tu2010introduction}]
A $C^\infty$ partition of unity on a manifold $\cM$ is a collection of nonnegative $C^\infty$ functions $\rho_i: \cM \mapsto \RR_+$ for $i \in \cA$ such that
\begin{enumerate}
\item the collection of supports, $\{\textrm{supp} (\rho_i)\}_{i \in \cA}$ is locally finite, i.e., every point on $\cM$ has a neighborhood that meets only finitely many of ${\rm supp} (\rho_i)$'s; 
\item $\displaystyle\sum \rho_i = 1$.
\end{enumerate}
\end{definition}
For a smooth manifold, a $C^\infty$ partition of unity always exists.
\begin{proposition}[Existence of a $C^\infty$ partition of unity, Theorem 13.7 in \citet{tu2010introduction}]\label{thm:parunity}
Let $\{U_i\}_{i \in \cA}$ be an open cover of a compact smooth manifold $\cM$. Then there is a $C^\infty$ partition of unity $\{\rho_i\}_{i \in \cA}$ where every $\rho_i$ has a compact support such that $\textrm{supp}(\rho_i) \subset U_i$.
\end{proposition}
Proposition \ref{thm:parunity} gives rise to the decomposition $f = \sum_{i=1}^\infty f_i$ with $f_i = f \rho_i$. Note that the $f_i$'s have the same regularity as $f$, since $$f_i \circ \phi_i^{-1} = (f \circ \phi_i^{-1}) \times (\rho_i \circ \phi_i^{-1})$$ for a chart $(U_i, \phi_i)$. 
This decomposition implies that we can express $f$ as a sum of the $f_i$'s, where every $f_i$ is only supported in a single chart.

To characterize the curvature of a manifold, we adopt the following geometric concept.
\begin{definition}[Reach \citep{federer1959curvature}, Definition 2.1 in \citet{aamari2019estimating}]
Denote $$\cC(\cM) = \left\{\xb \in \RR^D : \exists~ \pb \neq \qb \in \cM, \norm{\pb - \xb}_2 = \norm{\qb - \xb}_2 = \inf_{\yb \in \cM} \norm{\yb - \xb}_2\right\}$$ as the set of points that have at least two nearest neighbors on $\cM$. The reach $\tau > 0$ is defined as $$\tau = \inf_{\xb \in \cM, \yb \in \cC(\cM)} \norm{\xb - \yb}_2.$$
\end{definition}
\begin{figure}[!htb]
\centering
\includegraphics[width = 0.6\textwidth]{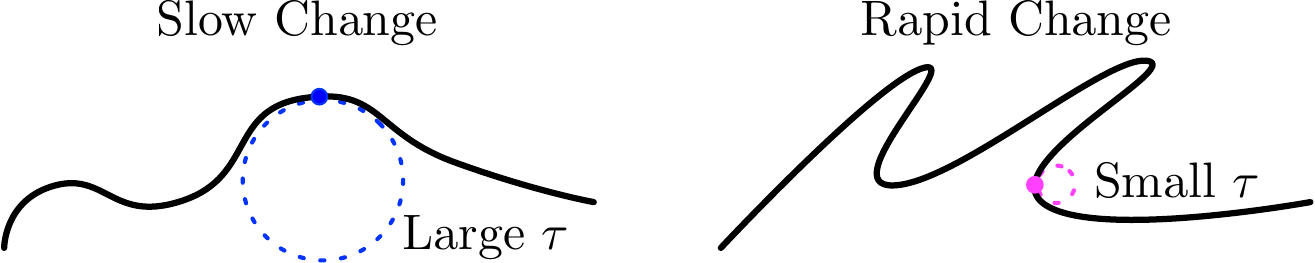}
\caption{Manifolds with large and small reaches.}
\label{fig:reach}
\end{figure}
Reach has a straightforward geometrical interpretation: At each point $\xb \in \cM$, the radius of the osculating circle is greater or equal to $\tau$. Intuitively, a large reach for $\cM$ requires the manifold $\cM$ not to change ``rapidly'' as shown in Figure \ref{fig:reach}.

In our proof for the universal approximation theory, reach determines a proper choice of an atlas for $\cM$. In Section \ref{sec:approxproof}, we choose each chart $U_i$ to be contained in a ball of radius less than $\tau / 2$. For smooth manifolds with a small $\tau$, we need a large number of charts. Therefore, reach of a smooth manifold reflects the complexity of the neural network for function approximation on $\cM$.


\section{Main Results}\label{sec:theory}

This section contains our main statistical estimation theory for H\"older functions on low-dimensional manifolds using deep neural networks. We begin with some assumptions on the regression model and the manifold $\cM$. 

\begin{assumption}\label{assump1}
$\cM$ is a $d$-dimensional compact Riemannian manifold isometrically embedded in $\RR^D$. There exists a constant $B > 0$ such that, for any point $\xb \in \cM$, we have $|x_j| \leq B$ for all $j = 1, \dots, D$.
\end{assumption}
\begin{assumption}\label{assump2}
The reach of $\cM$ is $\tau > 0$.
\end{assumption}

\begin{assumption}\label{assump3}
The ground truth function $f_0: \cM \mapsto \RR$ belongs to the H\"{o}lder space $\cH^{s, \alpha}(\cM)$ with a positive integer $s$ and $\alpha \in (0, 1]$. 
\end{assumption}
\begin{assumption}\label{assump4}
The noise $\xi_i$'s are i.i.d. sub-Gaussian with $\EE[\xi_i] = 0$ and variance proxy $\sigma^2$, which are  independent of the $\xb_i$'s. 
\end{assumption}

\subsection{Universal Approximation Theory}

An accurate estimation of the nonparametric regression function $f_0$ necessitates the existence of a good approximation of $f_0$ by our learning models --- neural networks. To aid the choice of a proper neural network class for learning $f_0$, we first investigate the following questions:

\begin{itemize}

\item  Given a desired approximation error $\epsilon > 0$, does there exist a ReLU neural network which universally represents H\"{o}lder functions supported on $\cM$?

\item If the answer is yes, what is the corresponding network architecture? \textcolor{red}{}

\end{itemize} We provide a positive answer in the theorem below and defer the proof to Section \ref{sec:approxproof}.
\begin{theorem}\label{thm:bias}
Suppose Assumptions \ref{assump1} and \ref{assump2} hold. Given any $\epsilon \in (0, 1)$, there exists a ReLU network structure $\cF(\cdot, \kappa, L, p, K)$, such that, for any $f: \cM \rightarrow \RR$ satisfying Assumption \ref{assump3}, if the weight parameters of the network are properly chosen, the network yields a function $\tilde{f}$ satisfying $\lVert \tilde{f} - f \rVert_\infty \leq \epsilon.$ Such a network has 
\begin{enumerate}
\item no more than $L = c_1 (\log \frac{1}{\epsilon} + \log D)$ layers, with width bounded by $p = c_2( \epsilon^{-\frac{d}{s+\alpha}} + D)$,
\item at most $K = c_3 (\epsilon^{-\frac{d}{s+\alpha}} \log \frac{1}{\epsilon} + D \log \frac{1}{\epsilon} + D \log D)$ neurons and weight parameters, with the range of weight parameters bounded by $\kappa = c_4\max\{1, B, \tau^2, \sqrt{d}\}$,
\end{enumerate}
where $c_1, c_2, c_3$ depend on $d$, $s$, $\tau$, $B$, the surface area of $\cM$, and the upper bounds on the derivatives of the coordinate systems $\phi_i$'s and the $\rho_i$'s in the partition of unity, up to order $s$, and $c_4$ depends on the upper bound on the derivatives of the $\rho_i$'s, up to order $s$.
\end{theorem}
This network class $\cF$ will be used later to estimate a regression function in Theorem \ref{thm:stat}. Our approximation theory does not require the output range to be bounded by $R$ in the network class (or equivalently by setting $R = +\infty$). The enforcement of $\|f\|_{\infty} \le R$
is to be imposed for regression in order to control the variance in statistical estimations.  

The network structure identified by Theorem \ref{thm:bias} consists of three sub-networks as shown in Figure \ref{fig:network} (The detailed construction of each sub-network is postponed to Section \ref{sec:approxproof}):
\begin{itemize}
\item \emph{Chart determination sub-network}, which assigns each input to its corresponding neighborhood;

\item \emph{Taylor approximation sub-network}, which approximates $f$ by polynomials in each neighborhood;

\item \emph{Pairing sub-network}, which yields multiplications of the proper pairs of the outputs from the chart determination and the Taylor approximation sub-networks.
\end{itemize}

\begin{figure}[!htb]
\centering
\includegraphics[width = 0.9\textwidth]{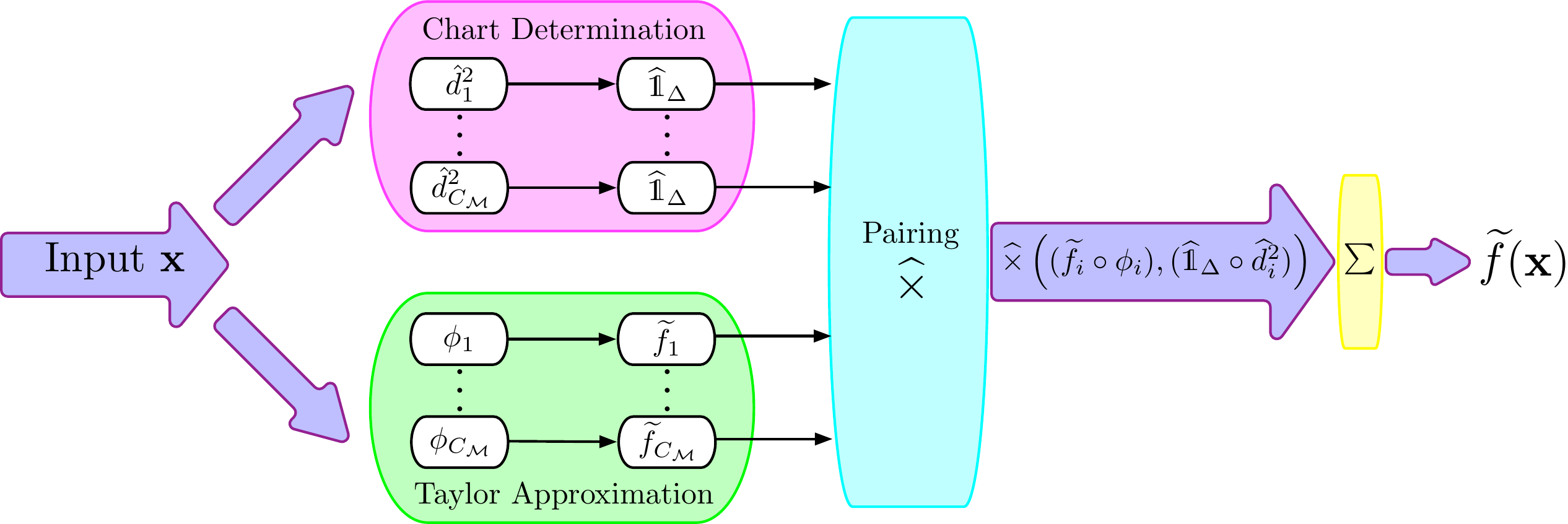}
\caption{The ReLU network identified by Theorem \ref{thm:bias}.}
\label{fig:network}
\end{figure}


Theorem \ref{thm:bias} significantly improves on existing approximation theories \citep{yarotsky2017error}, where the network size grows exponentially with respect to the ambient dimension $D$, i.e. $\epsilon^{-D/(s+\alpha)}$.
%
Theorem \ref{thm:bias} also improves \citet{shaham2018provable} for $C^s$ functions in the case that $s > 2$. When $s>2$, our network size scales like $\epsilon^{-d/s}$, which is significantly smaller than the one in \citet{shaham2018provable} in the order of $\epsilon^{-d/2}$.

Our approximation theory can be directly generalized to the Sobolev space $\cW^{k,\infty}$, which is embedded in $C^{k}$. The reason is that our proof of Theorem \ref{thm:bias} relies on local Taylor polynomial approximations of H\"older functions. For general Sobolev spaces $\cW^{k,p}$, one needs to consider averaged Taylor polynomials and the Bramble-Hilbert lemma \citep[Lemma 4.3.8]{brenner2007mathematical}. We refer to \citet{guhring2020error} for readers' interests.

Moreover, the size of our ReLU network in Theorem \ref{thm:bias} matches the lower bound in \citet{devore1989optimal} up to a logarithmic factor for the approximation of functions in the H\"{o}lder space $\cH^{s-1,1}([0,1]^d)$ defined on $[0,1]^d$. 
\begin{proposition}\label{thm:lowerbound}
Fix $d$ and $s$. Let $W$ be a positive integer and $\mathcal{T}: \RR^W \mapsto C([0, 1]^d)$ be any mapping. Suppose there is a continuous map $\Theta : \cH^{s-1,1}([0,1]^d) \mapsto \RR^W$ such that $\lVert f - \mathcal{T}(\Theta(f)) \rVert_\infty \leq \epsilon$ for any $f \in \cH^{s-1,1}([0,1]^d)$. Then $W \geq c \epsilon^{-\frac{d}{s}}$ with $c$ depending on $s$ only.
\end{proposition}
We take $\RR^W$ as the parameter space of a ReLU network, and $\mathcal{T}$ as the transformation given by the ReLU network. Theorem \ref{thm:lowerbound} implies that, to approximate any $f \in \cH^{s-1,1}([0,1]^d)$, the ReLU network needs to have at least $c\epsilon^{-\frac{d}{s}}$ weight parameters. Although Proposition \ref{thm:lowerbound} holds for functions defined on $[0,1]^d$, our network size remains in the same order up to a logarithmic factor even when the function is supported on a manifold of dimension $d$.

On the other hand, the lower bound also reveals that the low-dimensional manifold model plays a vital role to reduce the network size. To uniformly approximate functions in $\cH^{s-1,1}([0,1]^D)$ with an accuracy $\epsilon$, the minimal number of weight parameters is $O( \epsilon^{-\frac{D}{s}})$. This lower bound cannot be improved without low-dimensional structures of data.

\subsection{Statistical Estimation Theory}
Based on Theorem \ref{thm:bias}, we next present our main regression theorem, which characterizes the convergence rate for the estimation of $f_0$ using ReLU neural networks.
\begin{theorem}\label{thm:stat}
Suppose Assumptions \ref{assump1} - \ref{assump3} hold. Let $\hat{f}_n$ be the minimizer of empirical risk \eqref{eq:hat_f} with the network class $\cF(R, \kappa, L, p, K)$ properly designed such that
\begin{align*}
& L = \tilde{O}\left(\frac{s+\alpha}{2(s+\alpha) + d}\log n\right), \quad p = \tilde{O}\left(n^{\frac{d}{2(s+\alpha)+d}}\right), \quad K = \tilde{O}\left(\frac{s+\alpha}{2(s+\alpha) + d} n^{\frac{d}{2(s+\alpha)+d}}\log n\right), \\
& \hspace{1.5in} R = \norm{f_0}_\infty, \quad \textrm{~and~}\quad \kappa = O(\max\{1, B, \sqrt{d}, \tau^2\}).
\end{align*}
Then we have
\begin{align*}
\EE\left[\int_\cM \left(\hat{f}_n(\xb) - f_0(\xb)\right)^2 d \cD_x(\xb) \right] \leq c (R^2 + \sigma^2) \left(n^{-\frac{2(s+\alpha)}{2(s+\alpha)+d}} + \frac{D}{n} \right) \log^3 n,
\end{align*}
where the expectation is taken over the training samples $S_n$, and $c$ is a constant depending on $\log D$, $d$, $s$, $\tau$, $B$, the surface area of $\cM$, and the upper bounds of derivatives of the coordinate systems $\phi_i$'s and partition of unity $\rho_i$'s, up to order $s$.
\end{theorem}

Theorem \ref{thm:stat} is established by a bias-variance trade-off. We decompose the mean squared error to a squared bias term and a variance term. The bias is quantified by Theorem \ref{thm:bias}, and the variance term is proportional to the network size. A detailed proof of Theorem \ref{thm:stat} is provided in Section \ref{sec:statproof}. 
Here are some remarks: 
\begin{enumerate}
\item The network class in Theorem \ref{thm:stat} is sparsely connected, i.e. $K = O(Lp)$, while densely connected networks satisfy $K = O(Lp^2)$.

\item 
The network class $\cF(R, \kappa, L, p, K)$ has outputs uniformly bounded by $R$. Such a requirement can be achieved by appending an additional clipping layer to the end of the network structure, i.e.,
\begin{align*}
g(a) = \max\{-R, \min\{a, R\}\} = \textrm{ReLU} (a - R) - \textrm{ReLU}(a + R) - R.
\end{align*}

\item Each weight parameter in our network class is bounded by a constant $\kappa$ only depending on the curvature $\tau$, the range $B$ of the manifold $\cM$, and the manifold dimension $d$. Such a boundedness condition is crucial to our theory and can be computationally realized by normalization after each step of the stochastic gradient descent.
\end{enumerate}

\section{Proof of Approximation Theory}\label{sec:approxproof}
This section contains a proof sketch of Theorem \ref{thm:bias}. Before we proceed, we show how to approximate the multiplication operation using ReLU networks. This operation is heavily used in the Taylor approximation sub-network, since Taylor polynomials involve a sum of products. We first show ReLU networks can approximate quadratic functions.
\begin{lemma}[Proposition $2$ in \citet{yarotsky2017error}]\label{lemma:quad}
The function $f(x) = x^2$ with $x \in [0, 1]$ can be approximated by a ReLU network with any error $\epsilon > 0$. The network has depth and the number of neurons and weight parameters no more than $c \log(1/\epsilon)$ with an absolute constant $c$, and the width of the network is an absolute constant.
\end{lemma}
This lemma is proved in Appendix \ref{pf:quad}. The idea is to approximate quadratic functions using {a weighted sum of a series of sawtooth functions. Those sawtooth functions are obtained by compositing the triangular function $$g(x) = 2\textrm{ReLU}(x) - 4\textrm{ReLU}(x - 1/2) + 2\textrm{ReLU}(x - 1),$$ which can be implemented by a single layer ReLU network.} 

We then approximate the multiplication operation by invoking the identity $ab = \frac{1}{4}((a+b)^2 - (a-b)^2)$ where the two squares can be approximated by ReLU networks in Lemma \ref{lemma:quad}.
\begin{corollary}[Proposition $3$ in \citet{yarotsky2017error}]\label{cor:hatx}
Given a constant $C > 0$ and $\epsilon \in (0, C^2)$, there is a ReLU network which implements a function $\hat{\times} : \RR^2 \mapsto \RR$ such that: \textbf{1)}. For all inputs $x$ and $y$ satisfying $|x| \leq C$ and $|y| \leq C$, we have $|\hat{\times} (x, y) - xy | \leq \epsilon$; \textbf{2)}. The depth and the weight parameters of the network is no more than $c \log \frac{C^2}{\epsilon}$ with an absolute constant $c$.
\end{corollary}

The ReLU network in Theorem \ref{thm:bias} is constructed in the following 5 steps.

\textbf{Step 1. Construction of an atlas}. Denote the open Euclidean ball with center $\cbb$ and radius $r$ in $\RR^D$ by $\cB(\cbb, r)$. For any $r$, the collection $\{\cB(\xb, r)\}_{\xb \in \cM}$ is an open cover of $\cM$. Since $\cM$ is compact, there exists a finite collection of points $\cbb_i$ for $i = 1, \dots, C_\cM$ such that $\cM \subset \bigcup_i \cB(\cbb_i, r).$

The following lemma says that when the radius $r$ is properly chosen, $U_i = \cB(\cbb_i, r) \cap \cM$ is diffeomorphic to $\RR^d$.
\begin{lemma}\label{lemma:diffeomorphism}
Suppose Assumption \ref{assump1} and \ref{assump2} hold and let $r \leq \tau / 4$. Then the local neighborhood $U_i = \cB(\cbb_i, r) \cap \cM$ is diffeomorphic to $\RR^d$. In particular, the orthogonal projection ${\sf P}_i$ onto the tangent space $T_{\cbb_i}(\cM)$ at $\cbb_i$ is a diffeomorphism.
\end{lemma}
The proof is provided in Appendix \ref{pf:diffeomorphism}, which utilizes the results in \citet{niyogi2008finding}. Therefore, we pick radius $r \leq \tau / 4$, and let $\{(U_i, \phi_i)\}_{i=1}^{C_\cM}$ be an atlas on $\cM$ as illustrated in Figure \ref{fig:covering}, where $\phi_i$ is
\begin{wrapfigure}{r}{0.46\textwidth}
\vspace{-0.15in}
\centering
\includegraphics[width = 0.45\textwidth]{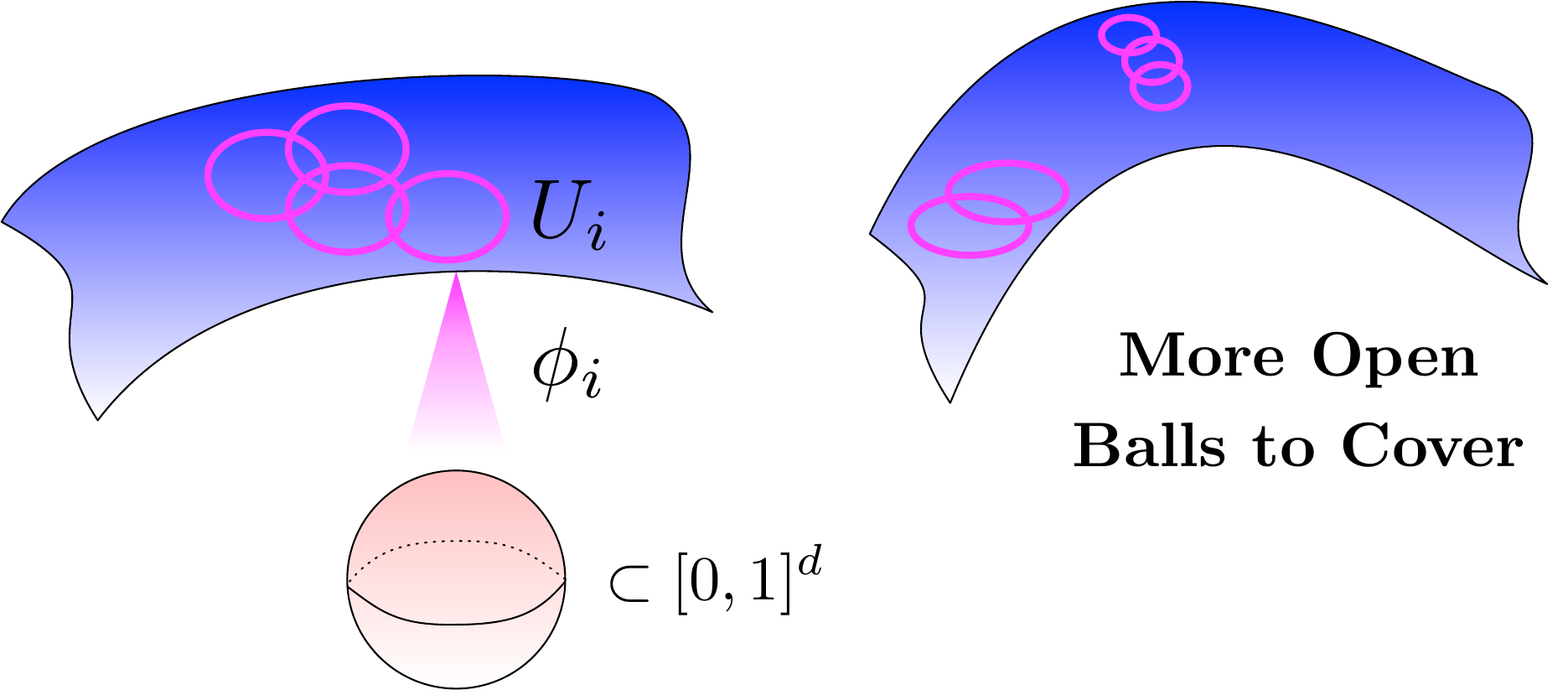}
\caption{Curvature decides the number of charts: smaller reach requires more chart.}
\label{fig:covering}
\vspace{-0.3in}
\end{wrapfigure}
to be defined in \textbf{Step 2}. The number of charts $C_\cM$ is upper bounded by
\begin{align}
C_\cM \leq \left\lceil\frac{SA(\cM)}{r^d} T_d\right\rceil, \notag
\end{align}
where $SA(M)$ is the surface area of $\cM$, and $T_d$ is the thickness of the $U_i$'s, which is defined as the average number of $U_i$'s that contain a point on $\cM$ (See Eq. (1) in Chapter $2$ of \citet{Conway:1987:SLG:39091}).

\begin{remark}
The thickness $T_d$ scales approximately linear in $d$. As shown in Eq. (19) in Chapter $2$ of \citet{Conway:1987:SLG:39091}, there exist coverings with $\frac{d}{e \sqrt{e}} \lesssim T_d \leq d \log d + d \log \log d + 5d$.
\end{remark}

\textbf{Step 2. Projection with rescaling and translation}. We denote the tangent space at $\cbb_i$ as $$T_{\cbb_i}(\cM) = \textrm{span}(\vb_{i1}, \dots, \vb_{id}),$$ where $\{\vb_{i1}, \dots, \vb_{id}\}$ form an orthonormal basis. We obtain the matrix $V_i = [\vb_{i1}, \dots, \vb_{id}] \in \RR^{D \times d}$ by concatenating the $\vb_{ij}$'s as column vectors.

Define $$\phi_i(\xb) = b_i (V_i^\top (\xb - \cbb_i) + \ub_i)\in [0, 1]^d$$ for any $\xb \in U_i$, where $b_i \in (0, 1]$ is a scaling factor and $\ub_i$ is a translation vector. Since $U_i$ is bounded, we can choose proper $b_i$ and $\ub_i$ to guarantee $\phi_i(\xb) \in [0, 1]^d$. We rescale and translate the projection to ease the notation for the development of local Taylor approximations in \textbf{Step 4}. We also remark that each $\phi_i$ is a linear function, and can be realized by a single layer linear network.

\textbf{Step 3. Chart determination}. This step is to assign a given input $\xb$ to the proper charts to which $\xb$ belongs. This avoids projecting $\xb$ using unmatched charts (i.e., $\xb \not\in U_j$ for some $j$) as illustrated in Figure \ref{fig:unmatch}.

\begin{wrapfigure}{r}{0.29\textwidth}
\vspace{-0.25in}
\centering
\includegraphics[width = 0.28\textwidth]{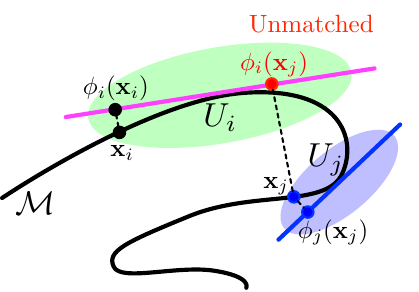}
\caption{Projecting $\xb_j$ using a matched chart (blue) $(U_j, \phi_j)$, and an unmatched chart (green) $(U_i, \phi_i)$.}
\label{fig:unmatch}
\vspace{-0.25in}
\end{wrapfigure}

An input $\xb$ can belong to multiple charts, and the chart determination sub-network determines all these charts. This can be realized by compositing an indicator function and the squared Euclidean distance $$d_i^2 (\xb) = \norm{\xb - \cbb_i}_2^2 = \sum_{j = 1}^D (x_j - c_{i, j})^2$$ for $i = 1, \dots, C_\cM$.
The squared distance $d_i^2 (\xb)$ is a sum of univariate quadratic functions, thus, we can apply Lemma \ref{lemma:quad} to approximate $d_i^2(\xb)$ by ReLU networks. Denote $\hat{h}_{\textrm{sq}}$ as an approximation of the quadratic function $x^2$ on $[0, 1]$ with an approximation error $\nu$. Then we define 
\begin{align}
\hat{d}_i^2(\xb) = 4B^2\sum_{j = 1}^D \hat{h}_{\textrm{sq}}\left(\left|\frac{x_j - c_{i, j}}{2B}\right|\right). \notag
\end{align}
as an approximation of $d_i^2(\xb)$.
The approximation error is $\lVert \hat{d}_i^2 - d_i^2 \rVert_\infty \leq 4B^2 D \nu$, by the triangle inequality. We consider an approximation of the indicator function $\mathds{1}(x \in [0, r^2])$ as in Figure \ref{fig:chartdetermine}:
\begin{align}
\hat{\mathds{1}}_{\Delta}(a) = 
\begin{cases}
1 & a \leq r^2 - \Delta + 4B^2 D\nu \\
- \frac{1}{\Delta - 8B^2 D\nu} a + \frac{r^2 - 4B^2 D\nu}{\Delta - 8B^2 D\nu} & a \in [r^2 - \Delta + 4B^2 D\nu, r^2 - 4B^2 D\nu] \\
0 & a > r^2 - 4B^2 D\nu
\end{cases},
\label{eq:approxindicator}
\end{align}
\begin{figure}[!htb]
\centering
\includegraphics[width = 0.7\textwidth]{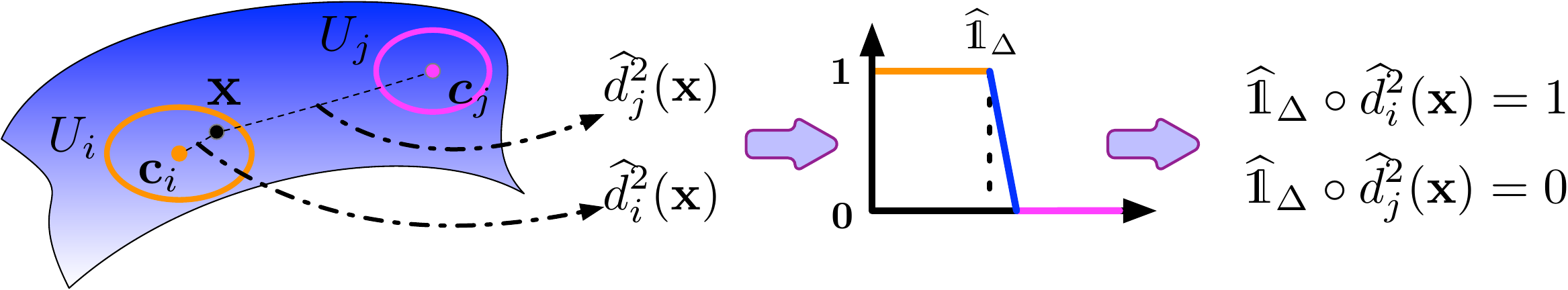}
\caption{Chart determination utilizes the composition of approximated distance function $\hat{d}_i^2$ and the indicator function $\hat{\mathds{1}}_\Delta$.}
\label{fig:chartdetermine}
\end{figure}
where $\Delta$ ($\Delta \ge 8 B^2 D\nu$) will be chosen later according to the accuracy $\epsilon$. 

To implement $\hat{\mathds{1}}_\Delta(a)$, we consider a basic step function $g = 2\textrm{ReLU}(x - 0.5(r^2 - 4B^2 D \nu)) - 2 \textrm{ReLU}(x - r^2 + 4B^2 D \nu)$. It is straightforward to check
\begin{align*}
g_k(a) = \underbrace{g\circ \cdots \circ g}_{k}(a) = \begin{cases}
0 & a < (1 - 2^{-k})(r^2 - 4B^2 D \nu)\\
2^k (a - r^2 + 4B^2 D\nu) + r^2 - 4B^2 D\nu & a \in \left[(1-\frac{1}{2^k})(r^2 - 4B^2 D\nu), r^2 - 4B^2 D\nu\right] \\
r^2 - 4B^2 D \nu & a > r^2 - 4B^2 D \nu
\end{cases}.
\end{align*}
Let $\hat{\mathds{1}}_\Delta = 1 - \frac{1}{r^2 - 4B^2 D \nu} g_k$. It suffices to choose $k$ satisfying $(1-\frac{1}{2^k})(r^2 - 4B^2 D\nu) \geq r^2 - \Delta + 4B^2 D\nu$, which yields $k =\left\lceil \log \frac{r^2}{\Delta}\right\rceil$.
We use $\hat{\mathds{1}}_\Delta \circ \hat{d}_i^2$ to approximate the indicator function on $U_i$: 
\begin{itemize}
\item if $\xb \not\in U_i$, i.e., $d_i^2(\xb) \geq r^2$, we have $\hat{\mathds{1}}_\Delta \circ \hat{d}_i^2 (\xb) = 0$; 
\item if $\xb \in U_i$ and $d_i^2(\xb) \leq r^2 - \Delta$, we have $\hat{\mathds{1}}_\Delta \circ \hat{d}_i^2 (\xb) = 1$.
\end{itemize}
We remark that although the approximate indicator function $\hat{\mathds{1}}_\Delta$ is a piecewise linear function with two breakpoints, we implement it using a deep neural network to control the range of weight parameters in the network. Otherwise, the parameter upper bound can be as large as $1/\Delta$ due to the steep slope in $\hat{\mathds{1}}_{\Delta}$, which undermines the statistical theory.

\textbf{Step 4. Taylor approximation}. In each chart $(U_i, \phi_i)$, we locally approximate $f$ using Taylor polynomials of order $n$ as shown in Figure \ref{fig:taylor}. Specifically, we decompose $f$ as $$f = \sum_{i=1}^{C_\cM} f_i \quad \textrm{with} \quad f_i = f \rho_i,$$ where $\rho_i$ is an element in a $C^\infty$ partition of unity on $\cM$ which is supported inside $U_i$. The existence of such a partition of unity is guaranteed by Proposition \ref{thm:parunity}. Since $\cM$ is a compact smooth manifold and $\rho_i$ is $C^\infty$, $f_i$ preserves the regularity (smoothness) of $f$ such that $f_i \in \cH^{s,\alpha}(\cM)$ for $i = 1, \dots, C_\cM$.

\begin{figure}[!htb]
\centering
\includegraphics[width = 0.7\textwidth]{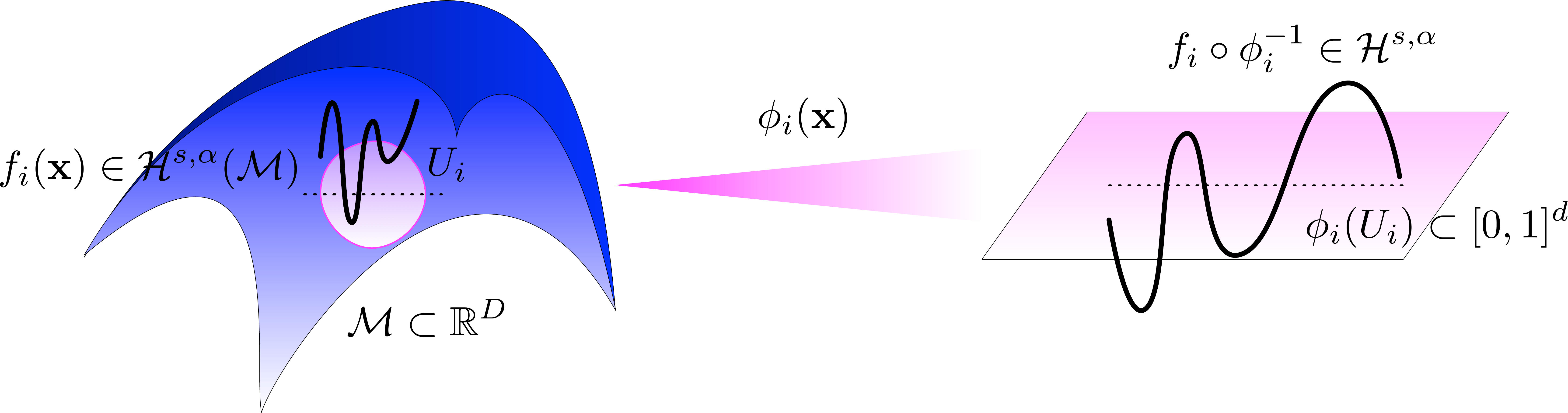}
\caption{Locally approximate $f$ in each chart $(U_i, \phi_i)$ using Taylor polynomials.}
\label{fig:taylor}
\end{figure}
\begin{lemma}\label{lemma:holderfi}
Suppose Assumption \ref{assump3} holds. For $i = 1, \dots, C_\cM$, the function $f_i$ is H\"{o}lder continuous on $\cM$, in the sense that there exists a H\"{o}lder coefficient $L_i$ depending on $d,$ the upper bounds of derivatives of the partition of unity $\rho_i$ and coordinate system $\phi_i$, up to order $s$, such that for any $|\sbb| = s$, we have
\begin{align*}
\left| D^{\sbb} (f_i \circ \phi_i^{-1}) \big|_{\phi_i(\xb_1)} - D^{\sbb} (f_i \circ \phi_i^{-1}) \big|_{\phi_i(\xb_2)}\right| \leq L_i \norm{\phi_i(\xb_1) - \phi_i(\xb_2)}_2^\alpha, \quad \forall \xb_1, \xb_2 \in U_i.
\end{align*}
\end{lemma}

\begin{proof}[Proof Sketch]
We provide a sketch here. More details are deferred to Appendix \ref{pf:holderfi}. 
Without loss of generality, suppose Assumption \ref{assump3} holds with the atlas chosen in {\bf Step 1}. Denote $g_1 = f \circ \phi_i^{-1}$ and $g_2 = \rho_i \circ \phi_i^{-1}$. By the Leibniz rule, we have $$D^{\sbb} (f_i \circ \phi_i^{-1}) = D^{\sbb} (g_1 \times g_2) = \sum_{|\pb| + |\qb| = s} {s \choose{|\pb|}} D^{\pb} g_1 D^{\qb} g_2.$$ Consider each term in the sum: for any $\xb_1, \xb_2 \in U_i$,
\begin{align*}
& \big|D^{\pb} g_1 D^{\qb} g_2 |_{\phi_i(\xb_1)} - D^{\pb} g_1 D^{\qb} g_2 |_{\phi_i(\xb_2)} \big| \\
\leq \hspace{0.03in}& |D^{\pb} g_1(\phi_i(\xb_1))| \big|D^{\qb} g_2 |_{\phi_i(\xb_1)} - D^{\qb} g_2 |_{\phi_i(\xb_2)}\big| + |D^{\qb} g_2(\phi_i(\xb_2))| \big|D^{\pb} g_1 |_{\phi_i(\xb_1)} - D^{\pb}g_1 |_{\phi_i(\xb_2)}\big| \\
\leq \hspace{0.03in}& \lambda_{i} \theta_{i, \alpha} \norm{\phi_i(\xb_1) - \phi_i(\xb_2)}_2^\alpha + \mu_{i} \beta_{i, \alpha} \norm{\phi_i(\xb_1) - \phi_i(\xb_2)}_2^\alpha.
\end{align*}
Here $\lambda_{i}$ and $\mu_i$ are uniform upper bounds on the derivatives of $g_1$ and $g_2$ with order up to $s$, respectively. 
The quantities $\theta_{i, \alpha}$ and $\beta_{i, \alpha}$ in the last inequality above is chosen as follows: by the mean value theorem, 
we have
\begin{align*}
\big|D^{\qb} g_2 |_{\phi_i(\xb_1)} - D^{\qb} g_2 |_{\phi_i(\xb_2)}\big| & \leq \sqrt{d} \mu_i \norm{\phi_i(\xb_1) - \phi_i(\xb_2)}_2\\ 
& = \sqrt{d} \mu_i \norm{\phi_i(\xb_1) - \phi_i(\xb_2)}_2^{1-\alpha} \norm{\phi_i(\xb_1) - \phi_i(\xb_2)}_2^\alpha \\
& \leq \sqrt{d} \mu_i (2r)^{1-\alpha} \norm{\phi_i(\xb_1) - \phi_i(\xb_2)}_2^\alpha,
\end{align*}
where the last inequality is due to the fact that $\norm{\phi_i(\xb_1) - \phi_i(\xb_2)}_2 \leq b_i \norm{V_i} \norm{\xb_1 - \xb_2}_2 \leq 2r$. Then we set $\theta_{i,\alpha} = \sqrt{d} \mu_i (2r)^{1-\alpha}$ and by a similar argument, we set $\beta_{i, \alpha} = \sqrt{d} \lambda_i (2r)^{1-\alpha}$. We complete the proof by taking $L_i = 2^{s+1} \sqrt{d} \lambda_i \mu_i (2r)^{1-\alpha}$.
\end{proof}

Lemma \ref{lemma:holderfi} is crucial for the error estimation in the local approximation of $f_i \circ \phi_i^{-1}$ by Taylor polynomials. This error estimate is given in the following theorem, where some of the proof techniques are from Theorem $1$ in \citet{yarotsky2017error}. 
\begin{theorem}\label{thm:taylor}
Let $f_i = f\rho_i$ as in \textbf{Step 4}. For any $\delta \in (0, 1)$, there exists a ReLU network structure that, if the weight parameters are properly chosen, the network yields an approximation of $f_i \circ \phi_i^{-1}$ uniformly with an $L_\infty$ error $\delta$. Such a network has 
\begin{enumerate}
\item no more than $c_1\left(\log \frac{1}{\delta} + 1\right)$ layers, with width bounded by $c_2 \delta^{-d/(s+\alpha)}$,
\item at most $c_3 \delta^{-\frac{d}{s+\alpha}}\left(\log \frac{1}{\delta} + 1 \right)$ neurons and weight parameters, with the range of weight parameters bounded by $\kappa = c_4\max\{1, \sqrt{d}\}$,
\end{enumerate}
where $c_1, c_2, c_3$ depend on $s, d$,  $\tau$, and the upper bound of derivatives of $f_i \circ \phi_i^{-1}$ up to order $s$, and $c_4$ depends on the upper bound of the derivatives of $\rho_i$'s up to order $s$.
\end{theorem}

\begin{proof}[Proof Sketch]
The detailed proof is provided in Appendix \ref{pf:taylor}. The proof consists of two steps: 
\begin{enumerate}
\item Approximate $f_i \circ \phi_i^{-1}$ using a weighted sum of Taylor polynomials; 
\item Implement the weighted sum of Taylor polynomials using ReLU networks. 
\end{enumerate}
Specifically, we set up a uniform grid and divide $[0, 1]^d$ into small cubes, and then approximate $f_i \circ \phi_i^{-1}$ by its $s$-th order Taylor polynomial in each cube. To implement such polynomials by ReLU networks, we recursively apply the multiplication $\hat{\times}$ operator in Corollary \ref{cor:hatx}, since these polynomials are sums of the products of different variables.
\end{proof}

\textbf{Step 5. Estimating the total error}. We have collected all the ingredients to implement the entire ReLU network to approximate $f$ on $\cM$. Recall that the network structure consists of 3 main sub-networks as demonstrated in Figure \ref{fig:network}. Let $\hat{\times}$ be an approximation to the multiplication operator in the pairing sub-network with error $\eta$. Accordingly, the function given by the whole network is 
\begin{align}
\tilde{f} = \sum_{i=1}^{C_\cM} \hat{\times}(\hat{f}_i, \hat{\mathds{1}}_\Delta \circ \hat{d}_i^2)  \quad\textrm{with}~~ \hat{f}_i = \tilde{f}_i \circ \phi_i, \notag
\end{align}
where $\tilde{f}_i$ is the approximation of $f_i \circ \phi_i^{-1}$ using Taylor polynomials in Theorem \ref{thm:taylor}. The total error can be decomposed into three components according to Lemma \ref{prop:error} below.  We denote $\mathds{1}(\xb \in U_i)$ as the indicator function of $U_i$. Let the approximation errors of the multiplication operation $\hat{\times}$ and the local Taylor polynomial in Theorem \ref{thm:taylor} be $\eta$ and $\delta$, respectively.
\begin{lemma}\label{prop:error}
For any $i = 1, \dots, C_\cM$, we have $\lVert \tilde{f} - f \rVert_\infty \leq \sum_{i=1}^{C_\cM}(A_{i, 1} + A_{i, 2} + A_{i, 3})$, where
\begin{align*}
A_{i, 1} & = \big\lVert\hat{\times}(\hat{f}_i, \hat{\mathds{1}}_\Delta \circ \hat{d}_i^2) - \hat{f}_i \times (\hat{\mathds{1}}_\Delta \circ \hat{d}_i^2)\big\rVert_\infty \leq \eta, \\
A_{i, 2} & = \big\lVert \hat{f}_i \times (\hat{\mathds{1}}_\Delta \circ \hat{d}_i^2) - f_i \times (\hat{\mathds{1}}_\Delta \circ \hat{d}_i^2)\big\rVert_\infty \leq \delta,  
\\
A_{i, 3} & = \big\lVert f_i \times (\hat{\mathds{1}}_\Delta \circ \hat{d}_i^2) - f_i \times \mathds{1}(\xb \in U_i)\big\rVert_\infty \leq \frac{c(\pi + 1)}{r(1 - r / \tau)} \Delta \quad \textrm{for some constant}~ c.
\end{align*}
\end{lemma}
Lemma \ref{prop:error} is proved in Appendix \ref{pf:error}. In order to achieve an $\epsilon$ total approximation error, i.e., $\lVert f - \tilde{f} \rVert_\infty \leq \epsilon$, we need to control the errors in the three sub-networks. In other words, we need to decide $\nu$ for $\hat{d}_i^2$, $\Delta$ for $\hat{\mathds{1}}_\Delta$, $\delta$ for $\tilde{f}_i$, and $\eta$ for $\hat{\times}$. 
Note that $A_{i, 1}$ is the error from the pairing sub-network, $A_{i, 2}$ is the approximation error in the Taylor approximation sub-network, and $A_{i, 3}$ is the error from the chart determination sub-network. The error bounds on $A_{i, 1}, A_{i, 2}$ are straightforward from the constructions of $\hat{\times}$ and $\hat{f}_i$. The estimate of $A_{i, 3}$ involves some technical analysis since $\lVert \hat{\mathds{1}}_\Delta \circ \hat{d}_i^2 - \mathds{1}(\xb \in U_i)\rVert_\infty = 1$. Note that we have $$\hat{\mathds{1}}_\Delta \circ \hat{d}_i^2 (\xb) - \mathds{1}(\xb \in U_i) = 0$$ whenever $\norm{\xb - \cbb_i}_2^2 < r^2 -\Delta$ or $\norm{\xb - \cbb_i}_2^2 > r^2$. Therefore, we only need to prove that $|f_i(\xb)|$ is sufficiently small in the shell region $$\cK_i = \{\xb \in \cM : r^2 - \Delta \leq \norm{\xb - \cbb_i}_2^2 \leq r^2 \}.$$ 

We bound the maximum of $f_i$ on $\cK_i$ using a first-order Taylor expansion. Since $f_i$ vanishes at the boundary of $U_i$ due to the partition of unity $\rho_i$, we can show that $\sup_{\xb \in \cK_i} |f_i(\xb)|$ is proportional to the width $\Delta$ of $\cK_i$. In particular, there exists a constant $c$ depending on $f_i$'s and $\phi_i$'s such that
\begin{align}\label{eq:shellmax}
\max_{\xb \in \cK_i} |f_i(\xb)| \leq \frac{c(\pi + 1)}{r(1 - r / \tau)} \Delta \quad \textrm{for any}\quad i = 1, \dots, C_{\cM}.
\end{align}
Then \eqref{eq:shellmax} immediately implies the upper bound on $A_{i, 3}$. The formal statement of \eqref{eq:shellmax} and its proof are deferred to Lemma \ref{lemma:fibound} and Appendix \ref{pf:fibound}.

Given Lemma \ref{prop:error}, we choose
\begin{equation}
\eta = \delta = \frac{\epsilon}{3 C_\cM} \quad \textrm{and} \quad \Delta = \frac{r(1-r/\tau)\epsilon}{3c(\pi + 1)C_\cM}
\label{eqchoice}
\end{equation}
so that the approximation error is bounded by $\epsilon$. Moreover, we choose 
\begin{equation}\label{eq:nuchoice}
\nu = \frac{\Delta}{16 B^2 D}
\end{equation} to guarantee $\Delta > 8B^2 D \nu$ so that the definition of $\hat{\mathds{1}}_\Delta$ is valid.

Finally we quantify the size of the ReLU network. Recall that the chart determination sub-network has $c_1 \log \frac{1}{\nu}$ layers, the Taylor approximation sub-network has $c_2 \log \frac{1}{\delta}$ layers, and the pairing sub-network has $c_3 \log \frac{1}{\eta}$ layers. Here $c_2$ depends on $d, s, f$, and $c_1, c_3$ are {absolute} constants. Combining these with \eqref{eqchoice} and \eqref{eq:nuchoice} yields the depth in Theorem \ref{thm:bias}. By a similar argument, we can obtain the number of neurons and weight parameters. A detailed analysis is given in Appendix \ref{pf:size}.

\section{Proof of the Statistical Estimation Theory}\label{sec:statproof}
In the proof of Theorem \ref{thm:stat}, we decompose the mean squared error of the estimator $\hat{f}_n$ into a squared bias term and a variance term. We bound the bias and variance separately, where the bias is tackled using the approximation theory (Theorem \ref{thm:bias}), and the variance is bounded using the metric entropy arguments \citep{van1996weak, gyorfi2006distribution}. We begin with an oracle-type decomposition of the $L_2$ risk, in which we introduce the empirical $L_2$ risk as the intermediate term:
\begin{align*}
\EE\left[\int_\cM \left(\hat{f}_n(\xb) - f_0(\xb)\right)^2 d \cD_x(\xb) \right] & = \underbrace{2 \EE \left[\frac{1}{n}\sum_{i=1}^n (\hat{f}_n(\xb_i) - f_0(\xb_i))^2 \right]}_{T_1} \\
& \quad + \underbrace{\EE\left[\int_{\cM} \left(\hat{f}_n(\xb) - f_0(\xb)\right)^2 d \cD_x(\xb)\right] - 2 \EE \left[\frac{1}{n}\sum_{i=1}^n (\hat{f}_n(\xb_i) - f_0(\xb_i))^2 \right]}_{T_2},
\end{align*}
where $T_1$ reflects the squared bias of using neural networks for estimating $f_0$ and $T_2$ is the variance term. We slightly abuse the notation $i$ to denote the index of samples.

\subsection{Bias Characterization --- Bounding $T_1$}
Since $T_1$ is the empirical $L_2$ risk of $\hat{f}_n$ evaluated on the samples $S_n$, we relate $T_1$ to the empirical risk \eqref{eq:hat_f} by rewriting $f_0(\xb_i) = y_i - \xi_i$. Substituting into $T_1$, we derive the following decomposition,
\begin{align}
T_1 & = 2 \EE \left[\frac{1}{n} \sum_{i=1}^n (\hat{f}_n(\xb_i) - y_i + \xi_i)^2 \right] \nonumber \\
& \overset{(i)}{=} 2 \EE \left[\frac{1}{n} \sum_{i=1}^n \left[(\hat{f}_n(\xb_i) - y_i)^2 + 2 \xi_i \hat{f}_n(\xb_i) - \xi_i^2 \right]\right] \nonumber \\
& = 2 \EE \left[\inf_{f \in \cF(R, \kappa, L, p, K)} \frac{1}{n}\sum_{i=1}^n \left[(f(\xb_i) - y_i)^2 - \xi_i^2 + 2 \xi_i \hat{f}_n(\xb_i) \right] \right] \nonumber \\
& \overset{(ii)}{\leq} 2 \underbrace{\inf_{f \in \cF(R, \kappa, L, p, K)} \int_\cM (f(\xb) - f_0(\xb))^2 d \cD_x(\xb)}_{(A)} + 4 \underbrace{\EE \left[\frac{1}{n} \sum_{i=1}^n \xi_i \hat{f}_n(\xb_i) \right]}_{(B)}. \label{eq:t1AB}
\end{align}
Equality $(i)$ is obtained by expanding the square, where the cross term $\EE[\xi_i y_i] = \EE[\xi_i (f_0(\xb_i) + \xi_i)] = \EE[\xi_i^2]$ due to the independence between $\xb_i$ and $\xi_i$. Inequality $(ii)$ invokes the Jensen's inequalty to pass the expectation. To obtain term $(A)$, we expand $(f(\xb_i) - y_i)^2 = (f(\xb_i) - f_0(\xb_i) - \xi_i)^2$, and observe the cancellation of $-\xi_i^2$. Note that term $(A)$ is the squared approximation error of neural networks, and we will tackle it later using Theorem \ref{thm:bias}. We bound term $(B)$ by quantifying the complexity of the network class $\cF(R, \kappa, L, p, K)$. A precise upper bound of $T_1$ is given in the following lemma, whose proof follows a similar argument in \citet[Lemma 4]{schmidt2017nonparametric}.
\begin{lemma}\label{lemma:t1}
Fix the neural network class $\cF(R, \kappa, L, p, K)$. For any constant $\delta \in (0, 2R)$, we have
\begin{align*}
T_1 & \leq 4 \inf_{f \in \cF(R, \kappa, L, p, K)} \int_{\cM} (f(\xb) - f_0(\xb))^2 d\cD_x(\xb) + 48 \sigma^2 \frac{\log \cN(\delta, \cF(R, \kappa, L, p, K), \norm{\cdot}_\infty) + 2}{n} \\
& \hspace{1.5in} + (8\sqrt{6} \sqrt{\frac{\log \cN(\delta, \cF(R, \kappa, L, p, K), \norm{\cdot}_\infty) + 2}{n}} + 8)\sigma \delta,
\end{align*}
where $\cN(\delta, \cF(R, \kappa, L, p, K), \norm{\cdot}_\infty)$ denotes the $\delta$-covering number of $\cF(R, \kappa, L, p, K)$ with respect to the $\ell_\infty$ norm, i.e., there exists a discretization of $\cF(R, \kappa, L, p, K)$ into $\cN(\delta, \cF(R, \kappa, L, p, K), \norm{\cdot}_\infty)$ distinct elements, such that for any $f \in \cF$, there is $\bar{f}$ in the discretization satisfying $\norm{\bar{f} - f}_\infty \leq \epsilon$.
\end{lemma}
\begin{proof}[Proof Sketch]
Given the derivation in \eqref{eq:t1AB}, we need to bound term $(B)$. We discretize the neural network class $\cF(R, \kappa, L, p, K)$ as $\{f^*_i\}_{i=1}^{\cN(\delta, \cF(R, \kappa, L, p, K), \norm{\cdot}_\infty)}$. By the definition of covering, there exists $f^*$ such that $\lVert \hat{f}_n - f^* \rVert_\infty \leq \delta$. Denoting $\norm{f - f_0}_n = \frac{1}{n} \sum_{i=1}^n (f(\xb_i) - f_0(\xb_i))^2$, we cast $(B)$ into
\begin{align*}
(B) & = \EE \left[\frac{1}{n} \sum_{i=1}^n \xi_i (\hat{f}_n(\xb_i) - f^*(\xb_i) + f^*(\xb_i) - f_0(\xb_i)) \right] \\
& \overset{(i)}{\leq} \EE \left[\frac{1}{n} \sum_{i=1}^n \xi_i (f^*(\xb_i) - f_0(\xb_i)) \right] + \delta \sigma \\
& = \EE \left[\frac{\norm{f^* - f_0}_n}{\sqrt{n}} \frac{\sum_{i=1}^n \xi_i (f^*(\xb_i) - f_0(\xb_i))}{\sqrt{n} \norm{f^* - f_0}_n} \right] + \delta \sigma \\
& \overset{(ii)}{\leq} \sqrt{2}\EE \left[\frac{\|\hat{f}_n - f_0\|_n + \delta}{\sqrt{n}} \left|\frac{\sum_{i=1}^n \xi_i (f^*(\xb_i) - f_0(\xb_i))}{\sqrt{n} \norm{f^* - f_0}_n}\right|\right] + \delta \sigma,
\end{align*}
where $(i)$ follows from H\"{o}lder's inequality and $(ii)$ is obtained by some algebraic manipulation. To break the dependence between $f^*$ and the samples, we replace $f^*$ by any $f^*_j$ in the $\delta$-covering and observe that $\left|\frac{\sum_{i=1}^n \xi_i (f^*(\xb_i) - f_0(\xb_i))}{\sqrt{n} \norm{f^* - f_0}_n}\right| \leq \max_j \left|\frac{\sum_{i=1}^n \xi_i (f^*_j(\xb_i) - f_0(\xb_i))}{\sqrt{n} \lVert f^*_j - f_0\rVert_n}\right|$. Applying the Cauchy-Schwarz inequality, we can show
\begin{align*}
(B) \leq \sqrt{2}\left(\sqrt{\frac{1}{n}\EE\left[\|\hat{f}_n - f_0\|_n^2\right]} + \frac{\delta}{\sqrt{n}}\right) \sqrt{\EE\left[\max_j~ z_j^2 \right]} + \delta \sigma,
\end{align*}
where $z_j = \left|\frac{\sum_{i=1}^n \xi_i (f^*(\xb_i) - f_0(\xb_i))}{\sqrt{n} \norm{f^* - f_0}_n}\right|$. Given $\xb_1, \dots, \xb_n$, we note that $z_j$ is a sub-Gaussian random variable with parameter $\sigma$ (i.e., its variance is bounded by $\sigma^2$). It is well established in the existing literature on empirical processes \citep{van1996weak} that the maximum of a collection of squared sub-Gaussian random variables satisfies
\begin{align*}
\EE \left[\max_{j} ~ z_j^2 ~|~ \xb_1, \dots, \xb_n \right] \leq 3\sigma^2\log \cN(\delta, \cF(R, \kappa, L, p, K), \norm{\cdot}_\infty) + 6\sigma^2.
\end{align*}
Substituting the above inequality into $(B)$ and combining $(A)$ and $(B)$, we have
\begin{align*}
T_1 = 2 \EE\left[\lVert \hat{f}_n - f_0 \rVert_n^2\right] & \leq 2 \inf_{f \in \cF(R, \kappa, L, p, K)} \EE\left[(f(\xb) - f_0(\xb))^2 \right] + 4\delta \sigma \nonumber \\
& \quad + 4\sqrt{6} \sigma \left(\sqrt{\EE \left[\|\hat{f}_n - f_0\|_n^2 \right]} + \delta \right) \sqrt{\frac{\log \cN(\delta, \cF(R, \kappa, L, p, K), \norm{\cdot}_\infty) + 2}{n}}.
\end{align*}
Some manipulation gives rise to the desired result
\begin{align*}
T_1 & \leq 4 \inf_{f \in \cF(R, \kappa, L, p, K)} \int_{\cM} (f(\xb) - f_0(\xb))^2 d\cD_x(\xb) + 48 \sigma^2 \frac{\log \cN(\delta, \cF(R, \kappa, L, p, K), \norm{\cdot}_\infty) + 2}{n} \\
& \hspace{1.5in} + (8\sqrt{6} \sqrt{\frac{\log \cN(\delta, \cF(R, \kappa, L, p, K), \norm{\cdot}_\infty) + 2}{n}} + 8)\sigma \delta.
\end{align*}
See proof details in Appendix \ref{pf:t1}.
\end{proof}

\subsection{Variance Characterization --- Bounding $T_2$}
We observe that $T_2$ is the difference between the population $L_2$ risk of $\hat{f}_n$ and its empirical counterpart. However, bounding such a difference is distinct from conventional concentration results due to the scaling factor $2$ before the empirical risk. In particular, we split the empirical risk evenly into two parts, and bound one part using its higher-order moment (fourth moment). Using Bernstein-type inequality allows us to establish a $1/n$ convergence rate of $T_2$; the corresponding upper bound is presented in the following lemma.
\begin{lemma}\label{lemma:t2}
For any constant $\delta \in (0, 2R)$, $T_2$ satisfies
\begin{align*}
T_2 \leq \frac{104R^2}{3n} \log \cN(\delta/4R, \cF(R, \kappa, L, p, K), \norm{\cdot}_\infty) + \left(4 +\frac{1}{2R}\right)\delta.
\end{align*}
\end{lemma}
\begin{proof}[Proof Sketch]
The detailed proof is deferred to Appendix \ref{pf:t2}. For notational simplicity, we denote $\hat{g}(\xb) = (\hat{f}_n(\xb) - f_0(\xb))^2$ and $\norm{\hat{g}}_\infty \leq 4R^2$. Applying the inequality $\int_\cM \hat{g}^2 d \cD_x(\xb) \leq 4R^2 \int_\cM \hat{g} d \cD_x(\xb)$ \citep{barron1991complexity}, we rewrite $T_2$ as
\begin{align*}
T_2 & = \EE\left[\int_{\cM} \hat{g}(\xb) d \cD_x(\xb) - \frac{2}{n}\sum_{i=1}^n \hat{g}(\xb_i) \right]\\
& = 2 \EE\left[\int_{\cM} \hat{g}(\xb) d \cD_x(\xb) - \frac{1}{n}\sum_{i=1}^n \hat{g}(\xb_i) - \frac{1}{2} \int_{\cM} \hat{g}(\xb) d \cD_x(\xb) \right] \\
& \leq 2 \EE\left[\int_{\cM} \hat{g}(\xb) d \cD_x(\xb) - \frac{1}{n}\sum_{i=1}^n \hat{g}(\xb_i) - \frac{1}{8R^2} \int_{\cM} \hat{g}^2(\xb) d \cD_x(\xb)\right].
\end{align*}
We now utilize ghost samples of $\xb$ to bound $T_2$, which is a common technique in existing literature on nonparametric statistics \citep{van1996weak, gyorfi2006distribution}. Specifically, let $\bar{\xb}_i$'s be independent replications of $\xb_i$'s. We bound $T_2$ as
\begin{align*}
T_2 & \leq 2 \EE\left[\sup_{g \in \cG} \int_{\cM} g(\xb) d \cD_x(\xb) - \frac{1}{n}\sum_{i=1}^n g(\xb_i) - \frac{1}{8R^2} \int_{\cM} g^2(\xb) d \cD_x(\xb)\right] \\
& \leq 2 \EE_{\xb, \bar{\xb}} \left[\sup_{g \in \cG} \frac{1}{n} \sum_{i=1}^n (g(\bar{\xb}_i) - g(\xb_i)) - \frac{1}{16R^2} \EE_{\xb, \bar{\xb}} \left[g^2(\xb) + g^2(\bar{\xb})\right] \right],
\end{align*}
where $\cG = \{g = (f - f_0)^2 ~|~ f \in \cF(R, \kappa, L, p, K)\}$. We use the shorthand $\EE_{\xb, \bar{\xb}}[\cdot]$ to denote the double integral $\int_{\cM} \int_{\cM} \cdot d\cD_x(\xb) d\cD_{x}(\bar{\xb})$ with respect to the joint distribution of $(\xb, \bar{\xb})$. The last inequality holds due to Jensen's inequality. Note here $g^2(\xb) + g^2(\bar{\xb})$ contributes as the variance term of $g(\bar{\xb}_i) - g(\xb_i)$, which yields a fast convergence of $T_2$ as $n$ grows.

Similar to bounding $T_1$, we discretize the function space $\cG$ using a $\delta$-covering denoted by $\cG^*$. This allows us to replace the supremum by the maximum over a finite set:
\begin{equation*}
T_2 \leq 2 \EE_{\bar{\xb}, \xb} \left[\sup_{g^* \in \cG^*} \frac{1}{n} \sum_{i=1}^n (g^*(\bar{\xb}_i) - g^*(\xb_i)) - \frac{1}{16R^2} \EE_{\xb, \bar{\xb}} \left[(g^*)^2(\xb) + (g^*)^2(\bar{\xb})\right] \right] + \left(4 + \frac{1}{2R}\right)\delta.
\end{equation*}

We can bound the above maximum by the Bernstein's inequality, which yields
\begin{equation*}
T_2 \leq \frac{104R^2}{3n} \log \cN(\delta, \cG, \norm{\cdot}_\infty) + \left(4 + \frac{1}{2R} \right)\delta.
\end{equation*}

The last step is to relate the covering number of $\cG$ to that of $\cF(R, \kappa, L, p, K)$. Specifically, consider any $g_1, g_2 \in \cG$ with $g_1 = (f_1 - f_0)^2$ and $g_2 = (f_2 - f_0)^2$, respectively. We can derive
\begin{equation*}
\norm{g_1 - g_2}_\infty = \sup_{\xb \in \cM} \left|f_1(\xb) - f_2(\xb) \right| \left| f_1(\xb) + f_2(\xb) - 2f_0(\xb) \right| \leq 4R \norm{f_1 - f_2}_\infty.
\end{equation*}
Therefore, the inequality $\cN(\delta, \cG, \norm{\cdot}_\infty) \leq \cN(\delta/4R, \cF(R, \kappa, L, p, K), \norm{\cdot}_\infty)$ holds, which implies
\begin{equation*}
T_2 \leq \frac{104R^2}{3n} \log \cN(\delta/4R, \cF(R, \kappa, L, p, K), \norm{\cdot}_\infty) + \left(4 + \frac{1}{2R} \right)\delta.
\end{equation*}
The proof is complete.
\end{proof}

\subsection{Covering Number of Neural Networks}\label{sec:coveringnumber}
The upper bounds of $T_1$ and $T_2$ in Lemmas \ref{lemma:t1} and \ref{lemma:t2} both depend on the covering number of the network class $\cF(R, \kappa, L, p, K)$. In this section, we provide an upper bound on the covering number $\cN(\delta, \cF(R, \kappa, L, p, K), \norm{\cdot}_\infty)$ for a given a resolution $\delta > 0$. Since each weight parameter in the network is bounded by a constant $\kappa$, we construct a covering by partitioning the range of each weight parameter into a uniform grid. By choosing a proper grid size, we show the following lemma.
\begin{lemma}\label{lemma:coveringbound}
Given $\delta > 0$, the $\delta$-covering number of the neural network class $\cF(R, \kappa, L, p, K)$ satisfies
\begin{align}\label{eq:coveringbound}
\cN(\delta, \cF(R, \kappa, L, p, K), \norm{\cdot}_\infty) \leq \left(\frac{2L^2 (pB + 2) \kappa^L p^{L+1}}{\delta} \right)^K.
\end{align}
\end{lemma}
\begin{proof}[Proof Sketch]
Consider $f, f' \in \cF(R, \kappa, L, p, K)$ with each weight parameter differing at most $h$. By an induction on the number of layers in the network, we show that the $\ell_\infty$ norm of the difference $f - f'$ scales as
\begin{equation*}
\norm{f - f'}_\infty \leq h L (p B + 2) (\kappa p)^{L-1}.
\end{equation*}
As a result, to achieve a $\delta$-covering, it suffices to choose $h$ such that $h L (p B + 2) (\kappa p)^{L-1} = \delta$. Moreover, there are ${Lp^2 \choose{K}} \leq (Lp^2)^K$ different choices of $K$ non-zero entries out of $Lp^2$ weight parameters. Therefore, the covering number is bounded by
\begin{equation*}
\cN(\delta, \cF(R, \kappa, L, p, K), \norm{\cdot}_\infty) \leq \left(Lp^2 \right)^K \left(\frac{2\kappa}{h}\right)^{K} \leq \left(\frac{2L^2 (pB + 2) \kappa^L p^{L+1}}{\delta} \right)^K.
\end{equation*}
The detailed proof is provided in Appendix \ref{pf:coveringbound}.
\end{proof}

\subsection{Bias-Variance Trade-off}
We are ready to finish the proof of Theorem \ref{thm:stat}. Combining the upper bounds of $T_1$ in Lemma \ref{lemma:t1} and $T_2$ in Lemma \ref{lemma:t2} together and substituting the covering number \eqref{eq:coveringbound}, we obtain
\begin{align*}
\EE\left[\int_\cM \left(\hat{f}_n(\xb) - f_0(\xb)\right)^2 d \cD_x(\xb) \right] & \leq 4 \inf_{f \in \cF(R, \kappa, L, p, K)} \int_{\cM} (f(\xb) - f_0(\xb))^2 d\cD_x(\xb) \\
& \quad + 48\sigma^2\frac{\log \cN(\delta, \cF(R, \kappa, L, p, K), \norm{\cdot}_\infty)+2}{n} \\
& \quad + 8\sqrt{6} \sqrt{\frac{\log \cN(\delta, \cF(R, \kappa, L, p, K), \norm{\cdot}_\infty)+2}{n}} \sigma \delta \\
& \quad + \frac{104R^2}{3n} \log \cN(\delta/4R, \cF(R, \kappa, L, p, K), \norm{\cdot}_\infty) \\
& \quad + \left(4 +\frac{1}{2R} + 8\sigma\right)\delta.
\end{align*}
It suffices to choose $\delta = 1/n$, which gives rise to
\begin{align}\label{eq:staterrorcombined}
\EE\left[\int_\cM \left(\hat{f}_n(\xb) - f_0(\xb)\right)^2 d \cD_x(\xb) \right] & \leq 4 \inf_{f \in \cF(R, \kappa, L, p, K)} \int_{\cM} (f(\xb) - f_0(\xb))^2 d\cD_x(\xb) \nonumber \\
& \hspace{1.2in} + \tilde{O}\left(\frac{R^2 + \sigma^2}{n} K L \log (R \kappa L p n) + \frac{\sigma^2}{n} \right),
\end{align}
where we also plug in the covering number upper bound in Lemma \ref{eq:coveringbound}. We further set the approximation error as $\epsilon$, i.e., $\inf_{f \in \cF(R, \kappa, L, p, K)} \lVert f(\xb) - f_0(\xb)\rVert_\infty \leq \epsilon$. Theorem \ref{thm:bias} suggests that we choose $L = \tilde{O}(\log \frac{1}{\epsilon})$, $p = \tilde{O}(\epsilon^{-\frac{d}{s+\alpha}})$, and $K = \tilde{O}\left(\epsilon^{-\frac{d}{s + \alpha}} \log \frac{1}{\epsilon} + D \log \frac{1}{\epsilon}\right)$. Substituting $L$, $p$, and $K$ into \eqref{eq:staterrorcombined}, we have
\begin{align*}
\EE\left[\int_\cM \left(\hat{f}_n(\xb) - f_0(\xb)\right)^2 d \cD_x(\xb) \right] = \tilde{O}\left(\epsilon^2 + \frac{R^2 + \sigma^2}{n} \left(\epsilon^{-\frac{d}{s+\alpha}} + D\right) \log^3 \frac{1}{\epsilon} + \frac{1}{n} \right).
\end{align*}
To balance the error terms, we pick $\epsilon$ satisfying $\epsilon^2 = \frac{1}{n} \epsilon^{-\frac{d}{s+\alpha}}$, which gives $\epsilon = n^{-\frac{s+\alpha}{d + 2(s + \alpha)}}$. The proof of Theorem \ref{thm:stat} is complete by plugging in $\epsilon = n^{-\frac{s+\alpha}{d + 2(s + \alpha)}}$ and rearranging the terms.

\section{Conclusion}
\label{sec:discuss}
We study nonparametric regression of functions supported on a $d$-dimensional Riemannian manifold $\cM$ isometrically embedded in $\RR^D$, using deep ReLU neural networks. Our result establishes an efficient statistical estimation theory for general regression functions including $C^s$ and H\"{o}lder functions supported on manifolds. We show that the $L_2$ loss for the estimation of $f_0 \in \cH^{s, \alpha}(\cM)$ converges in the order of $n^{-\frac{s + \alpha}{2(s +\alpha)+d}}$. To obtain an $\epsilon$-error for the estimation of $f_0$, the sample complexity scales in the order of $\epsilon^{-\frac{2(s + \alpha) + d}{s + \alpha}}$. This {\color{blue} sample complexity} depends on the intrinsic dimension $d$, and demonstrates that deep neural networks are adaptive to low-dimensional geometric structures of data sets. \textcolor{blue}{Such results can be viewed as theoretical justifications for the empirical success of deep learning in various real-world applications where the data sets exhibit low-dimensional structures.}

\section*{Acknowledgment}
This work was supported by NSF DMS $1818751$, NSF DMS 2012652, and NSF IIS-1717916.

\bibliographystyle{ims}
\bibliography{ref}


\newpage
\appendix

\section{Proofs of the Preliminary Results in Section \ref{sec:approxproof}}
\numberwithin{equation}{section}
\subsection{Proof of Lemma \ref{lemma:quad}}\label{pf:quad}
\begin{proof}
We partition the interval $[0, 1]$ uniformly into $2^N$ subintervals $I_k = [\frac{k}{2^N}, \frac{k+1}{2^N}]$ for $k = 0, \dots, 2^N-1$. We approximate $f(x) = x^2$ on these subintervals by a linear interpolation
\begin{equation*}
\hat{f}_k = \frac{2k+1}{2^N}\left(x - \frac{k}{2^N}\right) + \frac{k^2}{2^{2N}}, \quad \textrm{for} ~ x \in I_k.
\end{equation*}
It is straightforward to check that $\hat{f}_k$ meets $f$ at the endpoints $\frac{k}{2^N}, \frac{k+1}{2^N}$ of $I_k$.

We evaluate the approximation error of $\hat{f}_k$ on the interval $I_k$:
\begin{align*}
\max_{x \in I_k} \left|f(x) - \hat{f}_k(x)\right| & = \max_{x \in I_k} \left|x^2 - \frac{2k+1}{2^N}x + \frac{k^2 + k}{2^{2N}} \right| \\
& = \max_{x \in I_k} \left|\rbr{x - \frac{2k+1}{2^{N+1}}}^2 - \frac{1}{2^{2N+2}} \right| \\
& = \frac{1}{2^{2N+2}}.
\end{align*}
Note that this approximation error does not depend on $k$. Thus, in order to achieve an $\epsilon$ approximation error, we only need
\begin{equation*}
\frac{1}{2^{2N+2}} \leq \epsilon ~\Longrightarrow~ N \geq \frac{\log \frac{1}{\epsilon}}{2\log 2} - 1.
\end{equation*}
Since $2 \log 2 > 1$, we let $N = \left\lceil \log \frac{1}{\epsilon} \right\rceil$ and denote $f_N = \sum_{k=0}^{2^N - 1} \hat{f}_k \mathds{1}\{x \in I_k\}$. We compute the increment from $f_{N-1}$ to $f_N$ for $x \in \left[\frac{k}{2^{N-1}}, \frac{k+1}{2^{N-1}}\right]$ as
\begin{align*}
f_{N-1} - f_N & =
\begin{cases}
\frac{k^2}{2^{2(N-1)}} + \frac{2k+1}{2^{N-1}}\rbr{x - \frac{k}{2^{N-1}}} - \frac{k^2}{2^{2(N-1)}} - \frac{4k+1}{2^N}\rbr{x - \frac{k}{2^{N-1}}}, & x \in \left[\frac{k}{2^{N-1}}, \frac{2k+1}{2^{N}}\right) \\
\frac{k^2}{2^{2(N-1)}} + \frac{2k+1}{2^{N-1}}\rbr{x - \frac{k}{2^{N-1}}} - \frac{(2k+1)^2}{2^{2N}} - \frac{4k+3}{2^N}\rbr{x - \frac{2k+1}{2^N}}, & x \in \left[\frac{2k+1}{2^{N}}, \frac{k+1}{2^{N-1}}\right)
\end{cases} \\
& =
\begin{cases}
\frac{1}{2^N} x - \frac{k}{2^{2N - 1}}, & x \in \left[\frac{k}{2^{N-1}}, \frac{2k+1}{2^{N}}\right) \\
-\frac{1}{2^N} x + \frac{k+1}{2^{2N-1}}, & x \in \left[\frac{2k+1}{2^{N}}, \frac{k+1}{2^{N-1}}\right)
\end{cases}.
\end{align*}
We observe that $f_{N-1} - f_N$ is a triangular function on $\left[\frac{k}{2^{N-1}}, \frac{k+1}{2^{N-1}}\right]$. The maximum is $\frac{1}{2^{2N}}$ independent of $k$ attained at $x = \frac{2k+1}{2^N}$. The minimum is $0$ attained at the endpoints $\frac{k}{2^{N-1}}, \frac{k+1}{2^{N-1}}$. To implement $f_N$, we consider a triangular function representable by a one-layer ReLU network:
\begin{equation*}
g(x) = 2\sigma(x) - 4 \sigma(x - 0.5) + 2\sigma(x - 1).
\end{equation*}
Denote by $g_m = g \circ g \circ \cdots \circ g$ the composition of totally $m$ functions $g$. Observe that $g_m$ is a sawtooth function with $2^{m-1}$ peaks at $\frac{2k+1}{2^m}$ for $k = 0, \dots, 2^{m-1}-1$, and we have $g_m\left(\frac{2k+1}{2^m}\right) = 1$ for $k = 0, \dots, 2^{m-1}-1$. Then we have $f_{N-1} - f_N = \frac{1}{2^{2N}} g_N$. By induction, we have
\begin{align*}
f_N & = f_{N-1} - \frac{1}{2^{2N}} g_N \\
& = f_{N-2} - \frac{1}{2^{2N}} g_N - \frac{1}{2^{2N-2}} g_{N-1} \\
& = \cdots \\
& = x - \sum_{k=1}^N \frac{1}{2^{2k}} g_k.
\end{align*}
Therefore, $f_N$ can be implemented by a ReLU network of depth $\left\lceil \log \frac{1}{\epsilon} \right\rceil \leq \log\frac{1}{\epsilon} + 1$. Meanwhile, each layer consists of at most 3 neurons. Hence, the total number of neurons and weight parameters is no more than $c \log \frac{1}{\epsilon}$ for an absolute constant $c$.
\end{proof}

\subsection{Proof of Corollary \ref{cor:hatx}}\label{pf:hatx}
\begin{proof}
Let $\hat{f}_\delta$ be an approximation of the quadratic function on $[0, 1]$ with error $\delta \in (0, 1)$. We set
\begin{equation*}
\hat{\times} (x, y) = C^2 \rbr{\hat{f}_\delta \left(\frac{|x + y|}{2C}\right) - \hat{f}_{\delta}\left(\frac{|x - y|}{2C}\right)}.
\end{equation*}
Now we determine $\delta$. We bound the error of $\hat{\times}$
\begin{align*}
\left|\hat{\times} (x, y) - xy \right| & = C^2 \left|\hat{f}_\delta \left(\frac{|x + y|}{2C}\right) - \frac{|x + y|^2}{4C^2} - \hat{f}_{\delta}\left(\frac{|x - y|}{2C}\right) + \frac{|x - y|^2}{4C^2} \right| \\
& \leq C^2 \left|\hat{f}_\delta \left(\frac{|x + y|}{2C}\right) - \frac{|x + y|^2}{4C^2} \right| + \left| \hat{f}_{\delta}\left(\frac{|x - y|}{2C}\right) - \frac{|x - y|^2}{4C^2} \right| \\
& \leq 2C^2 \delta. 
\end{align*}
Thus, we pick $\delta = \frac{\epsilon}{2C^2}$ to ensure $\left|\hat{\times} (x, y) - xy \right| \leq \epsilon$ for any inputs $x$ and $y$. As shown in Lemma \ref{lemma:quad}, we can implement $\hat{f}_\delta$ using a ReLU network of depth at most $c' \log \frac{1}{\delta} = c\log \frac{C^2}{\epsilon}$ with absolute constants $c', c$. The proof is complete.
\end{proof}

\section{Proof of Approximation Theory of ReLU Network (Theorem \ref{thm:bias})}
\numberwithin{equation}{section}
This section consists of the detailed proofs of Lemma \ref{lemma:diffeomorphism}, Lemma \ref{lemma:holderfi}, local approximation theory Theorem \ref{thm:taylor}, error decomposition in Lemma \ref{prop:error} and a technical Lemma \ref{lemma:fibound} for bounding the error, as well as the configuration of the desired ReLU network class for universally approximating H\"{o}lder functions.

\subsection{Proof of Lemma \ref{lemma:diffeomorphism}}\label{pf:diffeomorphism}
\begin{proof}
We first show ${\sf P}_i$ defined on $U_i$ is a homeomorphism, which implies $(U_i, {\sf P}_i)$ is a chart on the manifold. Then by Proposition 6.10 in \citet{tu2010introduction}, we conclude ${\sf P}_i$ is a diffeomorphism. 

To show ${\sf P}_i$ is a homeomorphism on $U_i$, we only need to show ${\sf P}_i$ has a continuous inverse. By Lemma 5.4 in \citet{niyogi2008finding}, the derivative of ${\sf P}_i$ is nonsingular in $U_i$. The inverse function theorem implies that ${\sf P}_i$ is locally invertible in an open neighborhood $\cB(\cbb_i, c\tau) \bigcap \cM$ for some constant $c > 0$. In the following, we show by contradiction that the constant $c \geq 1/4$. Suppose not, there exist distinct points $\ab, \bbb \in U_i$ such that ${\sf P}_i(\ab) = {\sf P}_i(\bbb)$ with $\norm{\ab - \cbb_i}_2 < \tau / 4$ and $\norm{\bbb - \cbb_i}_2 < \tau / 4$. Using the triangle inequality, we obtain $\norm{\ab - \bbb}_2 < \tau/2$. Applying Proposition 6.3 in \citet{niyogi2008finding}, we derive
\begin{align*}
d_{\cM}(\ab, \bbb) < \tau \quad \textrm{and} \quad d_{\cM}(\ab, \cbb_i) < \tau (1 - \sqrt{2}/2) \quad \textrm{with}\quad d_{\cM}(\cdot, \cdot) \quad \textrm{being the geodesic distance}.
\end{align*}
Furthermore, using Proposition 6.2 in \citet{niyogi2008finding}, we lower bound the angle between the tangent spaces $T_{\cbb_i}(\cM)$ and $T_{\ab}(\cM)$ by
\begin{align}\label{eq:cos_tangent_1}
\cos \left(\angle (T_{\ab}(\cM), T_{\cbb_i}(\cM))\right) \overset{\triangle}{=} \min_{\ub \in T_{\ab}(\cM)} ~\max_{\vb \in T_{\cbb_i}(\cM)} |\inner{\ub}{\vb}| \geq 1 - \frac{1}{\tau} d_{\cM}(\ab, \cbb_i) > \sqrt{2}/2.
\end{align}
On the other hand, we consider a unit speed geodesic $\gamma(t)$ starting from $\ab$ and ending at $\bbb$, with $\gamma(0) = \ab$, $\gamma(d_{\cM}(\ab, \bbb)) = \bbb$, and $\norm{\dot{\gamma}}_2 = 1$. Integration by parts yields
\begin{align*}
\bbb - \ab & = \gamma(d_{\cM}(\ab, \bbb)) - \gamma(0) \\
& = \int_0^{d_{\cM}(\ab, \bbb)} \dot{\gamma}(t) dt \\
& = \dot{\gamma}(0) d_{\cM}(\ab, \bbb) + \int_{0}^{d_{\cM}(\ab, \bbb)} \int_0^t \ddot{\gamma}(s) ds dt.
\end{align*}
Rearranging terms gives rise to
\begin{align}\label{eq:abdistance}
\norm{\bbb - \ab - \dot{\gamma}(0) d_{\cM}(\ab, \bbb)}_2 \leq \int_{0}^{d_{\cM}(\ab, \bbb)} \int_0^t \norm{\ddot{\gamma}(s)}_2 ds dt \leq \frac{d_{\cM}^2(\ab, \bbb)}{2\tau},
\end{align}
where the last inequality follows from Proposition 6.1 in \citet{niyogi2008finding}. Dividing \eqref{eq:abdistance} by $d_{\cM}(\ab, \bbb)$ and plugging in $d_{\cM}(\ab, \bbb) \leq \tau$, we have
\begin{align*}
\norm{\frac{\bbb - \ab}{d_{\cM}(\ab, \bbb)} - \dot{\gamma}(0)}_2 < \frac{1}{2}.
\end{align*}
For any unit vector $\vb \in T_{\cbb_i}(\cM)$, we evaluate the inner product
\begin{align}\label{eq:cos_tangent_2}
\left|\inner{\dot{\gamma}(0)}{\vb}\right| & \leq \left| \inner{\dot{\gamma}(0) - \frac{\bbb - \ab}{d_{\cM}(\ab, \bbb)}}{\vb} \right| + \left| \inner{\frac{\bbb - \ab}{d_{\cM}(\ab, \bbb)}}{\vb}\right| \nonumber \\
& \overset{(i)}{=} \left| \inner{\dot{\gamma}(0) - \frac{\bbb - \ab}{d_{\cM}(\ab, \bbb)}}{\vb} \right| \nonumber \\
& \leq \norm{\frac{\bbb - \ab}{d_{\cM}(\ab, \bbb)} - \dot{\gamma}(0)}_2 \nonumber \\
& < \frac{1}{2},
\end{align}
where $\left| \inner{\frac{\bbb - \ab}{d_{\cM}(\ab, \bbb)}}{\vb}\right| = 0$ in equality $(i)$, since ${\sf P}_i(\ab) = {\sf P}_i(\bbb)$ by our assertion. Combining \eqref{eq:cos_tangent_1} and \eqref{eq:cos_tangent_2}, we obtain
\begin{align*}
\frac{\sqrt{2}}{2} < \cos \left(\angle (T_{\ab}(\cM), T_{\cbb_i}(\cM))\right) \leq \max_{\vb \in T_{\cbb_i}(\cM)} \left|\inner{\dot{\gamma}(0)}{\vb}\right| < \frac{1}{2},
\end{align*}
which is a contradiction. Therefore, we conclude that ${\sf P}_i$ is injective, and hence invertible on the local neighborhood $\cB(\cbb_i, \tau / 4) \bigcap \cM$. The continuity of ${\sf P}_i$ follows from its definition, and the inverse map of a continuous map is also continuous. Therefore, ${\sf P}_i$ is a homeomorphism on $\cB(\cbb_i, r) \bigcap \cM$ for $r \leq \tau / 4$.

The last step is to show ${\sf P}_i$ is also a diffeomorphism. We leverage the following proposition.
\begin{proposition}[Proposition 6.10 in \citet{tu2010introduction}]\label{prop:diffeomorphism}
If $(U, \phi)$ is a chart on a manifold $\cM$, then the coordinate map $\phi: U \mapsto \phi(U)$ is a diffeomorphism.
\end{proposition}
Since ${\sf P}_i$ is a homeomorphism, we deduce that $(U_i, {\sf P}_i)$ is a chart of $\cM$. Applying Proposition \ref{prop:diffeomorphism}, we conclude that ${\sf P}_i$ is a diffeomorphism.
\end{proof}

\subsection{Proof of Lemma \ref{lemma:holderfi}}\label{pf:holderfi}
\begin{proof}
Recall that we choose local coordinate neighborhood $U_i$ in {\bf Step 1} in Section \ref{sec:approxproof}. Let ${\sf P}_i$ be the projection onto the tangent space $T_{\cbb_i}(\cM)$. Then $\{(U_i, {\sf P}_i)\}$ is an atlas of $\cM$. Without loss of generality, we assume that $\{(U_i, {\sf P}_i)\}$ verifies the H\"{o}lder condition in Definition \ref{def:holder}. Now we rewrite $f_i \circ \phi_i^{-1}$ as
\begin{equation}
\underbrace{(f \circ \phi_i^{-1})}_{g_1} \times \underbrace{(\rho_i \circ \phi_i^{-1})}_{g_2}. \label{eq:twoparts}
\end{equation}
By the definition of the partition of unity, we know $g_2$ is $C^\infty$. This implies that $g_2$ is $(s+1)$ continuously differentiable. Since $\textrm{supp}(\rho_i)$ is compact, the $k$-th derivative of $g_2$ is uniformly bounded by $\lambda_{i, k}$ for any $k \leq s+1$. Let $\lambda_i = \max_{k \leq n + 1} \lambda_{i, k}$. We have for any $|\nbb| \leq n$ and $\xb_1, \xb_2 \in U_i$,
\begin{align*}
\left|D^{\nbb} g_2(\phi_i(\xb_1)) - D^\nbb g_2(\phi_i(\xb_2)) \right| & \leq \sqrt{d} \lambda_i \norm{\phi_i(\xb_1) - \phi_i(\xb_2)}_2 \\
& \leq \sqrt{d} \lambda_i b_i^{1-\alpha} \norm{\xb_1 - \xb_2}_2^{1-\alpha} \norm{\phi_i(\xb_1) - \phi_i(\xb_2)}_2^{\alpha}.
\end{align*}
The last inequality follows from $\phi_i(\xb) = b_i (V_i^\top (\xb - \cbb_i) + \ub_i)$ and $\norm{V_i}_2 = 1$. Observe that $U_i$ is bounded, hence, we have $\norm{\xb_1 - \xb_2}_2^{1-\alpha} \leq (2r)^{1-\alpha}$. Absorbing $\norm{\xb_1 - \xb_2}_2^{1-\alpha}$ into $\sqrt{d} \lambda_i b_i^{1-\alpha}$, we have the derivative of $g_2$ is H\"{o}lder continuous. We denote $\beta_{i, \alpha} = \sqrt{d} \lambda_i b_i^{1-\alpha} (2r)^{1-\alpha} \leq \sqrt{d} \lambda_i (2r)^{1-\alpha}$. Similarly, $g_1$ is $C^{s-1}$ by Assumption \ref{assump3}. Then there exists a constant $\mu_i$ such that the $k$-th derivative of $g_1$ is uniformly bounded by $\mu_i$ for any $k \leq n-1$. These derivatives are also H\"{o}lder continuous with coefficient $\theta_{i, \alpha} \leq \sqrt{d} \mu_i (2r)^{1-\alpha}$. 

By the Leibniz rule, for any $|\sbb| = s$, we expand the $s$-th derivative of $f_i \circ \phi_i^{-1}$ as
\begin{equation*}
D^{\sbb} (g_1 \times g_2) = \sum_{|\pb| + |\qb| = s} {s \choose{|\pb|}} D^{\pb} g_1 D^{\qb} g_2.
\end{equation*}
Consider each summand in the above right-hand side. For any $\xb_1, \xb_2 \in U_i$, we derive
\begin{align*}
& \big|D^{\pb} g_1(\phi_i(\xb_1)) D^{\qb} g_2(\phi_i(\xb_1)) - D^{\pb} g_1(\phi_i(\xb_2)) D^{\qb} g_2(\phi_i(\xb_2)) \big| \\
= & \big|D^{\pb} g_1(\phi_i(\xb_1)) D^{\qb} g_2(\phi_i(\xb_1)) - D^{\pb} g_1(\phi_i(\xb_1)) D^{\qb} g_2(\phi_i(\xb_2)) \\
& \quad+ D^{\pb} g_1(\phi_i(\xb_1)) D^{\qb} g_2(\phi_i(\xb_2)) - D^{\pb} g_1(\phi_i(\xb_2)) D^{\qb} g_2(\phi_i(\xb_2)) \big| \\
\leq & |D^{\pb} g_1(\phi_i(\xb_1))| |D^{\qb} g_2(\phi_i(\xb_1)) - D^{\qb} g_2(\phi_i(\xb_2))| \\
& \quad+ |D^{\qb} g_2(\phi_i(\xb_2))| |D^{\pb} g_1(\phi_i(\xb_1)) - D^{\pb}g_1(\phi_i(\xb_2))| \\
\leq & \mu_i \theta_{i, \alpha} \norm{\phi_i(\xb_1) - \phi_i(\xb_2)}_2^\alpha + \lambda_i \beta_{i, \alpha} \norm{\phi_i(\xb_1) - \phi_i(\xb_2)}_2^\alpha \\
\leq & 2\sqrt{d} \mu_i \lambda_{i} (2r)^{1-\alpha}\norm{\phi_i(\xb_1) - \phi_i(\xb_2)}_2^\alpha.
\end{align*}
Observe that there are totally $2^{s}$ summands in the right hand side of \eqref{eq:twoparts}. Therefore, for any $\xb_1, \xb_2 \in U_i$ and $|\sbb| = s$, we have
\begin{equation*}
\left|D^{\sbb} (f_i \circ \phi_i^{-1}) \big|_{\phi_i(\xb_1)} - D^{\sbb} (f_i \circ \phi_i^{-1}) \big|_{\phi_i(\xb_2)} \right| \leq 2^{s+1} \sqrt{d} \mu_i \lambda_i (2r)^{1-\alpha}\norm{\phi_i(\xb_1) - \phi_i(\xb_2)}_2^\alpha.
\end{equation*}
\end{proof}

\subsection{Proof of Theorem \ref{thm:taylor}}\label{pf:taylor}
\begin{proof}
The proof consists of two steps. We first approximate $f_i \circ \phi_i^{-1}$ by a Taylor polynomial, and then implement the Taylor polynomial using a ReLU network. To ease the analysis, we extend $f_i \circ \phi_i^{-1}$ to the whole cube $[0, 1]^d$ by assigning $f_i \circ \phi_i^{-1}(\xb) = 0$ for $\phi_i(\xb) \in [0, 1]^d \setminus \phi_i(U_i)$. It is straightforward to check that this extension preserves the regularity of $f_i \circ \phi_i^{-1}$, since $f_i$ vanishes on the complement of the compact set $\textrm{supp}(\rho_i) \subset U_i$. For notational simplicity, we denote $f_i^\phi = f_i \circ \phi_i^{-1}$ with the extension. Accordingly, Lemma \ref{lemma:holderfi} can be extended to the whole cube $[0, 1]^d$ without changing its proof, i.e., for any $\xb_1, \xb_2 \in [0, 1]^d$ and $|\sbb| = s$, we have
\begin{align}\label{eq:holderfiphi}
\left|D^{\sbb} f_i^\phi \big|_{\xb_1} - D^{\sbb} f_i^\phi \big|_{\xb_2} \right| \leq 2^{s+1} \sqrt{d} \mu_i \lambda_i (2r)^{1-\alpha} \norm{\xb_1 - \xb_2}_2^\alpha.
\end{align}

\textbf{Step 1.}
We define a trapezoid function
\begin{equation*}
\psi(x) = 
\begin{cases}
1 & |x| < 1 \\
2 - |x| & 1 \leq |x| \leq 2 \\
0 & |x| > 2
\end{cases}.
\end{equation*}
Note that we have $\norm{\psi}_\infty = 1$. Let $N$ be a positive integer, we form a uniform grid on $[0, 1]^d$ by dividing each coordinate into $N$ subintervals. We then consider a partition of unity on these grid defined by
\begin{equation*}
\zeta_\mb (\xb) = \prod_{k=1}^d \psi\rbr{3N\rbr{x_k - \frac{m_k}{N}}}.
\end{equation*}
We can check that $\sum_{\mb} \zeta_\mb(\xb) = 1$ as in Figure \ref{fig:psi}.
\begin{figure}[h]
\centering
\includegraphics[width = 0.5\textwidth]{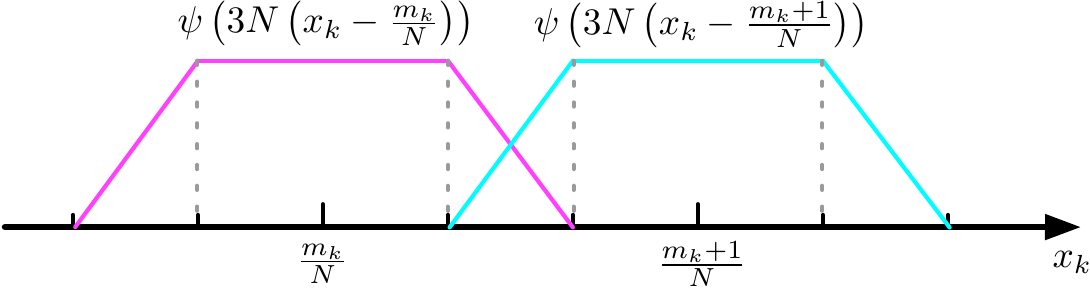}
\caption{Illustration of the construction of $\zeta_\mb$ on the $k$-th coordinate.}
\label{fig:psi}
\end{figure}

We also observe that $\textrm{supp}(\zeta_\mb) = \left\{\xb : \left|x_k - \frac{m_k}{N}\right| \leq \frac{2}{3N}, k = 1, \dots, d \right\} \subset \left\{\xb : \left|x_k - \frac{m_k}{N}\right|\leq \frac{1}{N}, k = 1, \dots, d \right\}$. We use the slightly enlarged support set of length $2/N$ to simplify the constant computation. Now we construct a Taylor polynomial of degree $s$ for approximating $f_i^\phi$ at $\frac{\mb}{N}$:
\begin{equation*}
P_\mb(\xb) = \sum_{|\sbb| \leq s} \frac{D^{\sbb} f_i^\phi}{\sbb!} \bigg|_{\xb = \frac{\mb}{N}}\left(\xb - \frac{\mb}{N}\right)^{\sbb}.
\end{equation*} 
Define $\bar{f}_i = \sum_{\mb \in \{0, \dots, N\}^d} \zeta_\mb P_\mb$. We bound the approximation error $\norm{\bar{f}_i - f_i^\phi}_\infty$:
\begin{align*}
\max_{\xb \in [0, 1]^d} \left|\bar{f}_i(\xb) - f_i^\phi(\xb)\right| & = \max_\xb \left|\sum_\mb \phi_{\mb}(\xb) (P_\mb(\xb) - f_i^\phi(\xb))\right| \\
& \leq \max_\xb \sum_{\mb : \left|x_k - \frac{m_k}{N} \right| \leq \frac{1}{N}} \left|P_\mb(\xb) - f_i^\phi(\xb)\right| \\
& \leq \max_\xb 2^d \max_{\mb: \left|x_k - \frac{m_k}{N} \right| \leq \frac{1}{N}} \left|P_\mb(\xb) - f_i^\phi(\xb)\right| \\
& \overset{(i)}{\leq} \max_\xb \frac{2^d d^s}{s!} \left(\frac{1}{N}\right)^s \max_{|\sbb| = s} \left|D^\sbb f_i^\phi \big|_{\frac{\mb}{N}} - D^\sbb f_i^\phi \big|_{\yb}\right| \\
& \overset{(ii)}{\leq} \max_\xb \frac{2^d d^s}{s!} \left(\frac{1}{N}\right)^s 2^{s + 1} \sqrt{d} \mu_i \lambda_i (2r)^{1-\alpha} \norm{\frac{\mb}{N} - \xb}_2^\alpha \\
& \leq \sqrt{d} \mu_i \lambda_i (2r)^{1-\alpha} \frac{2^{d+s+1} d^{s + \alpha / 2}}{s!} \left(\frac{1}{N}\right)^{s + \alpha}.
\end{align*}
Here $\yb$ is the linear interpolation of $\frac{\mb}{N}$ and $\xb$, determined by the Taylor remainder, and inequality $(i)$ follows from the Taylor expansion of $f_i^\phi$ around $\mb/N$. Note that only $s$-th order derivative remains in step $(i)$ and there are at most $d^s$ terms.
Inequality $(ii)$ is obtained by the H\"{o}lder continuity in the inequality \eqref{eq:holderfiphi}.

By setting $$\sqrt{d} \mu_i \lambda_i (2r)^{1-\alpha} \frac{2^{d+s+1} d^{s + \alpha / 2}}{s!} \left(\frac{1}{N}\right)^{s + \alpha} \leq \frac{\delta}{2},$$ we get $N \geq \left(\frac{\sqrt{d} \mu_i \lambda_i (2r)^{1-\alpha} 2^{d+s+2} d^{s+\alpha/2}}{\delta s!}\right)^{\frac{1}{s+\alpha}}$. Accordingly, the approximation error is bounded by $\lVert \bar{f}_i - f_i^\phi \rVert_\infty \leq \frac{\delta}{2}$.

\textbf{Step 2.} We next implement $\tilde{f}_i$ by a ReLU network that approximates $\bar{f}_i$ up to an error $\frac{\delta}{2}$. We denote
\begin{align*}
P_\mb(\xb) = \sum_{|\sbb| \leq s} a_{\mb, \sbb} \rbr{\xb - \frac{\mb}{N}}^{\sbb},
\end{align*}
where $a_{\mb, \sbb} = \frac{D^{\sbb} f_i^\phi}{\sbb!} \bigg|_{\xb = \frac{\mb}{N}}$. Then we rewrite $\bar{f}_i$ as
\begin{align}\label{eq:tildef}
\bar{f}_i(\xb) = \sum_{\mb \in \{0, \dots, N\}^d} \sum_{|\sbb| \leq s} a_{\mb, \sbb} \zeta_\mb(\xb)\rbr{\xb - \frac{\mb}{N}}^{\sbb}.
\end{align}
Note that \eqref{eq:tildef} is a linear combination of products $\zeta_\mb\rbr{\xb - \frac{\mb}{N}}^{\sbb}$. Each product involves at most $d + n$ univariate terms: $d$ terms for $\zeta_\mb$ and $n$ terms for $\rbr{\xb - \frac{\mb}{N}}^{\sbb}$. We recursively apply Corollary \ref{cor:hatx} to implement the product. Specifically, let $\hat{\times}_\epsilon$ be the approximation of the product operator in Corollary \ref{cor:hatx} with error $\epsilon$, which will be chosen later. Consider the following chain application of $\hat{\times}_\epsilon$:
\begin{align*}
\tilde{f}_{\mb, \sbb}(\xb) = \hat{\times}_\epsilon\left(\psi(3N x_1 - 3m_1), \hat{\times}_\epsilon\left(\dots, \hat{\times}_\epsilon\big(\psi(3N_d x_d - m_d), \hat{\times}_\epsilon\big(x_1 - \frac{m_1}{N}, \dots\big)\big)\right)\right).
\end{align*}
Now we estimate the error of the above approximation. Note that we have $|\psi(3Nx_k - 3m_k)| \leq 1$ and $\left|x_k - \frac{m_k}{N}\right| \leq 1$ for all $k \in \{1, \dots, d\}$ and $\xb \in [0, 1]^d$. We then have
\begin{align*}
\left|\tilde{f}_{\mb, \sbb}(\xb) -  \zeta_\mb\rbr{\xb - \frac{\mb}{N}}^{\sbb} \right| & = \bigg|\hat{\times}_\epsilon\Big(\psi(3N x_1 - 3m_1), \hat{\times}_\epsilon\big(\dots, \hat{\times}_\epsilon(x_1 - \frac{m_1}{N}, \dots)\big)\Big) - \zeta_\mb\rbr{\xb - \frac{\mb}{N}}^{\sbb} \bigg| \\
& \leq \big|\hat{\times}_\epsilon\left(\psi(3N x_1 - 3m_1), \hat{\times}_\epsilon(\psi(3Nx_2 - 3m_2), \dots)\right) \\
& \quad \quad \quad - \psi(3N_1 - 3m_1) \hat{\times}_\epsilon(\psi(3Nx_2 - 3m_2), \dots) \big| \\
& \quad + \left|\psi(3N x_1 - m_1)\right| \left| \hat{\times}_\epsilon(\psi(3Nx_2 - 3m_2), \dots) - \psi(3Nx_2 - 3m_2) \hat{\times}_\epsilon(\dots)\right| \\
& \quad + \dots \\
& \leq (s+d) \epsilon.
\end{align*}
Moreover, we have $\tilde{f}_{\mb, \sbb}(\xb) = \zeta_\mb\rbr{\xb - \frac{\mb}{N}}^{\sbb} = 0$, if $\xb \not\in \textrm{supp}(\zeta_\mb)$. Now we define
\begin{equation*}
\tilde{f}_i = \sum_{\mb \in \{0, \dots, N\}^d} \sum_{|\sbb| \leq s} a_{\mb, \sbb} \tilde{f}_{\mb, \sbb}.
\end{equation*}
The approximation error is bounded by 
\begin{align*}
\max_{\xb} \left|\tilde{f}_i(\xb) - \bar{f}_i(\xb) \right| & = \left|\sum_{\mb \in \{0, \dots, N\}^d} \sum_{|\sbb| \leq n} a_{\mb, \sbb} \left(\tilde{f}_{\mb, \nbb}(\xb) - \zeta_\mb\rbr{\xb - \frac{\mb}{N}}^{\sbb}\right) \right| \\
& \leq \max_{\xb} \lambda_{i} \mu_i 2^{d+s+1} \max_{\mb: \xb \in \textrm{supp}(\zeta_\mb)} \sum_{|\sbb| \leq s} \left|\tilde{f}_{\mb, \sbb}(\xb) - \zeta_\mb\left(\xb - \frac{\mb}{N}\right)^\sbb \right| \\
& \leq \lambda_i \mu_i 2^{d+s+1} d^{s} (d+s)\epsilon.
\end{align*}
We choose $\epsilon = \frac{\delta}{\lambda_i \mu_i 2^{d+s+2} d^{s} (d + s)}$, so that $\lVert\bar{f}_i - \tilde{f}_i\rVert_\infty \leq \frac{\delta}{2}$. Thus, we eventually have $\lVert\tilde{f}_i - f_i^\phi \rVert_\infty \leq \delta$. Now we compute the depth and computational units for implement $\tilde{f}_i$.  $\tilde{f}_i$ can be implemented by a collection of parallel sub-networks that compute each $\tilde{f}_{\mb, \sbb}$. The total number of parallel sub-networks is bounded by $d^s(N+1)^d$. For each sub-network, we observe that $\psi$ can be exactly implemented by a single layer ReLU network, i.e., $\psi(x) = \textrm{ReLU}(x + 2) - \textrm{ReLU}(x + 1) - \textrm{ReLU}(x - 1) + \textrm{ReLU}(x - 2)$. Corollary \ref{cor:hatx} shows that $\hat{\times}_\epsilon$ can be implemented by a depth $c_1 \log \frac{1}{\epsilon}$ ReLU network. Therefore, the whole network for implementing $\tilde{f}_i$ has no more than $c'_1 \rbr{\log \frac{1}{\epsilon} + 1}$ layers with width bounded by $O(d^s(N+1)^d)$ and $c'_1 d^{s} (N+1)^d \rbr{\log \frac{1}{\epsilon} + 1}$ neurons and weight parameters. With $\epsilon = \frac{\delta}{\lambda_i \mu_i2^{d+s+2} d^{s} (d + s)}$ and $N = \Big\lceil \big(\frac{\mu_i\lambda_i (2r)^{1-\alpha}2^{d+s+2} d^{s+\alpha/2}}{\delta s!}\big)^{\frac{1}{s+\alpha}} \Big\rceil$, we obtain that the whole network has no more than $L = c_1 \log \frac{1}{\delta}$ layers, with width bounded by $p = c_2 \delta^{-\frac{d}{s + \alpha}}$, and at most $K = c_2 \delta^{-\frac{d}{s+\alpha}} \rbr{\log \frac{1}{\delta} + 1}$ neurons and weight parameters, for constants $c_1, c_2, c_3$ depending on $d, s, \tau$, and upper bound of derivatives of $f_i \circ \phi_i^{-1}$, up to order $s$. Lastly, from \eqref{eq:tildef}, we see each parameter has a range bounded by the upper bound of derivatives of $f_i \circ \phi_i^{-1}$ up to order $s$ --- scales as $\sqrt{d}$ as in \eqref{eq:holderfiphi}.
\end{proof}

\subsection{Proof of Lemma \ref{prop:error}}\label{pf:error}
\begin{proof}
We expand the estimation error as
\begin{align*}
& \quad \norm{\hat{f} - f}_\infty \\
& = \norm{\sum_{i=1}^{C_\cM} \hat{\times}(\hat{f}_i, \hat{\mathds{1}}_\Delta \circ \hat{d}_i^2) - f}_\infty \\
& = \norm{\sum_{i=1}^{C_\cM} \hat{\times}(\hat{f}_i, \hat{\mathds{1}}_\Delta \circ \hat{d}_i^2) - f \rho_i \mathds{1}(\xb \in U_i)}_\infty \\
& \leq \sum_{i=1}^{C_\cM} \norm{\hat{\times}(\hat{f}_i, \hat{\mathds{1}}_\Delta \circ \hat{d}_i^2) - f_i \mathds{1}(\xb \in U_i)}_\infty \\
& \leq \sum_{i=1}^{C_\cM} \Big\lVert \hat{\times}(\hat{f}_i, \hat{\mathds{1}}_\Delta \circ \hat{d}_i^2) - \hat{f}_i \cdot (\hat{\mathds{1}}_\Delta \circ \hat{d}_i^2) + \hat{f}_i \cdot (\hat{\mathds{1}}_\Delta \circ \hat{d}_i^2) - f_i \cdot (\hat{\mathds{1}}_\Delta \circ \hat{d}_i^2) + f_i \cdot (\hat{\mathds{1}}_\Delta \circ \hat{d}_i^2) - f_i \cdot \mathds{1}(\xb \in U_i) \Big\rVert_\infty \\
& \leq \sum_{i=1}^{C_\cM} \underbrace{\norm{\hat{\times}(\hat{f}_i, \hat{\mathds{1}}_\Delta \circ \hat{d}_i^2) - \hat{f}_i \times (\hat{\mathds{1}}_\Delta \circ \hat{d}_i^2)}_\infty}_{A_{i, 1}} + \underbrace{\norm{\hat{f}_i \times (\hat{\mathds{1}}_\Delta \circ \hat{d}_i^2) - f_i \times (\hat{\mathds{1}}_\Delta \circ \hat{d}_i^2)}_\infty}_{A_{i, 2}} \\
& \quad\quad + \underbrace{\norm{f_i \times (\hat{\mathds{1}}_\Delta \circ \hat{d}_i^2) - f_i \times \mathds{1}(\xb \in U_i)}_\infty}_{A_{i, 3}}.
\end{align*}
The first two terms $A_{i, 1}, A_{i, 2}$ are straightforward to handle, since by the construction we have
\begin{align*}
A_{i, 1} & = \norm{\hat{\times}(\hat{f}_i, \hat{\mathds{1}}_\Delta \circ \hat{d}_i^2) - \hat{f}_i \cdot (\hat{\mathds{1}}_\Delta \circ \hat{d}_i^2)}_\infty \leq \eta, \quad \textrm{and}\\
A_{i, 2} & = \norm{\hat{f}_i \times (\hat{\mathds{1}}_\Delta \circ \hat{d}_i^2) - f_i \cdot (\hat{\mathds{1}}_\Delta \circ \hat{d}_i^2)}_\infty \leq \norm{\hat{f}_i - f_i}_\infty \norm{\hat{\mathds{1}}_\Delta \circ \hat{d}_i^2}_\infty \leq \delta.
\end{align*}
By Lemma \ref{lemma:fibound}, we have $\max_{\xb \in \cK_i} |f_i(\xb)| \leq \frac{c(\pi + 1)}{r(1 - r / \tau)} \Delta$ for a constant $c$ depending on $f_i$. Then we bound $A_{i, 3}$ as
\begin{equation*}
A_{i, 3} = \norm{f_i \times (\hat{\mathds{1}}_\Delta \circ \hat{d}_i^2) - f_i \times \mathds{1}(\xb \in U_i)}_\infty \leq \max_{\xb \in \cK_i} |f_i(\xb)| \leq \frac{c(\pi + 1)}{r(1 - r / \tau)} \Delta.
\end{equation*} 
\end{proof}

\subsection{Helper Lemma for Bounding $A_{i, 3}$ and Its Proof}\label{pf:fibound}
\begin{lemma}\label{lemma:fibound}
For any $i = 1, \dots, C_\cM$, denote $$\cK_i = \{\xb \in \cM : r^2 - \Delta \leq \norm{\xb - \cbb_i}_2^2 \leq r^2 \}.$$ Then there exists a constant $c$ depending on the upper bounds of the first derivatives of the partition of unity $\rho_i$'s and coordinate system $\phi_i$'s such that
\begin{align*}
\max_{\xb \in \cK_i} |f_i(\xb)| \leq \frac{c(\pi + 1)}{r(1 - r / \tau)} \Delta.
\end{align*}
\end{lemma}
\begin{proof}
We extend $f_i \circ \phi_i^{-1}$ to the whole cube $[0, 1]^d$ as in the proof of Theorem \ref{thm:taylor}. We also have $f_i(\xb) = 0$ for $\norm{\xb - \cbb_i}_2 = r$. By the first order Taylor expansion, for any $\xb, \yb \in U_i$, we have
\begin{align*}
\left|f_i(\xb) -  f_i(\yb)\right| & = \left|f_i \circ \phi_i^{-1} (\phi_i(\xb)) -  f_i \circ \phi_i^{-1} (\phi_i(\yb)) \right| \\
& \leq \norm{\nabla (f_i \circ \phi_i^{-1})(\zb)}_2 \norm{\phi_i(\xb) -\phi_i(\yb)}_2 \\
& \leq \norm{\nabla (f_i \circ \phi_i^{-1})(\zb)}_2 b_i \norm{V_i}_2 \norm{\xb -\yb}_2,
\end{align*}
where $\zb$ is a linear interpolation of $\phi_i(\xb)$ and $\phi_i(\yb)$ satisfying the mean value theorem. Since $f_i \circ \phi_i^{-1}$ is $C^s$ in $[0, 1]^d$, the first derivative is uniformly bounded, i.e., $\norm{\nabla f_i \circ \phi_i^{-1}(\zb)}_2 \leq \alpha_i$ for any $\zb \in [0, 1]^d$. Let $\yb \in U_i$ satisfying $f_i(\yb) = 0$. In order to bound the function value for any $\xb \in \cK_i$, we only need to bound the Euclidean distance between $\xb$ and $\yb$. More specifically, for any $\xb \in \cK_i$, we need to show that there exists $\yb \in U_i$ satisfying $f_i(\yb) = 0$, such that $\norm{\xb - \yb}_2$ is sufficiently small.

Before continuing with the proof, we introduce some notations. Let $\gamma(t)$ be a geodesic on $\cM$ parameterized by the arc length. In the following context, we use $\dot{\gamma}$ and $\ddot{\gamma}$ to denote the first and second derivatives of $\gamma$ with respect to $t$. By the definition of geodesic, we have $\norm{\dot{\gamma}(t)}_2 = 1$ (unit speed) and $\ddot{\gamma}(t) \perp \dot{\gamma}(t)$. 

Without loss of generality, we shift $\cbb_i$ to $\mathbf{0}$. We consider a geodesic starting from $\xb$ with initial ``velocity'' $\dot{\gamma}(0) = \vb$ in the tangent space of $\cM$ at $\xb$. To utilize polar coordinate, we define two auxiliary quantities: $\ell(t) = \norm{\gamma(t)}_2$ and $\theta(t) = \arccos \frac{\gamma(t)^\top \dot{\gamma}(t)}{\norm{\gamma(t)}_2} \in [0, \pi]$. As can be seen in Figure \ref{fig:elltheta}, $\ell$ and $\theta$ have clear geometrical interpretations: $\ell$ is the radial distance from the center $\cbb_i$, and $\theta$ is the angle between the velocity and $\gamma(t)$.
\begin{figure}[!htb]
\centering
\includegraphics[width = 0.35\textwidth]{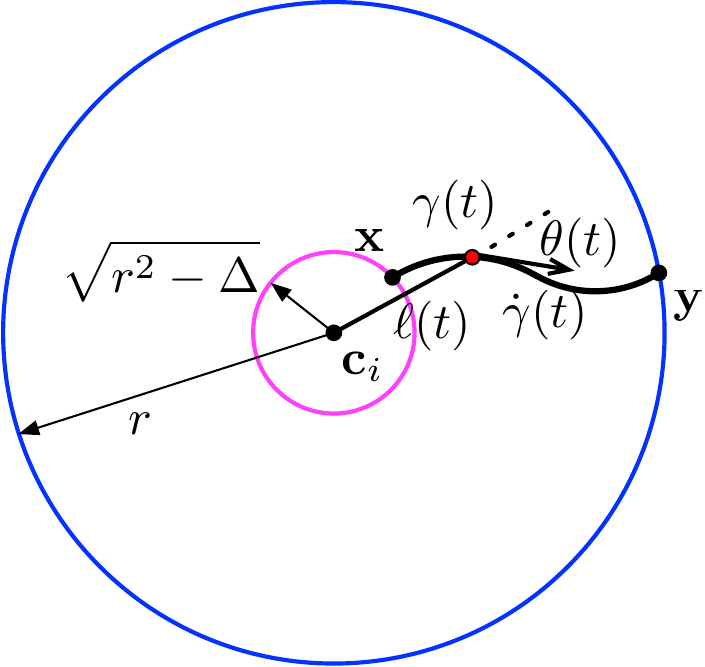}
\caption{Illustration of $\ell$ and $\theta$ along a parametric curve $\gamma$.}
\label{fig:elltheta}
\end{figure}

Suppose $\yb = \gamma(T)$, we need to upper bound $T$. Note that $\ell(T) - \ell(0) \leq r - \sqrt{r^2 - \Delta} \leq \Delta / r$. Moreover, observe that the derivative of $\ell$ is $\dot{\ell}(t) = \cos \theta(t)$, since $\gamma$ has unit speed. It suffices to find a lower bound on $\dot{\ell}(t) = \cos \theta(t)$ so that $T \leq \frac{\Delta}{r \inf_t \dot{\ell}(t)}$.

We immediately have the second derivative of $\ell$ as $\ddot{\ell}(t) = - \sin \theta(t) \dot{\theta}(t)$. Meanwhile, using the equation $\ell(t) = \sqrt{\gamma(t)^\top \gamma(t)}$, we also have
\begin{equation}\label{eq:ddotell}
\ddot{\ell}(t) =  \frac{\left(\ddot{\gamma}(t)^\top \gamma(t) + \dot{\gamma}(t)^\top \dot{\gamma}(t)\right) \sqrt{\gamma(t)^\top \gamma(t)} - \left(\gamma(t)^\top \dot{\gamma}(t)\right)^2 / \sqrt{\gamma(t)^\top \gamma(t)}}{\gamma(t)^\top \gamma(t)}.
\end{equation}
Note that by definition, we have $\dot{\gamma}(t)^\top \dot{\gamma}(t) = 1$ and $\gamma(t)^\top \dot{\gamma}(t) = \cos \theta(t) \sqrt{\gamma(t)^\top \gamma(t)}$. Plugging into \eqref{eq:ddotell}, we can derive
\begin{equation}\label{eq:ddotgamma}
\ddot{\ell}(t) = \frac{1 + \ddot{\gamma}(t)^\top \gamma(t) - \cos^2 \theta(t)}{\ell(t)} = \frac{\sin^2 \theta(t) + \ddot{\gamma}(t)^\top \gamma(t)}{\ell(t)}.
\end{equation}
Now we find a lower bound on $\ddot{\gamma}(t)^\top \gamma(t)$. Specifically, by Cauchy-Schwarz inequality, we have
\begin{align*}
\ddot{\gamma}(t)^\top \gamma(t) & \geq - \norm{\ddot{\gamma}(t)}_2 \norm{\gamma(t)}_2 \left|\cos \angle \rbr{\ddot{\gamma}(t), \gamma(t)} \right| \\
& \geq - \frac{r}{\tau} \left|\cos \angle \rbr{\ddot{\gamma}(t), \gamma(t)} \right|.
\end{align*}
The last inequality follows from $\norm{\ddot{\gamma}(t)}_2 \leq \frac{1}{\tau}$ \citep{niyogi2008finding} and $\norm{\gamma(t)}_2 \leq r$. We now need to bound $\angle (\ddot{\gamma}(t), \gamma(t))$, given $\angle \rbr{\gamma(t), \dot{\gamma}(t)} = \theta(t)$ and $\ddot{\gamma}(t) \perp \dot{\gamma}(t)$. Consider the following optimization problem,
\begin{align}
\min \quad & \ab^\top \xb, \label{eq:cosangle}\\
\textrm{subject to} \quad & \xb^\top \xb = 1 \notag, \\
& \bbb^\top \xb = 0. \notag
\end{align}
By assigning $\ab = \frac{\gamma(t)}{\norm{\gamma(t)}_2}$ and $\bbb = \frac{\dot{\gamma}(t)}{\norm{\dot{\gamma}(t)}_2}$, the optimal objective value is exactly the minimum of $\cos \angle \rbr{\ddot{\gamma}(t), \gamma}$. Additionally, we can find the maximum of $\cos \angle \rbr{\ddot{\gamma}(t), \gamma}$ by replacing the minimization in \eqref{eq:cosangle} by maximization. We solve \eqref{eq:cosangle} by the Lagrangian method. More precisely, let
\begin{align*}
\cL(x, \lambda, \mu) = -\ab^\top \xb + \lambda(\xb^\top \xb - 1) + \mu (\bbb^\top \xb).
\end{align*}
We have the optimal solution $\xb^*$ satisfying $\nabla_x \cL = 0$, which implies $\xb^* = \frac{1}{2\lambda^*}(\ab - \mu^* \bbb)$ with $\mu^*$ and $\lambda^*$ being the optimal dual variable. By the primal feasibility, we have $\mu^* = \ab^\top \bbb$ and $\lambda^* = - \frac{1}{2}\sqrt{1 - (\ab^\top \bbb)^2}$. Therefore, the optimal objective value is $- \sqrt{1 - (\ab^\top \bbb)^2}$. Similarly, the maximum is $\sqrt{1 - (\ab^\top \bbb)^2}$. Note that $\ab^\top \bbb = \cos \theta(t)$, we then get
\begin{equation*}
\ddot{\gamma}(t)^\top \gamma(t) \geq - \frac{r}{\tau} \sin \theta(t).
\end{equation*}
Substituting into \eqref{eq:ddotgamma}, we have the following lower bound
\begin{equation*}
\ddot{\ell}(t) = \frac{\sin \theta^2(t) + \ddot{\gamma}(t)^\top \gamma(t)}{\ell(t)} \geq \frac{1}{\ell(t)} \left(\sin^2 \theta(t) - \frac{r}{\tau} \sin \theta(t) \right).
\end{equation*}
Now combining with $\ddot{\ell}(t) = -\sin \theta(t) \dot{\theta}(t)$, we can derive
\begin{equation}\label{eq:dottheta}
\dot{\theta}(t) \leq - \frac{1}{\ell(t)} \left(\sin \theta(t) - \frac{r}{\tau} \right).
\end{equation}
Inequality \eqref{eq:dottheta} has an important implication: When $\sin \theta(t) > \frac{r}{\tau}$, as $t$ increasing, $\theta(t)$ is monotone decreasing until $\sin \theta(t') = \frac{r}{\tau}$ for some $t' = t$. Thus, we distinguish two cases depending on the value of $\theta(0)$. Indeed, we only need to consider $\theta(0) \in [0, \pi / 2]$. The reason behind is that if $\theta(0) \in (\pi / 2, \pi]$, we only need to set the initial velocity in the opposite direction.

\textbf{Case 1}: $\theta(0) \in \left[0, \arcsin \frac{r}{\tau}\right]$. We claim that $\theta(t) \in \left[0, \arcsin \frac{r}{\tau}\right]$ for all $t \leq T$. In fact, suppose there exists some $t_1 \leq T$ such that $\theta(t_1) > \arcsin \frac{r}{\tau}$. By the continuity of $\theta$, there exists $t_0 < t_1$, such that $\theta(t_0) = \arcsin \frac{r}{\tau}$ and $\theta(t) \geq \arcsin \frac{r}{\tau}$ for $t \in [t_0, t_1]$. This already gives us a contradiction:
\begin{equation*}
\theta(t_0) < \theta(t_1) = \theta(t_0) + \underbrace{\int_{t_0}^{t_1} \dot{\theta}(t) dt}_{\leq 0} \leq \theta(t_0).
\end{equation*}
Therefore, we have $\dot{\ell}(t) \geq \cos \arcsin \frac{r}{\tau} = \sqrt{1 - \frac{r^2}{\tau^2}}$, and thus $T \leq \frac{\Delta}{r\sqrt{1 - \frac{r^2}{\tau^2}}}$.

\textbf{Case 2}: $\theta(0) \in \big(\arcsin \frac{r}{\tau}, \pi / 2\big]$. It is enough to show that $\theta(0)$ can be bounded sufficiently away from $\pi / 2$. Let $\gamma_{\cbb, \xb} \subset \cM$ be a geodesic from $\cbb_i$ to $\xb$. We analogously define $\theta_{\cbb, \xb}$ and $\ell_{\cbb, \xb}$ as for the geodesic from $\xb$ to $\yb$. Let $T_{r/2} = \sup{\{t:\ell_{\cbb, \xb}(t) \leq r/2 - \Delta / r\}} $, and denote $\zb = \gamma_{\cbb, \xb}(T_{r/2})$. We must have $\theta_{\cbb, \xb}(T_{r/2}) \in [0, \pi / 2]$ and $\ell_{\cbb, \xb}(T_{r/2})=r/2 - \Delta / r$, otherwise there exists $T'_{r/2} > T_{r/2}$ satisfying $\ell_{\cbb, \xb}(T'_{r/2}) \leq r/2$. Denote $T_\xb$ satisfying $\xb = \gamma_{\cbb, \xb}(T_\xb)$. We bound $\theta_{\cbb, \xb}(T_\xb)$ as follows,
\begin{align*}
\theta_{\cbb, \xb}(T_\xb) & = \theta_{\cbb, \xb}(T_{r/2}) + \int_{T_{r/2}}^{T_\xb} \dot{\theta}_{\cbb, \xb}(t) dt \\
& \leq \frac{\pi}{2} - \int_{T_{r/2}}^{T_\xb} \frac{1}{\ell_{\cbb, \xb}(t)} \left(\sin \theta_{\cbb, \xb}(t) - \frac{r}{\tau} \right) dt.
\end{align*}
If there exists some $t \in (T_{r/2}, T_\xb]$ such that $\sin \theta_{\cbb, \xb}(t) \leq \frac{r}{\tau}$, by the previous reasoning, we have $\sin \theta_{\cbb, \xb}(T_\xb) \leq \frac{r}{\tau}$. Thus, we only need to handle the case when $\sin \theta_{\cbb, \xb}(t) > \frac{r}{\tau}$ for all $t \in (T_{r/2}, T_\xb]$. In this case, $\theta_{\cbb, \xb}(t)$ is monotone decreasing, hence we further have
\begin{align*}
\theta_{\cbb, \xb}(T_\xb) & \leq \frac{\pi}{2} - \int_{T_{r/2}}^{T_\xb} \frac{1}{\ell_{\cbb, \xb}(t)} \left(\sin \theta_{\cbb, \xb}(T_\xb) - \frac{r}{\tau} \right) dt \\
& \leq \frac{\pi}{2} - (T_\xb - T_{r/2}) \frac{1}{r} \left(\sin \theta_{\cbb, \xb}(T_\xb) - \frac{r}{\tau} \right) \\
& \leq \frac{\pi}{2} - \frac{1}{2} \left(\sin \theta_{\cbb, \xb}(T_\xb) - \frac{r}{\tau} \right).
\end{align*}
The last inequality follows from $T_\xb - T_{r/2} \geq r / 2$. Using the fact, $\sin x \geq \frac{2}{\pi} x$, we can derive
\begin{align*}
& \theta_{\cbb, \xb}(T_\xb) \leq \frac{\pi}{2} - \frac{1}{2} \left(\frac{2}{\pi} \theta_{\cbb, \xb}(T_\xb) - \frac{r}{\tau} \right) \\
\Longrightarrow ~~& \theta_{\cbb, \xb}(T_\xb) \leq \frac{\pi}{2} \left(\frac{\pi + r / \tau}{\pi + 1} \right).
\end{align*}
We can then set $\theta(0) = \theta_{\cbb, \xb}(T_\xb)$, and thus
\begin{align*}
\cos \theta(0) \geq \cos \left(\frac{\pi}{2} \frac{\pi + r / \tau}{\pi + 1}\right) & = \cos \left(\frac{\pi}{2} \left(1 - \frac{1 - r / \tau}{\pi + 1}\right)\right) \\
& = \sin \left(\frac{\pi}{2} \frac{1 - r / \tau}{\pi + 1}\right) \\
& \geq \frac{1 - r / \tau}{\pi + 1}.
\end{align*}
Therefore, we have $T \leq \frac{\Delta}{r \cos \theta(0)} \leq \frac{\pi + 1}{r(1 - r / \tau)} \Delta$. By the choice of $r \leq \tau / 4$, we immediately have $\frac{\tau}{\sqrt{\tau^2 - r^2}} < \frac{\pi + 1}{1 - r / \tau}$. Hence, combining case 1 and case 2, we conclude
\begin{equation*}
T \leq \frac{\pi + 1}{r(1 - r / \tau)} \Delta.
\end{equation*}
Therefore, the function value $f(\xb)$ on $\cK_i$ is bounded by $\alpha_i \frac{\pi + 1}{r(1 - r / \tau)} \Delta$. It suffices to set $c = \max_i \alpha_i b_i \norm{V_i}_2$, and we complete the proof.
\end{proof}

\subsection{Characterization of the Size of the ReLU Network}\label{pf:size}
\begin{proof}
We evenly split the error $\epsilon$ into $3$ parts for $A_{i, 1}, A_{i, 2}$, and $A_{i, 3}$, respectively. We pick $\eta = \frac{\epsilon}{3C_\cM}$ so that $\sum_{i=1}^{C_\cM} A_{i, 1} \leq \frac{\epsilon}{3}$. The same argument yields $\delta = \frac{\epsilon}{3C_\cM}$. Analogously, we can choose $\Delta = \frac{r(1-r/\tau)\epsilon}{3c(\pi + 1)C_\cM}$. Finally, we pick $\nu = \frac{\Delta}{16B^2 D}$ so that $8B^2 D \nu < \Delta$.

Now we compute the number of layers, width, the number of neurons and weight parameters, and the range of each weight parameter in the ReLU network identified by Theorem \ref{thm:bias}. 
\begin{enumerate}
\item For the chart determination sub-network, $\hat{\mathds{1}}_\Delta$ can be implemented by a ReLU network with $\left \lceil\log \frac{r^2}{\Delta} \right \rceil$ layers and $2$ neurons in each layer. The weight parameters in the network is bounded by $O(\max\{\tau^2, 1\})$. The approximation of the distance function $\hat{d}_i^2$ can be implemented by a network of depth $O\left(\log \frac{1}{\nu}\right)$, width bounded by a constant, and the number of neurons and weight parameters is at most $O\left(\log \frac{1}{\nu}\right)$. Each weight parameter is bounded by $B$. Plugging in our choice of $\nu$ and $\Delta$, we have the depth is no greater than $c_1 \left(\log \frac{1}{\epsilon} + \log D\right)$ with $c_1$ depending on $d, f, \tau$, and the surface area of $\cM$. The number of neurons and weight parameters is also $c'_1 \left(\log \frac{1}{\epsilon} + \log D\right)$ except for a different constant. Note that there are $D$ parallel networks computing $\hat{d}_i^2$ for $i = 1, \dots, C_\cM$. Hence, the total number of neurons and weight parameters is $c'_1 C_\cM D \left(\log \frac{1}{\epsilon} + \log D\right)$ with $c'_1$ depending on $d, f, \tau$, and the surface area of $\cM$. As can be seen, the width of the chart-determination network is bounded by $O(C_{\cM}D)$, and the weight parameter is bounded by $O(\max\{1, \tau^2, B\})$.

\item For the Taylor polynomial sub-network, $\phi_i$ can be implemented by a linear network with at most $D d$ weight parameters. To implement each $\hat{f}_i$, we need a ReLU network of depth $c_4 \log \frac{1}{\delta}$. The number of neurons and weight parameters is $c'_4 \delta^{-\frac{d}{s+\alpha}} \log \frac{1}{\delta}$, and the width is bounded by $c''_4\delta^{-\frac{d}{s + \alpha}}$. Here $c_4, c'_4, c''_4$ depend on $s, d, \tau, f_i \circ \phi_i^{-1}$. In addition, all the weight parameters are bounded by the upper bound of the derivatives of $f_i \circ \phi_i^{-1}$ up to order $s$ (which scales as $\sqrt{d}$ as in Lemma \ref{lemma:holderfi}). Substituting $\delta = \frac{\epsilon}{3C_\cM}$, we get the depth is $c_2 \log \frac{1}{\epsilon}$ and the number of neurons and weight parameters is $c'_2 \epsilon^{-\frac{d}{s+\alpha}} \log \frac{1}{\epsilon}$. There are totally $C_\cM$ parallel $\hat{f}_i$'s, hence the width is further bounded by $c''_2 C_{\cM} \epsilon^{-\frac{d}{s + \alpha}}$. Meanwhile, the total number of neurons and weight parameters is $c'_2 C_\cM \epsilon^{-\frac{d}{s+\alpha}} \log \frac{1}{\epsilon}$. Here constants $c'_2$ and $c''_2$ depend on $d, s, f_i \circ \phi_i^{-1}, \tau$, and the surface area of $\cM$.

\item For the product sub-network, the analysis is similar to the chart determination sub-network. The depth is $O\left(\log \frac{1}{\eta}\right)$, the width is bounded by a constant, he number of neurons and weight parameters is $O\left(\log \frac{1}{\eta}\right)$, and all the weight parameters are bounded by a constant. The choice of $\eta$ yields that the depth is $c_3 \log \frac{1}{\epsilon}$, and the number of neurons and weight parameters is $c'_3 \log \frac{1}{\epsilon}$. There are $C_\cM$ parallel pairs of outputs from the chart determination and the Taylor polynomial sub-networks. Hence, the total number of weight parameters is $c'_3 C_\cM \log \frac{1}{\epsilon}$ with $c'_3$ depending on $d, \tau$, and the surface area of $\cM$.
\end{enumerate}
Combining these 3 sub-networks, and redefining the constants $c_1$, $c_2$, $c_3$ and $c_4$ in the sequel, we obtain that the depth of the full network is $L = c_1 \left(\log \frac{1}{\epsilon} + \log D \right)$ for some constant $c_1$ depending on $d, s, \tau$, and the surface area of $\cM$. The depth of the neural network is bounded by $p = c_2 (\epsilon^{-\frac{d}{s + \alpha}} + D)$ with $c_2$ depending on $d, s, \tau$, the surface area of $\cM$, and the upper bounds on derivatives of $\phi_i$'s and $\rho_i$'s, up to order $s$. The total number of neurons and weight parameters is $K = c_3 \left(\epsilon^{-\frac{d}{s+\alpha}} \log \frac{1}{\epsilon} + D \log \frac{1}{\epsilon} + D \log D\right)$ for some constant $c_3$ depending on $d, s, f, \tau$, and the surface area of $\cM$. Lastly, all the weight parameters in the network is bounded by $c_4 \max\{1, \tau^2, B, \sqrt{d}\}$ with $c_4$ depends on the upper bound of derivatives of $\rho_i$'s up to order $s$.
\end{proof}

\section{Proof of Statistical Recovery of ReLU Network (Theorem \ref{thm:stat})}
\numberwithin{equation}{section}
This section consists of the detailed proofs, in Section \ref{pf:t1}, \ref{pf:t2} and \ref{pf:coveringbound}, respectively, for upper bounding bias in Lemma \ref{lemma:t1}, upper bounding variance in Lemma \ref{lemma:t2}, and upper bounding covering number in Lemma \ref{lemma:coveringbound}. Lastly, the statistical bound in Theorem \ref{thm:stat} is established in Section \ref{pf:regression} by choosing a proper approximation error and covering accuracy via the bias-variance trade-off argument.
\subsection{Proof of Lemma \ref{lemma:t1}}
\begin{proof}\label{pf:t1}
$T_1$ essentially reflects the bias of estimating $f_0$:
\begin{allowdisplaybreaks}
\begin{align}
T_1 & = \EE \left[\frac{2}{n}\sum_{i=1}^n (\hat{f}_n(\xb_i) - f_0(\xb_i) - \xi_i + \xi_i)^2 \right] \notag\\
& = \frac{2}{n} \EE \left[\sum_{i=1}^n (\hat{f}_n(\xb_i) - f_0(\xb_i) - \xi_i)^2 + 2 \xi_i (\hat{f}_n(\xb_i) - f_0(\xb_i) - \xi_i) + \xi_i^2 \right] \notag\\
& \overset{(i)}{=} \frac{2}{n} \EE \left[\sum_{i=1}^n (\hat{f}_n(\xb_i) - f_0(\xb_i) - \xi_i)^2 + 2 \xi_i \hat{f}_n(\xb_i) - \xi_i^2\right] \notag\\
& = \frac{2}{n} \EE \left[\sum_{i=1}^n (\hat{f}_n(\xb_i) - y_i)^2 + 2 \xi_i \hat{f}_n(\xb_i) - \xi_i^2 \right] \notag\\
& = \frac{2}{n} \EE \left[\inf_{f \in \cF(R, \kappa, L, p, K)} \sum_{i=1}^n (f(\xb_i) - y_i)^2 + 2 \xi_i \hat{f}_n(\xb_i) - \xi_i^2 \right] \notag\\
& \overset{(ii)}{\leq} 2 \inf_{f \in \cF(R, \kappa, L, p, K)} \EE\left[\frac{1}{n}\sum_{i=1}^n (f(\xb_i) - f_0(\xb_i) - \xi_i)^2 - \xi_i^2 \right] + \EE \left[\frac{4}{n} \sum_{i=1}^n \xi_i \hat{f}_n(\xb_i) \right] \nonumber \\
& = 2 \inf_{f \in \cF(R, \kappa, L, p, K)} \EE\left[\frac{1}{n}\sum_{i=1}^n (f(\xb_i) - f_0(\xb_i))^2 - 2\xi_i(f(\xb_i) - f_0(\xb_i)) \right] + \EE \left[\frac{4}{n} \sum_{i=1}^n \xi_i \hat{f}_n(\xb_i) \right] \nonumber \\
& = 2 \inf_{f \in \cF(R, \kappa, L, p, K)}\int_{\cM} (f(\xb) - f_0(\xb))^2 d\cD_x(\xb) + \EE \left[\frac{4}{n} \sum_{i=1}^n \xi_i \hat{f}_n(\xb_i) \right] \label{eq:t1},
\end{align}
\end{allowdisplaybreaks}%
where $(i)$ follows from $\EE[\xi_i f_0(\xb_i)] = 0$ due to the independence between $\xi_i$ and $\xb$, and $(ii)$ follows from Jensen's inequality. Now we need to bound $\EE \left[\frac{1}{n} \sum_{i=1}^n \xi_i \hat{f}_n(\xb_i) \right]$. We discretize the class $\cF(R, \kappa, L, p, K)$ into $\cF^*(R, \kappa, L, p, K) = \{f^*_i\}_{i=1}^{\cN(\delta, \cF(R, \kappa, L, p, K), \norm{\cdot}_\infty)}$, where $\cN(\delta, \cF(R, \kappa, L, p, K), \norm{\cdot}_\infty)$ denotes the $\delta$-covering number with respect to the $\ell_\infty$ norm. Accordingly, there exists $f^*$ such that $\|f^* - \hat{f}_n\|_\infty \leq \delta$. Denote $\|\hat{f}_n - f_0\|_n^2 = \frac{1}{n} \sum_{i=1}^n (\hat{f}_n(\xb_i) - f_0(\xb_i))^2$. Then we have
\begin{align}
\EE \left[\frac{1}{n} \sum_{i=1}^n \xi_i \hat{f}_n(\xb_i) \right] & = \EE \left[\frac{1}{n} \sum_{i=1}^n \xi_i (\hat{f}_n(\xb_i) - f^*(\xb_i) + f^*(\xb_i) - f_0(\xb_i)) \right] \nonumber \\
& \overset{(i)}{\leq} \EE \left[\frac{1}{n} \sum_{i=1}^n \xi_i (f^*(\xb_i) - f_0(\xb_i)) \right] + \delta \sigma \nonumber \\
& = \EE \left[\frac{\norm{f^* - f_0}_n}{\sqrt{n}} \frac{\sum_{i=1}^n \xi_i (f^*(\xb_i) - f_0(\xb_i))}{\sqrt{n} \norm{f^* - f_0}_n} \right] + \delta \sigma \nonumber \\
& \overset{(ii)}{\leq} \sqrt{2}\EE \left[\frac{\|\hat{f}_n - f_0\|_n + \delta}{\sqrt{n}} \left|\frac{\sum_{i=1}^n \xi_i (f^*(\xb_i) - f_0(\xb_i))}{\sqrt{n} \norm{f^* - f_0}_n}\right| \right] + \delta \sigma. \label{eq:t1cauchy}
\end{align}
Here $(i)$ is obtained by applying H\"{o}lder's inequality to $\xi_i (\hat{f}_n(\xb_i) - f^*(\xb_i))$ and invoking the Jensen's inequality:
\begin{align*}
\EE \left[\frac{1}{n} \sum_{i=1}^n \xi_i (\hat{f}_n(\xb_i) - f^*(\xb_i)) \right] & \leq \EE \left[\frac{1}{n} \sum_{i=1}^n |\xi_i| \norm{f^* - \hat{f}_n}_\infty \right] \\
& \leq \frac{1}{n} \sum_{i=1}^n \EE [|\xi_i|] \delta \\
& \leq \frac{1}{n} \sum_{i=1}^n \sqrt{\EE [|\xi_i|^2]} \delta \\
& \leq \delta \sigma.
\end{align*}
Step $(ii)$ holds, since by invoking the inequality $2ab \leq a^2 + b^2$, we have
\begin{align*}
\norm{f^* - f_0}_n & = \sqrt{\frac{1}{n} \sum_{i=1}^n (f^*(\xb_i) - \hat{f}_n(\xb_i) + \hat{f}_n(\xb_i) - f_0(\xb_i))^2} \\
& \leq \sqrt{\frac{2}{n} \sum_{i=1}^n (f^*(\xb_i) - \hat{f}_n(\xb_i))^2 + (\hat{f}_n(\xb_i) - f_0(\xb_i))^2} \\
& \leq \sqrt{\frac{2}{n} \sum_{i=1}^n \left[\delta^2 + \hat{f}_n(\xb_i) - f_0(\xb_i))^2\right]} \\
& \leq \sqrt{2}\norm{\hat{f}_n - f_0}_n + \sqrt{2}\delta.
\end{align*}
To bound the expectation term in \eqref{eq:t1cauchy}, we first break the dependence between $f^*$ and the samples $(\xb_i, y_i)$. In detail, we replace $f^*$ by any $f^*_j$ in the $\delta$-covering, and observe that $\frac{\sum_{i=1}^n \xi_i (f^*(\xb_i) - f_0(\xb_i))}{\sqrt{n} \norm{f^* - f_0}_n} \leq \max_{j} \frac{\sum_{i=1}^n \xi_i (f_j^*(\xb_i) - f_0(\xb_i))}{\sqrt{n} \| f_j^* - f_0\|_n}$. For notational simplicity, we denote $z_j = \frac{\sum_{i=1}^n \xi_i (f_j^*(\xb_i) - f_0(\xb_i))}{\sqrt{n} \|f_j^* - f_0\|_n}$. Applying Cauchy-Schwarz inequality, we cast the expectation term in \eqref{eq:t1cauchy} as
\begin{align}
& \quad~ \EE \left[\frac{\|\hat{f}_n - f_0\|_n + \delta}{\sqrt{n}} \left|\frac{\sum_{i=1}^n \xi_i (f^*(\xb_i) - f_0(\xb_i))}{\sqrt{n} \norm{f^* - f_0}_n}\right| \right] \nonumber \\
& \leq \EE \left[\frac{\|\hat{f}_n - f_0\|_n + \delta}{\sqrt{n}} \max_j \left| z_j \right| \right] \nonumber \\
& = \EE \left[\frac{\|\hat{f}_n - f_0\|_n}{\sqrt{n}} \max_j \left| z_j\right| + \frac{\delta}{\sqrt{n}} \max_j \left|z_j\right| \right] \nonumber \\
& \leq \EE\left[\left(\sqrt{\frac{1}{n}\EE\left[\|\hat{f}_n - f_0\|_n^2\right]} + \frac{\delta}{\sqrt{n}}\right) \sqrt{\EE\left[\max_j~ z_j^2 \right]} \right]. \label{eq:t1cauchyexpectation}
\end{align}
For given $\xb_1, \dots, \xb_n$, each term $\frac{\sum_{i=1}^n \xi_i (f_j^*(\xb_i) - f_0(\xb_i))}{\sqrt{n} \| f_j^* - f_0 \|_n}$ is sub-guassian with parameter $\sigma$. Consequently, the last inequality \eqref{eq:t1cauchyexpectation} involves the maximum of a collection of squared sub-Gaussian random variables $z_j^2$. Indeed, $z_j^2$ is sub-exponential for each $j$. We can bound it using the moment generating function: For any $t > 0$, we have
\begin{align}
\EE\left[\max_j~ z_j^2 ~|~ \xb_1, \dots, \xb_n \right] & = \frac{1}{t} \log \exp\left(t \EE[\max_j~ z_j^2 ~|~ \xb_1, \dots, \xb_n]\right) \nonumber \\
& \overset{(i)}{\leq} \frac{1}{t} \log \EE\left[\exp \left(t \max_j~ z_j^2 \right) | \xb_1, \dots, \xb_n \right] \nonumber \\
& \leq \frac{1}{t} \log \EE\left[\sum_j \exp \left(t z_j^2 \right) | \xb_1, \dots, \xb_n \right] \nonumber \\
& \leq \frac{1}{t} \log \cN(\delta, \cF(R, \kappa, L, p, K), \norm{\cdot}_\infty) + \frac{1}{t} \log \EE[\exp(t z_1^2) | \xb_1, \dots, \xb_n]. \label{eq:t1momentgenerating}
\end{align}
Since $z_1$ is $\sigma^2$-sub-Gaussian given $\xb_1, \dots, \xb_n$, we derive
\begin{align}
\EE[\exp(t z_1^2) | \xb_1, \dots, \xb_n] & = 1 + \sum_{p=1}^\infty \frac{t^p \EE[z_1^{2p} | \xb_1, \dots, \xb_n]}{p!} \nonumber \\
& = 1 + \sum_{p=1}^\infty \left[\frac{t^p}{p!} \int_0^{\infty} \PP(|z_1| \geq u^{1/2p}) du\right] \nonumber \\
& \leq 1 + 2 \sum_{p=1}^\infty \left[\frac{t^p}{p!} \int_0^{\infty} \exp\left(-\frac{u^{1/p}}{2\sigma^2}\right) du\right] \nonumber \\
& = 1 + 2 \sum_{p=1}^\infty (2t\sigma^2)^p. \nonumber
\end{align}
Taking $t = (3\sigma^2)^{-1}$ and substituting into \eqref{eq:t1momentgenerating}, we deduce $\EE\left[\max_j~ z_j^2 ~|~ \xb_1, \dots, \xb_n \right]$ is bounded by
\begin{align}
\EE\left[\max_j~ z_j^2 ~|~ \xb_1, \dots, \xb_n \right] & \leq 3\sigma^2 \log \cN(\delta, \cF(R, \kappa, L, p, K), \norm{\cdot}_\infty) + 3\sigma^2 \log 5 \nonumber \\
& \leq 3\sigma^2 \log \cN(\delta, \cF(R, \kappa, L, p, K), \norm{\cdot}_\infty) + 6\sigma^2. \label{eq:t1expsubexponential}
\end{align}
Combining \eqref{eq:t1expsubexponential}, \eqref{eq:t1cauchyexpectation}, \eqref{eq:t1cauchy}, and substituting back into \eqref{eq:t1}, we obtain the following implicit error estimation on $T_1$:
\begin{align}
T_1 = 2 \EE\left[\lVert \hat{f}_n - f_0 \rVert_n^2\right] & \leq 2 \inf_{f \in \cF(R, \kappa, L, p, K)} \int_{\cM} (f(\xb) - f_0(\xb))^2 d\cD_x(\xb) + 4\delta \sigma \nonumber \\
& \quad + 4\sqrt{6} \sigma \left(\sqrt{\EE \left[\|\hat{f}_n - f_0\|_n^2 \right]} + \delta \right) \sqrt{\frac{\log \cN(\delta, \cF(R, \kappa, L, p, K), \norm{\cdot}_\infty) + 2}{n}}. \nonumber
\end{align}
We denote $v = \sqrt{\EE \left[\|\hat{f}_n - f_0\|_n^2 \right]}$. Then the above implicit bound on $T_1$ implies
\begin{align}
& v^2 \leq b + 2av \label{eq:t1implicit} \\
\textrm{with}\quad & a = \sqrt{6} \sigma \sqrt{\frac{\log \cN(\delta, \cF(R, \kappa, L, p, K), \norm{\cdot}_\infty) + 2}{n}}, \nonumber \\
& b = \inf_{f \in \cF(R, \kappa, L, p, K)} \int_{\cM} (f(\xb) - f_0(\xb))^2 d\cD_x(\xb) \nonumber \\
& \qquad + \left(2\sqrt{6}\sqrt{\frac{\log \cN(\delta, \cF(R, \kappa, L, p, K), \norm{\cdot}_\infty) + 2}{n}} + 2\right)\sigma \delta. \nonumber
\end{align}  
Rearranging \eqref{eq:t1implicit} for $a, b > 0$, we deduce $(v - a)^2 \leq b + a^2$. Some manipulation then yields $v^2 \leq 4a^2 + 2b$, which implies
\begin{align}
T_1 = 2v^2 & \leq 4 \inf_{f \in \cF(R, \kappa, L, p, K)} \int_{\cM} (f(\xb) - f_0(\xb))^2 d\cD_x(\xb) + 48 \sigma^2 \frac{\log \cN(\delta, \cF(R, \kappa, L, p, K), \norm{\cdot}_\infty) + 2}{n} \nonumber \\
& \quad + \left(8\sqrt{6} \sqrt{\frac{\log \cN(\delta, \cF(R, \kappa, L, p, K), \norm{\cdot}_\infty) + 2}{n}} + 8\right)\sigma \delta. \nonumber
\end{align}
The proof is complete.
\end{proof}

\subsection{Proof of Lemma \ref{lemma:t2}}
\begin{proof}\label{pf:t2}
Recall that we denote $\hat{g}(\xb) = (\hat{f}_n(\xb) - f_0(\xb))^2$. We rewrite $T_2$ as
\begin{align*}
T_2 & = \EE\left[\int_{\cM} \hat{g}(\xb) d \cD_x(\xb) - \frac{2}{n}\sum_{i=1}^n \hat{g}(\xb_i) \right]\\
& = 2 \EE\left[\int_{\cM} \hat{g}(\xb) d \cD_x(\xb) - \frac{1}{n}\sum_{i=1}^n \hat{g}(\xb_i) - \frac{1}{2} \int_{\cM} \hat{g}(\xb) d \cD_x(\xb) \right] \\
& \leq 2 \EE\left[\int_{\cM} \hat{g}(\xb) d \cD_x(\xb) - \frac{1}{n}\sum_{i=1}^n \hat{g}(\xb_i) - \frac{1}{8R^2} \int_{\cM} \hat{g}^2(\xb) d \cD_x(\xb)\right].
\end{align*}
We lower bound $\int_{\cM} \hat{g}(\xb) d \cD_x(\xb)$ by its second moment:
\begin{align*}
\int_{\cM} \hat{g}(\xb) d \cD_x(\xb) & = \int_{\cM} \left(\hat{f}_n(\xb) - f_0(\xb)\right)^4 d\cD_x(\xb) \\
& = \int_{\cM} \left(\hat{f}_n(\xb) - f_0(\xb)\right)^2 \hat{g}(\xb) d\cD_x(\xb) \\
& \leq \int_{\cM} 4R^2 \hat{g}(\xb) d\cD_x(\xb).
\end{align*}
The last inequality follows from $\left|\hat{f}_n(\xb) - f_0(\xb)\right| \leq 2R$. Now we cast $T_2$ into
\begin{align}\label{eq:t1bound}
T_2 \leq 2 \EE\left[\int_{\cM} \hat{g}(\xb) d\cD_x(\xb) - \frac{1}{n}\sum_{i=1}^n \hat{g}(\xb_i) - \frac{1}{8R^2} \int_{\cM} \hat{g}^2(\xb) d\cD_x(\xb) \right].
\end{align}
Introducing the second moment allows us to establish a fast convergence of $T_2$. Specifically, we denote $\bar{\xb}_i$'s as independent copies of $\xb_i$'s following the same distribution. We also denote $$\cG = \left\{g(\xb) = \left(f(\xb) - f_0(\xb)\right)^2 ~\big|~ f \in \cF(R, \kappa, L, p, K) \right\}$$ as the function class induced by $\cF(R, \kappa, L, p, K)$. Then we upper bound \eqref{eq:t1bound} as
\begin{align}
T_2 & \leq 2 \EE\left[\sup_{g \in \cG} \left(\int_{\cM} g(\bar{\xb})d\cD_x(\bar{\xb}) - \frac{1}{n}\sum_{i=1}^n g(\xb_i) - \frac{1}{8R^2} \int_{\cM} g^2(\xb) d\cD_x(\xb)\right) \right] \nonumber \\
& \overset{(i)}{\leq} 2 \EE_{\xb, \bar{\xb}} \left[\sup_{g \in \cG} \frac{1}{n} \sum_{i=1}^n (g(\bar{\xb}_i) - g(\xb_i)) - \frac{1}{16R^2} \EE_{\xb, \bar{\xb}} [g^2(\bar{\xb}) + g^2(\xb)] \right], \label{eq:t2symmetric}
\end{align}
where $(i)$ follows from Jensen's inequality and shorthand $\EE_{\xb, \bar{\xb}}[\cdot]$ denotes the expectation (double integral $\int_{\cM} \int_{\cM} \cdot d\cD_x(\xb) d\cD_x(\bar{\xb})$) with respect to the joint distribution of $(\xb, \bar{\xb})$.

We discretize $\cG$ with respect to the $\ell_\infty$ norm. The $\delta$-covering number is denoted as $\cN(\delta, \cG, \norm{\cdot}_\infty)$ and the elements in the covering is denoted as $\cG^* = \left\{g^*_i\right\}_{i=1}^{\cN(\delta, \cG, \norm{\cdot}_\infty)}$, that is, for any $g \in \cG$, there exists a $g^*$ satisfying $\norm{g - g^*}_\infty \leq \delta$.

We replace $g \in \cG$ by $g^* \in \cG^*$ in bounding $T_2$, which then boils down to deriving concentration results on a finite concept class. Specifically, for $g^*$ satisfying $\norm{g - g^*}_\infty \leq \delta$, we have
\begin{align*}
g(\bar{\xb}_i) - g(\xb_i)& = g(\bar{\xb}_i) - g^*(\bar{\xb}_i) + g^*(\bar{\xb}_i) - g^*(\xb_i) + g^*(\xb_i) - g(\xb_i) \\
& \leq g^*(\bar{\xb}_i) - g^*(\xb_i) + 2\delta.
\end{align*}
We also have
\begin{align*}
g^2(\bar{\xb}) + g^2(\xb) & = \left[g^2(\bar{\xb}) - (g^*)^2(\bar{\xb})\right] + \left[(g^*)^2(\bar{\xb}) + (g^*)^2(\xb)\right] - \left[(g^*)^2(\xb) - g^2(\xb)\right] \\
& = (g^*)^2(\bar{\xb}) + (g^*)^2(\xb) + (g(\bar{\xb}) - g^*(\bar{\xb}))(g(\bar{\xb}) + g^*(\bar{\xb})) + (g^*(\xb) - g(\xb))(g^*(\xb) + g(\xb)) \\
& \geq (g^*)^2(\bar{\xb}) + (g^*)^2(\xb) - \left|g(\bar{\xb}) - g^*(\bar{\xb})\right|\left|g(\bar{\xb}) + g^*(\bar{\xb})\right| - \left| g^*(\xb) - g(\xb)\right| \left|g^*(\xb) + g(\xb)\right| \\
& \geq (g^*)^2(\bar{\xb}) + (g^*)^2(\xb) - 2R \delta - 2R \delta.
\end{align*}
Plugging the above two items into \eqref{eq:t2symmetric}, we upper bound $T_2$ as
\begin{align*}
T_2 & \leq 2 \EE_{\xb, \bar{\xb}} \left[\sup_{g^* \in \cG^*} \frac{1}{n} \sum_{i=1}^n \left(g^*(\bar{\xb}_i) - g^*(\xb_i)\right) - \frac{1}{16R^2} \EE_{\xb, \bar{\xb}} [(g^*)^2(\bar{\xb}) + (g^*)^2(\xb)] \right] + \left(4+\frac{1}{2R}\right)\delta \\
& = 2 \EE_{\xb, \bar{\xb}} \left[\max_j \frac{1}{n} \sum_{i=1}^n \left(g_j^*(\bar{\xb}_i) - g_j^*(\xb_i)\right) - \frac{1}{16R^2} \EE_{\xb, \bar{\xb}} [(g_j^*)^2(\bar{\xb}) + (g_j^*)^2(\xb)] \right] + \left(4+\frac{1}{2R}\right)\delta.
\end{align*}
Denote $h_j(i) = g^*_j(\bar{\xb}_i) - g^*_j(\xb_i)$. By symmetry, it is straightforward to see $\EE [h_j(i)] = 0$. The variance of $h_j(i)$ is computed as
\begin{equation*}
\Var[h_j(i)] = \EE \left[h_j^2(i) \right] = \EE \left[\left(g_j^*(\bar{\xb}_i) - g_j^*(\xb_i)\right)^2 \right] \overset{(i)}{\leq} 2 \EE \left[(g_j^*)^2(\bar{\xb}_i) + (g_j^*)^2(\xb_i) \right].
\end{equation*}
The last inequality $(i)$ utilizes the identity $(a - b)^2 \leq 2(a^2 + b^2)$. Therefore, we derive the following upper bound for $T_2$,
\begin{equation*}
T_2 \leq 2 \EE \left[\max_j \frac{1}{n} \sum_{i=1}^n h_j(i) - \frac{1}{32R^2} \frac{1}{n} \sum_{i=1}^n \Var [h_j(i)] \right] + \left(4+\frac{1}{2R}\right)\delta.
\end{equation*}
We invoke the moment generating function to bound $T_2$. Note that we have $\|h_j\|_\infty \leq (2R)^2$. Then by Taylor expansion, for $0 < t/n < \frac{3}{4R^2}$ and any $j$, we have
\begin{align}\label{eq:hiMGF}
\EE \left[\exp\left(\frac{t}{n} h_j(i)\right)\right] & = \EE\left[1 + \frac{t}{n} h_j(i) + \sum_{k=2}^\infty \frac{(t/n)^k h_j^k(i)}{k!} \right] \notag\\
& \leq \EE\left[1 + \frac{t}{n} h_j(i) + \sum_{k=2}^\infty \frac{(t/n)^k h_j^2(i) (4R^2)^{k-2}}{2 \times 3^{k-2}} \right] \notag\\
& = \EE\left[1 + \frac{t}{n} h_j(i) + \frac{(t/n)^2 h_j^2(i)}{2} \sum_{k=2}^\infty \frac{(t/n)^{k-2} (4R^2)^{k-2}}{3^{k-2}} \right] \notag\\
& = \EE\left[1 + \frac{t}{n} h_j(i) + \frac{(t/n)^2 h_j^2(i)}{2} \frac{1}{1-4tR^2/(3n)} \right] \notag\\
& = 1 + (t/n)^2 \Var[h_j(i)] \frac{1}{2-8tR^2/(3n)} \notag\\
& \overset{(i)}{\leq} \exp\left(\Var[h_j(i)] \frac{3(t/n)^2}{6-8tR^2/n} \right).
\end{align}
Step $(i)$ follows from the fact $1 + x \leq \exp(x)$ for $x \geq 0$. Given \eqref{eq:hiMGF}, we proceed to bound $T_2$. To ease the presentation, we temporarily neglect $\left(4+\frac{1}{2R}\right)\delta$ term and denote $T'_2 = T_2 - \left(4+\frac{1}{2R}\right)\delta$. Then for $0 < t/n < \frac{3}{4R^2}$, we have
\begin{align*}
\exp\left(t \frac{T'_2}{2}\right) & = \exp \left(t \EE \left[\max_j \frac{1}{n} \sum_{i=1}^n h_j(i) - \frac{1}{32R^2} \frac{1}{n} \sum_{i=1}^n \Var [h_j(i)] \right] \right) \\
& \overset{(i)}{\leq} \EE \left[\exp\left(t \max_j \frac{1}{n} \sum_{i=1}^n h_j(i) - \frac{1}{32R^2} \frac{1}{n} \sum_{i=1}^n \Var [h_j(i)] \right) \right] \\
& \leq \EE \left[\sum_j \exp \left(\frac{t}{n} \sum_{i=1}^n h_j(i) - \frac{1}{32R^2} \frac{t}{n} \sum_{i=1}^n \Var [h_j(i)] \right) \right] \\
& \overset{(ii)}{\leq} \EE \left[ \sum_j \exp\left(\sum_{i=1}^n \Var[h_j(i)] \frac{3 (t/n)^2}{6-8tR^2/n} - \frac{1}{32R^2} \frac{t}{n} \Var [h_j(i)]\right) \right]\\
& = \EE \left[\sum_j \exp\left(\sum_{i=1}^n \frac{t}{n} \Var[h_j(i)] \left(\frac{3 t/n}{6-8tR^2/n} - \frac{1}{32R^2}\right) \right) \right].
\end{align*}
Step $(i)$ follows from Jensen's inequality, and step $(ii)$ invokes \eqref{eq:hiMGF} for each $h(i)$. We now choose $t$ so that $\frac{3 t/n}{6-8tR^2/n} - \frac{1}{32R^2} = 0$, which yields $t = \frac{3n}{52R^2} < \frac{3n}{4R^2}$. Substituting our choice of $t$ into $\exp(t T'_2/2)$, we have
\begin{equation*}
t \frac{T'_2}{2} \leq \log \sum_j \exp(0) \Longrightarrow T'_2 \leq \frac{2}{t} \log \cN(\delta, \cG, \norm{\cdot}_\infty) = \frac{104R^2}{3n} \log \cN(\delta, \cG, \norm{\cdot}_\infty).
\end{equation*}
To complete the proof, we relate the covering number of $\cG$ to that of $\cF(R, \kappa, L, p, K)$. Consider any $g_1, g_2 \in \cG$ with $g_1 = (f_1 - f_0)^2$ and $g_2 = (f_2 - f_0)^2$, respectively, for $f_1, f_2 \in \cF(R, \kappa, L, p, K)$. We can derive
\begin{align*}
\norm{g_1 - g_2}_\infty & = \sup_{\xb} \left|\left(f_1(\xb) - f_0(\xb)\right)^2 - \left(f_2(\xb) - f_0(\xb)\right)^2\right| \\
& = \sup_{\xb} \left| f_1(\xb) - f_2(\xb) \right| \left| f_1(\xb) + f_2(\xb) - 2f_0(\xb) \right| \\
& \leq 4R \norm{f_1 - f_2}_\infty.
\end{align*}
The above characterization immediately implies $\cN(\delta, \cG, \norm{\cdot}_\infty) \leq \cN(\delta/4R, \cF(R, \kappa, L, p, K), \norm{\cdot}_\infty)$. Therefore, we derive the desired upper bound on $T_2$:
\begin{equation*}
T_2 \leq \frac{104R^2}{3n} \log \cN(\delta/4R, \cF(R, \kappa, L, p, K), \norm{\cdot}_\infty) + \left(4 +\frac{1}{2R}\right)\delta.
\end{equation*}
\end{proof}

\subsection{Proof of Lemma \ref{lemma:coveringbound}}\label{pf:coveringbound}
\begin{proof}
To construct a covering for $\cF(R, \kappa, L, p, K)$, we discretize each weight parameter by a uniform grid with grid size $h$. Recall we write $f \in \cF(R, \kappa, L, p, K)$ as $f = W_L \cdot \textrm{ReLU} (W_{L-1} \cdots \textrm{ReLU} (W_1 \xb + \bbb_1) \dots + \bbb_{L-1})+\bbb_L$. Let $f, f' \in \cF$ with all the weight parameters at most $h$ from each other. Denoting the weight matrices in $f, f'$ as $W_L, \dots, W_1, \bbb_L, \dots, \bbb_1$ and $W'_L, \dots, W'_1, \bbb'_L, \dots, \bbb'_1$, respectively, we bound the $\ell_\infty$ difference $\norm{f - f'}_\infty$ as
\begin{align*}
\norm{f - f'}_\infty & = \big\|W_L \cdot \textrm{ReLU} (W_{L-1} \cdots \textrm{ReLU} (W_1 \xb + \bbb_1) \cdots + \bbb_{L-1}) + \bbb_L \\
& \quad - (W'_L \cdot \textrm{ReLU} (W'_{L-1} \cdots \textrm{ReLU} (W'_1 \xb + \bbb'_1) \cdots + \bbb'_{L-1})-\bbb'_L)\big \|_\infty \\
& \leq \norm{\bbb_L - \bbb'_L}_\infty + \norm{W_L - W'_L}_1 \norm{W_{L-1} \cdots \textrm{ReLU} (W_1 \xb + \bbb_1) \cdots + \bbb_{L-1}}_\infty \\
&\quad + \norm{W_L}_1 \norm{W_{L-1} \cdots \textrm{ReLU} (W_1 \xb + \bbb_1) \cdots + \bbb_{L-1} - (W'_{L-1} \cdots \textrm{ReLU} (W'_1 \xb + \bbb'_1) \cdots + \bbb'_{L-1})}_\infty \\
& \leq h + h p \norm{W_{L-1} \cdots \textrm{ReLU} (W_1 \xb + \bbb_1) \cdots + \bbb_{L-1}}_\infty \\
& \quad + \kappa p \norm{W_{L-1} \cdots \textrm{ReLU} (W_1 \xb + \bbb_1) \cdots + \bbb_{L-1} - (W'_{L-1} \cdots \textrm{ReLU} (W'_1 \xb + \bbb'_1) \cdots + \bbb'_{L-1})}_\infty.
\end{align*}
We derive the following bound on $\norm{W_{L-1} \cdots \textrm{ReLU} (W_1 \xb + \bbb_1) \dots + \bbb_{L-1}}_\infty$:
\begin{align*}
\norm{W_{L-1} \cdots \textrm{ReLU} (W_1 \xb + \bbb_1) \cdots + \bbb_{L-1}}_\infty & \leq \norm{W_{L-1} (\cdots \textrm{ReLU} (W_1 \xb + \bbb_1) \cdots)}_\infty + \norm{\bbb_{L-1}}_\infty \\
& \leq \norm{W_{L-1}}_1 \norm{W_{L-2} (\cdots \textrm{ReLU} (W_1 \xb + \bbb_1) \cdots) + \bbb_{L-2}}_\infty + \kappa \\
& \leq \kappa p \norm{W_{L-2} (\cdots \textrm{ReLU} (W_1 \xb + \bbb_1) \cdots) + \bbb_{L-2}}_\infty + \kappa \\
& \overset{(i)}{\leq} (\kappa p)^{L-1} B + \kappa \sum_{i=0}^{L-3}(\kappa p)^i \\
& \leq (\kappa p)^{L-1} B + \kappa (\kappa p)^{L-2},
\end{align*}
where $(i)$ is obtained by induction and $\norm{\xb}_\infty \leq B$. The last inequality holds, since $\kappa p > 1$. Substituting back into the bound for $\norm{f - f'}_\infty$, we have
\begin{align*}
\norm{f - f'}_\infty & \leq \kappa p \norm{W_{L-1} \cdots \textrm{ReLU} (W_1 \xb + \bbb_1) \cdots + \bbb_{L-1} - (W'_{L-1} \cdots \textrm{ReLU} (W'_1 \xb + \bbb'_1) \cdots + \bbb'_{L-1})}_\infty \\
& \quad + h + h p \left[(\kappa p)^{L-1} B + \kappa (\kappa p)^{L-2} \right] \\
& \leq \kappa p \norm{W_{L-1} \cdots \textrm{ReLU} (W_1 \xb + \bbb_1) \cdots + \bbb_{L-1} - (W'_{L-1} \cdots \textrm{ReLU} (W'_1 \xb + \bbb'_1) \cdots + \bbb'_{L-1})}_\infty \\
& \quad + h (p B + 2) (\kappa p)^{L-1} \\
& \overset{(i)}{\leq} (\kappa p)^{L-1} \norm{W_1 \xb + \bbb_1 - W'_1 \xb - \bbb'_1}_\infty + h (L-1) (pB + 2) (\kappa p)^{L-1} \\
& \leq h L (pB + 2) (\kappa p)^{L-1},
\end{align*}
where $(i)$ is obtained by induction. We choose $h$ satisfying $h L (pB + 2) (\kappa p)^{L-1} =\delta$. Then discretizing each parameter uniformly into $2\kappa / h$ grid points yields a $\delta$-covering on $\cF$. Note that there are ${Lp^2 \choose K} \leq (Lp^2)^K$ different choices of $K$ non-zero entries out of $Lp^2$ total weight parameters. Therefore, the covering number is upper bounded by
\begin{equation*}
\cN(\delta, \cF(R, \kappa, L, p, K), \norm{\cdot}_\infty) \leq (Lp^2)^K \left(\frac{2\kappa}{h}\right)^{K} \leq \left(\frac{2L^2 (pB + 2) \kappa^L p^{L+1}}{\delta}\right)^K.
\end{equation*}
\end{proof}

\subsection{Proof of Theorem \ref{thm:stat} --- Bias-Variance Trade-off}\label{pf:regression}
\begin{proof}
We recall the bias and variance decomposition of $\EE\left[\int_{\cM} \left(\hat{f}_n(\xb) - f_0(\xb)\right)^2 d\cD_x(\xb) \right]$ as
\begin{align*}
\EE\left[\int_\cM \left(\hat{f}_n(\xb) - f_0(\xb)\right)^2 d\cD_x(\xb) \right] & = \underbrace{\EE \left[\frac{2}{n}\sum_{i=1}^n (\hat{f}_n(\xb_i) - f_0(\xb_i))^2 \right]}_{T_1} \\
& \quad + \underbrace{\EE \left[\int_\cM \left(\hat{f}_n(\xb) - f_0(\xb)\right)^2 d\cD_x(\xb) \right] - \EE \left[\frac{2}{n}\sum_{i=1}^n (\hat{f}_n(\xb_i) - f_0(\xb_i))^2 \right]}_{T_2}.
\end{align*}
Combining the upper bounds on $T_1$ and $T_2$ in Lemmas \ref{lemma:t1} and \ref{lemma:t2}, we can derive
\begin{align*}
\EE\left[\int_\cM \left(\hat{f}_n(\xb) - f_0(\xb)\right)^2 d\cD_x(\xb) \right] & \leq 4 \inf_{f \in \cF(R, \kappa, L, p, K)} \int_{\cM} (f(\xb) - f_0(\xb))^2 d\cD_x(\xb) \\
& \quad + 48\sigma^2\frac{\log \cN(\delta, \cF(R, \kappa, L, p, K), \norm{\cdot}_\infty)+2}{n} \\
& \quad + 8\sqrt{6} \sqrt{\frac{\log \cN(\delta, \cF(R, \kappa, L, p, K), \norm{\cdot}_\infty)+2}{n}} \sigma \delta \\
& \quad + \frac{104R^2}{3n} \log \cN(\delta/4R, \cF(R, \kappa, L, p, K), \norm{\cdot}_\infty) \\
& \quad + \left(4 +\frac{1}{2R} + 8\sigma\right)\delta.
\end{align*}
By our choice of $\cF(R, \kappa, L, p, K)$, there exists a network class which can yield a function $f$ satisfying $\norm{f - f_0}_\infty \leq \epsilon$ for $\epsilon \in (0, 1)$. We will choose $\epsilon$ later for the bias-variance trade-off. Such a network consists of $L = \tilde{O}\left(\log \frac{1}{\epsilon}\right)$ layers and $K = \tilde{O}\left(\left(\epsilon^{-\frac{d}{s+\alpha}} + D\right)\log \frac{1}{\epsilon}\right)$ weight parameters. Invoking the upper bound of the covering number in Lemma \ref{lemma:coveringbound}, we derive
\begin{align}
\EE\left[\int_\cM \left(\hat{f}_n(\xb) - f_0(\xb)\right)^2 d\cD_x(\xb) \right]  & \leq 4\epsilon^2 + \frac{48\sigma^2}{n} \left(K \log \left(2R^2L^2 (pB + 2) \kappa^L p^{L+1} / \delta\right) + 2\right) \nonumber \\
& \quad + 8\sqrt{6} \sqrt{\frac{K \log \left(2RL^2 (pB + 2) \kappa^L p^{L+1} / \delta\right)}{n}} \sigma \delta \nonumber \\
& \quad + \frac{104R^2}{3n} K \log \left(8R^2L^2 (pB + 2) \kappa^L p^{L+1} / \delta\right) \nonumber \\
& \quad + \left(4 +\frac{1}{2R} + 8\sigma\right)\delta \nonumber \\
& = \tilde{O}\bigg(\epsilon^2 + \frac{R^2+\sigma^2}{n} \left(\epsilon^{-\frac{d}{s+\alpha}} + D \right) \log \frac{1}{\epsilon} \log \frac{L^2 (\kappa p)^{L+1}}{\delta} \nonumber \\
& \qquad \quad + \sigma \delta \sqrt{\frac{\left(\epsilon^{-\frac{d}{s+\alpha}} + D \right) \log \frac{1}{\epsilon} \log \frac{L^2 (\kappa p)^{L+1}}{\delta}}{n}} + \sigma \delta + \frac{\sigma^2}{n} \bigg). \label{eq:staterrorboundraw}
\end{align}
Now we choose $\epsilon$ to satisfy $\epsilon^2 = \frac{1}{n} \epsilon^{-\frac{d}{s+\alpha}}$, which gives $\epsilon = n^{-\frac{s+\alpha}{d + 2(s + \alpha)}}$. It suffices to pick $\delta = \frac{1}{n}$. Substitute both $\epsilon$ and $\delta$ into \eqref{eq:staterrorboundraw}, we deduce the desired estimation error bound
\begin{align*}
\EE\left[\int_\cM \left(\hat{f}_n(\xb) - f_0(\xb)\right)^2 d\cD_x(\xb) \right] & = \tilde{O}\bigg(\epsilon^2 + \frac{R^2+\sigma^2}{n} \left(\epsilon^{-\frac{d}{s+\alpha}} + D \right) \log \frac{1}{\epsilon} \log \frac{L^2 (\kappa p)^{L+1}}{\delta} \nonumber \\
& \qquad \quad + \sigma \delta \sqrt{\frac{\left(\epsilon^{-\frac{d}{s+\alpha}} + D \right) \log \frac{1}{\epsilon} \log \frac{L^2 (\kappa p)^{L+1}}{\delta}}{n}} + \sigma \delta + \frac{\sigma^2}{n} \bigg) \\
& \leq c (R^2 + \sigma^2) \left(n^{-\frac{2(s + \alpha)}{d + 2(s + \alpha)}} + \frac{D}{n}\right) \log^3 n,
\end{align*}
where constant $c$ depends on depending on $\log D$, $d$, $s$, $\tau$, $B$, the surface area of $\cM$, and the upper bounds of derivatives of the coordinate systems $\phi_i$'s and partition of unity $\rho_i$'s, up to order $s$.
\end{proof}

\end{document}